\documentclass{article}

\usepackage[preprint]{neurips_2025} 

\usepackage[utf8]{inputenc} 
\usepackage[T1]{fontenc}    
\usepackage{hyperref}       
\usepackage{url}            
\usepackage{booktabs}       
\usepackage{amsfonts}       
\usepackage{nicefrac}       
\usepackage{microtype}      
\usepackage{xcolor}         
\usepackage{longtable}
\usepackage{multirow}
\usepackage[table]{xcolor}
\usepackage{pdfpages}

\usepackage{graphicx}
\usepackage{float}
\usepackage{caption}
\usepackage{subcaption}
\usepackage{mathtools}
\usepackage{chngcntr}
\usepackage[capitalize,noabbrev]{cleveref}
\crefname{equation}{}{} 
\usepackage{algorithm}
\usepackage{algpseudocode}
\usepackage{amsthm,amsfonts,amssymb,amsmath,bbm,mathrsfs}
\usepackage{tcolorbox}
\usepackage{tikz}
\usetikzlibrary{arrows.meta,calc,positioning}

\setcitestyle{authoryear,open={(},close={)}}

\definecolor{attentionblue}{RGB}{0,102,204}
\definecolor{attentiongray}{RGB}{240,240,240}

\newtcbox{\tokenbox}{%
  fontupper=\ttfamily,
  colback=gray!10,
  boxrule=0pt,             
  arc=2pt,
  boxsep=0pt,
  frame empty,
  left=2pt,
  right=2pt,
  top=2pt,                 
  bottom=2pt,              
  nobeforeafter,
  valign=center,
  baseline,
  tcbox raise base,
  verbatim,                
  before upper={\vphantom{Äg}},
}

\newtcbox{\tokenboxline}{%
  fontupper=\ttfamily,
  colback=gray!10,
  boxrule=0.5pt,           
  arc=2pt,
  boxsep=0pt,
  left=2pt,
  right=2pt,
  top=2pt,                 
  bottom=2pt,              
  nobeforeafter,
  valign=center,
  baseline,
  tcbox raise base,
  verbatim,
  before upper={\vphantom{Äg}},
}

\NewDocumentCommand{\hatlambda}{O{}}{\hat{\lambda}(#1)}
\NewDocumentCommand{\Ln}{O{}}{L_n(#1)}

\newcommand{\E}{\mathbb{E}}

\newcommand{\dataset}[1]{\texttt{#1}}
\newcommand{\pilesub}[1]{\texttt{#1}}

\newcommand\legendbox[2]{%
    \tcbox[%
        nobeforeafter,tcbox raise base,size=fbox,
        colback=#1,colframe=#1!70!black,
    ]{#2}
}

\definecolor{patternWordPart}{HTML}{E377C2}
\definecolor{patternInduction}{HTML}{FF7F0E}
\definecolor{patternFormatting}{HTML}{686868}
\definecolor{patternSpacing}{HTML}{2CA02C}
\definecolor{patternWordStart}{HTML}{9467BD}
\definecolor{patternRightDelimiter}{HTML}{D62728}
\definecolor{patternLeftDelimiter}{HTML}{D62728}
\definecolor{patternNumeric}{HTML}{BCBD22}
\definecolor{patternWordEnd}{HTML}{1F77B4}
\definecolor{patternFunction}{HTML}{21d6d9}
\definecolor{patternCapitalized}{HTML}{a31bb5}

\NewDocumentCommand{\task}{O{}}{\mathbf{t}_{#1}}
\NewDocumentCommand{\residactvec}{O{(l)} O{i}}{\mathbf{z}^{#1}_{#2}}
\NewDocumentCommand{\residactcomp}{O{(l))} O{i} O{j}}{z^{#1}_{#2,#3}}

\newcommand{\batchsize}{m}
\NewDocumentCommand{\batchloss}{O{}}{L_{\batchsize}^{#1}}

\newcommand{\tok}[1]{\tokenbox{#1}}

\newcommand{\clu}[1]{\textbf{C#1}}
\newcommand{\cluq}[1]{\textbf{#1}}
\newcommand{\sae}[2]{\textbf{L#1-#2}}

\usepackage[table]{xcolor}

\definecolor{clusterBlue}{RGB}{235,245,255}   
\definecolor{clusterPurple}{RGB}{245,240,255}   
\definecolor{clusterGreen}{RGB}{240,250,240} 
\definecolor{clusterOrange}{RGB}{255,245,235}    
\definecolor{clusterGray}{RGB}{245,245,245}    

\newlength{\CARDHEIGHT}
\newcommand{\clustercard}[7]{%
\colorbox{#7}{%
\parbox[t][\CARDHEIGHT][t]{0.48\linewidth}{%
\footnotesize
\vspace{2pt}%
\textbf{#1}\hfill{\scriptsize(#2)}\\
{\raggedright\emph{#3}\par}%
\vspace{4pt}

\texttt{#4}\\
\texttt{#5}\\
\texttt{#6}%
\vfill%
}%
}%
}

\usepackage{tabularx}

\usepackage{pifont}

\usepackage{xcolor}
\definecolor{tokyellow}{RGB}{255,242,170}
\newcommand{\finaltok}[1]{\begingroup\setlength{\fboxsep}{1pt}\colorbox{tokyellow}{#1}\endgroup}

\usepackage{listings}
\renewcommand{\paragraph}[1]{\textbf{#1\ \ \ }}

\theoremstyle{plain}
\newtheorem{theorem}{Theorem}[section]

\newtheorem{lemma}[theorem]{Lemma}

\theoremstyle{definition}
\newtheorem{definition}[theorem]{Definition}

\theoremstyle{remark}

\title{%
    Towards Spectroscopy: Susceptibility Clusters in Language Models
}
\author{
  Andrew Gordon$^{=}$ \\
  Timaeus \\
  \texttt{andrew@timaeus.co}
  \And
  Garrett Baker$^{=}$ \\
  Timaeus \\
  \texttt{garrett@timaeus.co}
  \And
  George Wang \\
  Timaeus \\
  \texttt{george@timaeus.co}
  \AND
  William Snell \\
  Timaeus \\
  \texttt{william@timaeus.co}
  \And
  Stan van Wingerden \\
  Timaeus \\
  \texttt{stan@timaeus.co}
  \And
  Daniel Murfet \\
  Timaeus \\
  \texttt{daniel@timaeus.co}
}

\begin{document}

\maketitle

\begin{abstract}
Spectroscopy infers the internal structure of physical systems by measuring their response to perturbations. We apply this principle to neural networks: perturbing the data distribution by upweighting a token $y$ in context $x$, we measure the model's response via susceptibilities $\chi_{xy}$, which are covariances between component-level observables and the perturbation computed over a localized Gibbs posterior via stochastic gradient Langevin dynamics (SGLD). Theoretically, we show that susceptibilities decompose as a sum over \emph{modes} of the data distribution, explaining why tokens that follow their contexts ``for similar reasons'' cluster together in susceptibility space. Empirically, we apply this methodology to Pythia-14M, developing a conductance-based clustering algorithm that identifies 510 interpretable clusters ranging from grammatical patterns to code structure to mathematical notation. Comparing to sparse autoencoders, 50\% of our clusters match SAE features, validating that both methods recover similar structure.
\end{abstract}

\section{Introduction}

A fundamental challenge in science is inferring the internal structure of complex systems. In deep learning this takes the form of interpretability, which seeks to understand internal structure in neural networks and how it underlies generalization. For physical matter a key interpretability technique is \emph{spectroscopy}. This is a system of methodologies for probing systems with external perturbations, measuring their responses, and inferring facts about structure from this data. In this paper we follow \citet{baker2025studyingsmalllanguagemodels} in applying the same conceptual framework to neural network interpretability for language models from the Pythia family \citep{biderman2023pythia}.

Shine light of frequency $\omega$ on a molecule: at most frequencies nothing interesting happens, but at certain frequencies the molecule absorbs or emits strongly. These peaks reveal internal structure: the existence of structural groups, the nature of bonds and the geometry of the molecule. These peaks exist because physical systems have \emph{normal modes}, characteristic patterns of collective vibration each with its own natural frequency $\omega_\alpha$. When a probe matches a mode's natural frequency, energy transfers efficiently. The spectrum of responses maps the modes and thus the structure.

\textbf{What is the spectrum of a language model?} If we take a neural network as our complex material then the natural probes are perturbations to the data distribution -- specifically, upweighting particular tokens $y$ in contexts $x$. The ``response'' is how the model's internal components adjust under this perturbation, measured by susceptibilities $\chi_{xy}$ as introduced in \citet{baker2025studyingsmalllanguagemodels}. And the ``normal modes'' are patterns in the data distribution that explain why certain continuations follow certain contexts. It is a natural hypothesis that internal structure in neural networks (e.g. features and circuits) form over training in response to such patterns, and so mapping these modes is a promising route to understanding internal structure. The spectrum of responses in this case is the set of susceptibility vectors $\chi_{xy}$ for sequences $xy \sim q(x,y)$ sampled from a distribution like the Pile \citep{gao2020pile} and in this paper we study this ``spectrum'' extensively for Pythia-14M (\cref{fig:pythia14m_umap1}).

\textbf{Our main finding is that Pythia-14M has developed specialized responses to hundreds of identifiable patterns in its training data}. We find 510 interpretable clusters in susceptibility space, ranging from linguistic structures (sentence boundaries, prepositions) to dataset-specific regularities (LaTeX math mode, code syntax). These clusters reflect modes in the data distribution and internal structure that has formed in response to those modes. To justify this claim theoretically, we introduce the appropriate analogue of ``normal modes'' for a language model. Following \citet{modes2}, the conditional distribution $q(y|x)$ admits a singular value decomposition into modes indexed by $\alpha, \beta$. Each context-token pair $(x,y)$ has a \emph{propensity profile} $\{s_{\alpha\beta}(xy)\}_{\alpha, \beta}$ measuring how strongly each mode explains this particular continuation. The susceptibility decomposes as:
\begin{equation}\label{eq:chi_decomposition}
\chi_{xy} = \sum_{\alpha,\beta} s_{\alpha\beta}(xy) \chi_{\alpha\beta} - \bar{\chi}
\end{equation}
where $\chi_{\alpha\beta}$ is the model's response to the mode pair $(\alpha,\beta)$ and $\bar{\chi}$ is defined by $\E_{xy}[ \chi_{xy} ] = 0$. As in physical spectroscopy, it is difficult to prepare a perfectly pure probe because we do not know the modes of the material \emph{a priori}. We make do with what we have (i.e. polychromatic light sources, individual token sequences) and so the measured response is a \emph{superposition} like \eqref{eq:chi_decomposition} of responses $\chi_{\alpha\beta}$ to all the modes in the probe signal, weighted by the probe's spectral weights $s_{\alpha\beta}$.

Every token pair is an impure probe, exciting a mixture of modes. But if two pairs $xy$ and $x'y'$ share similar propensity profiles, that is, if $y$ follows $x$ for ``similar reasons'' to why $y'$ follows $x'$, then their susceptibility vectors will be nearby $\chi_{xy} \approx \chi_{x'y'}$. If a set of token sequences $\{ x_i y_i \}_{i=1}^m$ all follow a particularly strong pattern $\alpha$, and the model is sensitive to this pattern $\| \chi_{\alpha \alpha} \| \gg 0$, then \eqref{eq:chi_decomposition} will push the vectors $\chi_{x_i y_i}$ in a similar direction. \textbf{This is our proposed mechanism for cluster formation, and for the relation of clusters to modes: clusters are the spectral signature of patterns.}

Our contributions:

\begin{itemize}
    \item \textbf{We develop a clustering methodology} based on conductance that identifies clusters directly from the susceptibility data without relying on visual inspection, yielding 510 highly interpretable clusters (\cref{section:methodology-clustering}).

    \item \textbf{We provide a theoretical link between clusters and patterns} in the data distribution via the mode decomposition above (\cref{sec:susceptibilities-modes}). Simple clusters form when token pairs share a dominant mode; the cluster's semantic content reflects that mode.

    \item \textbf{We discover networks of linked clusters} corresponding to structured phenomena like HTML markup and code blocks, suggesting that linked modes in the data give rise to interacting structure in the model (\cref{section:networks_linked}).

    \item \textbf{We compare susceptibility clusters to SAE features}, finding that 50\% of clusters have a matched feature in Pythia-70M SAEs, validating that both methods recover similar structure (\cref{section:sae-methodology}).

    \item \textbf{We show that these patterns persist at scale}: measuring conductance of Pythia-14M clusters in larger models (up to 1.4B), we find they remain coherent, indicating the clusters are genuine rather than artifacts of the small model (\cref{section:scaling}).
\end{itemize}

In \citet{baker2025studyingsmalllanguagemodels} the discovery of internal structure in neural networks through analysis of susceptibility data was termed \emph{structural inference}. The results of this paper demonstrate that \textbf{structural inference reveals rich, interpretable structure in Pythia-14M}: the model has internalized responses to hundreds of patterns, from basic syntax to domain-specific conventions, and this structure persists across model scales. The framework connecting susceptibilities to modes provides a principled basis for understanding how patterns in training data give rise to organization in trained models.

\begin{figure}[p]
    \centering
    \begin{tikzpicture}
        \begin{scope}
            \clip[rounded corners=12pt] (0,0) rectangle (\textwidth,0.72\textwidth);
            \node[anchor=south west,inner sep=0] (image) at (0,0) 
                {\includegraphics[width=\textwidth]{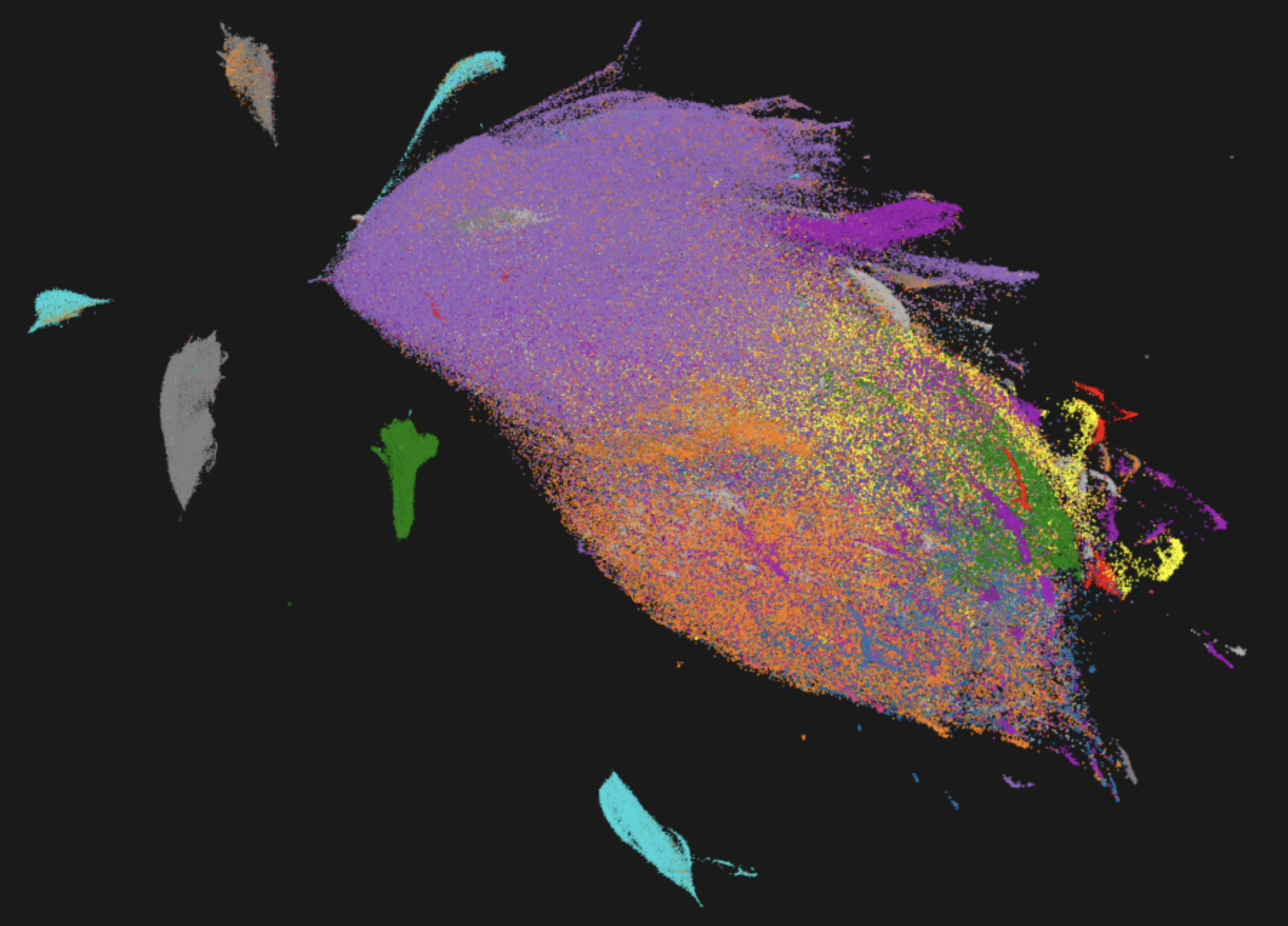}};
        \end{scope}
        
        \draw[white, line width=1pt, rounded corners=12pt] (0,0) rectangle (\textwidth,0.75\textwidth);

        \begin{scope}[shift={(0.5cm,0.5cm)}]
            \clip[rounded corners=6pt] (0,0) rectangle (0.25\textwidth,0.25\textwidth);
            \node[anchor=north west,inner sep=0] at (0,0.25\textwidth) 
                {\includegraphics[width=0.15\textwidth]{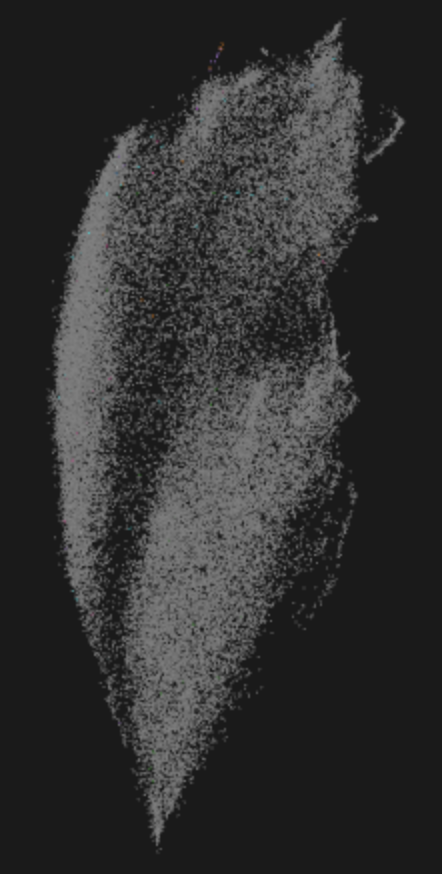}};
            \draw[white, line width=0.5pt, rounded corners=6pt] (0,0) rectangle (0.25\textwidth,0.25\textwidth);
        \end{scope}
        
        \begin{scope}[x={(image.south east)},y={(image.north west)}]
            \node[fill=black, fill opacity=0.7, text=gray, rounded corners=3pt, inner sep=1pt, font=\footnotesize\bfseries] at (0.32,0.38) {\tokenbox{\textbackslash n}};
            \node[fill=black, fill opacity=0.7, text=gray, rounded corners=3pt, inner sep=1pt, font=\footnotesize\bfseries] at (0.2,0.83) {\tokenbox{-}};
            \node[fill=black, fill opacity=0.7, text=gray, rounded corners=3pt, inner sep=1pt, font=\footnotesize\bfseries] at (0.35,0.95) {\tokenbox{~a}};
            \node[fill=black, fill opacity=0.7, text=gray, rounded corners=3pt, inner sep=1pt, font=\footnotesize\bfseries] at (0.14,0.43) {\tokenbox{.},\tokenbox{,}};
            \node[fill=black, fill opacity=0.7, text=gray, rounded corners=3pt, inner sep=1pt, font=\footnotesize\bfseries] at (0.05,0.62) {\tokenbox{~and}};
            \node[fill=black, fill opacity=0.7, text=gray, rounded corners=3pt, inner sep=1pt, font=\footnotesize\bfseries] at (0.47,0.05) {\tokenbox{~the}};
            \node[fill=black, fill opacity=0.7, text=gray, rounded corners=3pt, inner sep=1pt, font=\footnotesize\bfseries] at (0.55,0.95) {\tokenbox{~which}};
            \node[fill=black, fill opacity=0.7, text=gray, rounded corners=3pt, inner sep=1pt, font=\footnotesize\bfseries] at (0.84,0.74) {\tokenbox{~be},\tokenbox{~been}};

            \node[white, font=\bfseries] (A) at (0.7,0.85) {(A)};
            \draw[->,thick,white] (A) -- (0.7,0.77);

            \node[white, font=\bfseries] (B) at (0.65,0.4) {(B)};
            \draw[->,thick,white] (B) -- (0.6,0.53);

            \node[white, font=\bfseries] (C) at (0.95,0.35) {(C)};
            \draw[->,thick,white] (C) -- (0.92,0.4);

            \node[white, font=\bfseries] (D) at (0.85,0.68) {(D)};
            \draw[->,thick,white] (D) -- (0.77,0.65);

            \node[white, font=\bfseries] (E) at (0.65,0.95) {(E)};
            \draw[->,thick,white] (E) -- (0.57,0.87);

            \node[white, font=\bfseries] (F) at (0.91,0.6) {(F)};
            \draw[->,thick,white] (F) -- (0.88,0.55);

            \node[white, font=\bfseries] (G) at (0.93,0.52) {(G)};
            \draw[->,thick,white] (G) -- (0.86,0.53);

            \node[white, font=\bfseries] (H) at (0.95,0.2) {(H)};
            \draw[->,thick,white] (H) -- (0.95,0.28);

            \node[white, font=\bfseries] (L) at (0.24,0.36) {(L)};
            \draw[->,thick,white] (L) -- (0.11,0.36);
            
            \node[white, font=\bfseries] (I) at (0.24,0.32) {(I)};
            \draw[->,thick,white] (I) -- (0.17,0.32);

            \node[white, font=\bfseries] (J) at (0.24,0.16) {(J)};
            \draw[->,thick,white] (J) -- (0.16,0.16);

            \node[white, font=\bfseries] (K) at (0.24,0.2) {(K)};
            \draw[->,thick,white] (K) -- (0.07,0.2);

        \end{scope}
    \end{tikzpicture}
    
\vspace{1em}
\noindent\small
\legendbox{patternWordStart}{\tiny\textcolor{white}{Word start}}
\legendbox{patternWordPart}{\tiny\textcolor{white}{Word part}}
\legendbox{patternWordEnd}{\tiny\textcolor{white}{Word end}}
\legendbox{patternInduction}{\tiny\textcolor{white}{Induction}}
\legendbox{patternSpacing}{\tiny\textcolor{white}{Spacing}}
\legendbox{patternRightDelimiter}{\tiny\textcolor{white}{Delimiter}}
\legendbox{patternFormatting}{\tiny\textcolor{white}{Formatting}}
\legendbox{patternNumeric}{\tiny\textcolor{white}{Numeric}}
\legendbox{patternFunction}{\tiny\textcolor{white}{Function Word}}
\legendbox{patternCapitalized}{\tiny\textcolor{white}{Capitalized}}
\vspace{0.5em}

    \caption{
        \textbf{The spectrum of Pythia-14M.} Representation using UMAP of $780{,}000$ susceptibility vectors $\{ \chi_{x_iy_i} \}_i$ computed for a 14M parameter language model. Shown is one view on a 3D point cloud. Each point represents a token $y$ in context $x$, colored by pattern type (legend above, see \cref{tab:token-patterns}). Marked are some external bodies for token sequences where $y$ is a particular token and $x$ varies (e.g. there is a green body for $y = \tok{\textbackslash n}$). Descriptions of clusters (A)-(L) can be found in Table \ref{tab:umap-clusters}. \emph{Bottom left inset:} zoomed in view of the body of tokens \tok{.},\tok{,}.
    }
    \label{fig:pythia14m_umap1}
\end{figure}
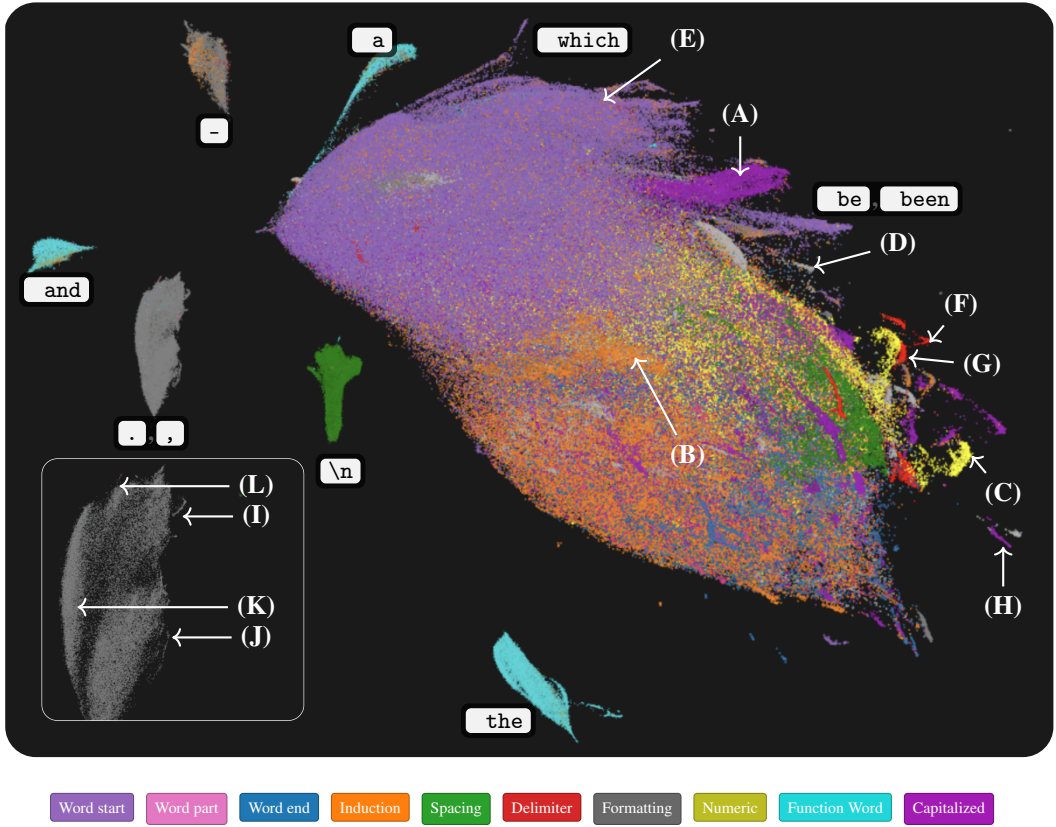

\begin{table}[p]
\centering
\small
\begin{tabular}{@{}cllcc@{}}
\toprule
\textbf{Label} & \textbf{Description} & \textbf{Examples} & \textbf{Single} & \textbf{Cluster}\\
\midrule
(A) & Sentence starters & \tok{.}\tokenboxline{~Finally}, \tok{.}\tokenboxline{~Thus} & No & \clu{180}, \cluq{280}, \cluq{305}\\
(B) & Variable names in mathematics & \tokenboxline{s}\tok{**}\tok{2}, \tok{3}\tok{*}\tokenboxline{t} & No & \clu{51}, \cluq{58}, \cluq{72}\\ 
(C) & Exponent $2$ & \tok{a}\tok{**}\tokenboxline{2}, \tok{b}\tok{**}\tokenboxline{2} & Yes & \clu{148}\\
(D) & Enclitics & \tok{I}\tokenboxline{'m}, \tok{I}\tokenboxline{'ve} & No & \clu{238}\\
(E) & Logical implication & \tokenboxline{~thus}, \tokenboxline{~therefore} & No & \clu{424}\\
(F) & Left brackets in multiple choice & \tokenboxline{(}\tok{a}\tok{)}, \tokenboxline{(}\tok{b}\tok{)} & Yes & \clu{241}, \cluq{316}\\
(G) & Left brackets in functions & \tok{p}\tokenboxline{(}\tok{n}\tok{)}, \tok{w}\tokenboxline{(}\tok{y}\tok{)} & Yes & \clu{81}, \cluq{95}, \cluq{349}\\
(H) & True and False & \tokenboxline{True}, \tokenboxline{False} & No & \clu{250}\\
(I) & Pointed abbreviations & \tok{U}\tok{.}\tok{S}\tokenboxline{.} & Yes & \clu{308}, \cluq{474}\\
(J) & Comma in list of email addresses & \tok{ECT}\tok{@}\tok{ECT}\tokenboxline{,} & Yes & \clu{154}\\
(K) & End of sentence in general language & \tok{~as}\tok{~Trustee}\tokenboxline{.} & Yes & \textbf{-}\\
(L) & End of sentence in math & \tok{0}\tok{=}\tok{g}\tok{*}\tok{d}\tokenboxline{.} & Yes & \clu{94}\\
\bottomrule
\end{tabular}
\vspace{1em}
\caption{\label{tab:umap-clusters}
\textbf{A selection of clusters} from  \cref{fig:pythia14m_umap1}. The clusters span natural language (A, D, E), mathematics (B, C, H, G, L), and code-like syntax (J). In the fourth column we show whether the cluster has a single shared $y$ token. For a complete list of clusters see \cref{section:cluster_table}.}
\end{table}

\section{Background}\label{section:background}

\subsection{Tokens and patterns} \label{section:setup-tokens-and-patterns}

We denote by $\Sigma$ the set of tokens. All Pythia models use the GPT-NeoX tokenizer \citep{biderman2023pythia}, a byte-pair encoding (BPE) tokenizer with $|\Sigma| = 50{,}304$ tokens. We often consider a token $y \in \Sigma$ in context $x \in \Sigma^k$. When presenting these we typically only give $y$ with at most some small segment of $x$, e.g. \tokenbox{~wa}\tokenbox{vel}\tokenboxline{ength} where only the last two tokens of the context $x$ are given and we indicate the token to be predicted with a solid black outline.

Following \citet{baker2025studyingsmalllanguagemodels,wang2025lang2pt5} we present susceptibility UMAPs (such as \cref{fig:pythia14m_umap1}) by coloring tokens according to \emph{token pattern categories} that were found to be meaningful for small language models. We extend the eight categories from \citet{wang2025lang2pt5} with two additional categories -- \emph{Function} (function words like \tok{~the}, \tok{~and}) and \emph{Capitalized} (capitalized words and acronyms) -- for a total of ten. In the same way that we organize visible light into broad categories (red, blue, etc.) within which there is much finer variation, so too here we find that clusters often tend to be monochromatic in this coloring scheme: the finer distinctions made by the model often subdivide these token categories, although not always. Full definitions are given in \cref{appendix:token-definitions}.

\subsection{Susceptibility Space}

The central tool is the \emph{susceptibility vector} $\chi_{xy}$ associated to a token sequence $x \in \Sigma^k$ and next token $y \in \Sigma$. Here $\Sigma$ denotes the set of tokens, $q(x,y)$ the true data distribution over contexts $x \in \Sigma^k$ and continuations $y \in \Sigma$, and $p(y|x,w)$ the model's predictive distribution with parameters $w$. 

Given a component $C$ of the network (such as an attention head) we define the \emph{per-token susceptibility}
\begin{equation}\label{eq:per-token-susceptibility}
\chi^C_{xy} \;=\; -\mathrm{Cov}\Big[\phi_C(w),\; \ell_{xy}(w) - L(w)\Big]
\end{equation}
where $\ell_{xy}(w) = -\log p(y|x,w)$ is the per-token loss, $L(w) = \mathbb{E}_{xy}[\ell_{xy}(w)]$ is the population loss, $\phi_C(w)$ is an observable localised on the component $C$, and the covariance is taken with respect to a tempered posterior distribution over weights; see \citet{baker2025studyingsmalllanguagemodels} for details. Intuitively, $\chi^C_{xy}$ measures the first-order response of the component $C$ to upweighting the token sequence $xy$ in the data distribution: it captures how fluctuations in $C$'s behavior covary with fluctuations in the prediction of $y$ given $x$ specifically. Aggregating over components gives the susceptibility vector
\begin{gather}
\chi_{xy} = \big( \chi^{C_1}_{xy}, \ldots, \chi^{C_H}_{xy} \big) \in \mathbb{R}^H \label{eq:intro_suscep_vector}
\end{gather}
which provides an overall response profile of the model to ``fluctuations'' in the data distribution. The claim in \citet{baker2025studyingsmalllanguagemodels, wang2025lang2pt5} is that this quantity is sensitive to \emph{how} the model computes the prediction of $y$ in context $x$.

We can use susceptibilities to relate \emph{structure in the data} with \emph{structure in the model}: if $\{ x_i y_i \}_{i=1}^n$ is a set of token sequences, then the singular value decomposition of the $n \times H$ response matrix $X$ with rows $\chi_{x_iy_i}$ couples principal components (linear combinations of token sequences) with parts of the model. The main result of \citet{baker2025studyingsmalllanguagemodels} is that in a 3M parameter attention-only transformer this couples induction patterns (structure in the data) to the induction circuit (structure in the model). In \citet{wang2025lang2pt5} this was used to study the development of the induction circuit (and other structure) over training. In short, it has been established in these works that susceptibilities can be used to study internal structure in very small language models.

\paragraph{Visualization.} We visualize the high-dimensional susceptibility vectors $\{ \chi_{x_iy_i} \} \subseteq \mathbb{R}^H$ using UMAP, a standard dimensionality reduction technique. In this paper, the role of UMAP is purely illustrative: no claims are based solely on examining the UMAP. It serves as a visual device for communicating quantitative insights that are derived and supported by other means (principally, the conductance-based clustering of \cref{section:methodology-clustering}). While it is reasonable to question the faithfulness of UMAP embeddings, this concern has no bearing on our main results.

\subsection{The spectrum of a language model}\label{sec:susceptibilities-modes}

Informally, the reason that $y \in \Sigma$ follows $x \in \Sigma^k$ in natural language can be attributed to various \emph{patterns} or \emph{modes} in the data distribution. These modes can be defined formally via singular value decomposition of the conditional distribution $q(y|x)$, viewed as a matrix indexed by contexts and continuations. Such spectral decompositions have a long history in NLP and learning theory \citep{anandkumar2014tensor, hsu2012spectral}; here we follow the formulation of \citet{modes2}.

In real data the mode structure is complex, but simplified distributions illustrate the concept. \citet{modes2} analyze theoretical examples such as ``absolute bigrams'' (where a context uniquely determines its continuation) and compute empirical modes from the Pile, finding patterns involving punctuation, infinitive constructions, and newlines. In \cref{appendix:toy-model} we analyze a distribution where sentences ending in fullstops can continue with either a capitalized word or a newline. The SVD yields two modes: a uniform mode (weighted average of continuations) and a contrast mode that distinguishes capitalizing from newlines. Contexts that consistently prefer one continuation over the other load heavily on the contrast mode, while ambiguous contexts do not.

To explain in more detail, consider the conditional distribution $q(y|x)$ as defining a linear map from contexts to distributions over continuations. Formally, let $\mathscr{H} = L^2(\Sigma^k, q; \mathbb{R}^{\Sigma})$ be the space of functions from contexts to vectors over tokens, equipped with $\langle f, g \rangle_{\mathscr{H}} = \int \langle f(x), g(x) \rangle_{\mathbb{R}^\Sigma} q(x) \, dx$. The conditional distribution defines an element $\mathcal{C} \in \mathscr{H}$ by $\mathcal{C}(x) = \sum_{y} q(y|x) \, y$. The model parameter $w$ also defines an element $\Phi(w)\in \mathscr{H}$, given by $\Phi(w)(x) := \sum_y \ell_{xy}(w) y$.

Applying SVD to $\mathcal{C}$ yields singular values $s_\alpha$ with right singular vectors $v_\alpha$ (context patterns) and left singular vectors $u_\alpha$ (continuation patterns). These define basis elements $e_{\alpha\beta}$ for $\mathscr{H}$ where $e_{\alpha\beta}(x)(y)$ is the loading of $x$ on the right singular vector $v_\alpha$ times the loading of $y$ on the left singular vector $u_\beta$. We call this the \emph{mode basis}. For more details see \cref{appendix:modes}.

The per-token susceptibility \eqref{eq:per-token-susceptibility} corresponds to a perturbation of the data distribution that upweights the single pair $(x,y)$: we show in \cref{appendix:per-token-perturbation} that $\chi_{xy}$ is proportional to the susceptibility for the perturbation $q' = (1-\epsilon)q + \epsilon \delta_{(x,y)}$. Expanding this point-mass perturbation in the mode basis gives coefficients $s_{\alpha\beta}(xy) = e_{\alpha\beta}(x)(y)$, which we call the \emph{propensities}. These measure the extent to which the mode pair $(\alpha,\beta)$ is responsible for $y$ following $x$ in the true distribution, and depend only on the data distribution $q$, not on the model.

\paragraph{Transform from tokens to modes} Each mode pair $(\alpha,\beta)$ induces a characteristic response
\begin{equation}\label{eq:pure-susceptibility}
\chi^C_{\alpha\beta} \;=\; -\mathrm{Cov}\Big[\phi_C(w),\; \Phi_{\alpha\beta}(w)\Big]
\end{equation}
where $\Phi_{\alpha\beta}(w) = \langle \Phi(w), e_{\alpha\beta} \rangle_{\mathscr{H}}$ projects the model's loss profile onto the basis element. The mode susceptibility depends on the model but not on the particular token sequence. Aggregating over components defines the vector $\chi_{\alpha\beta}$. One can show that (Lemma \ref{lemma:formula_chixy_chialphabeta})
\begin{equation}\label{eq:eta-decomposition}
\chi_{xy} = \sum_{\alpha, \beta} s_{\alpha\beta}(xy) \chi_{\alpha\beta} - \bar{\chi}
\end{equation}
where $\bar{\chi}$ is uniquely determined since $\E_q[\chi_{xy}] = 0$.

The diagonal terms $\chi_{\alpha\alpha}$ are particularly important: they measure the model's response to mode $\alpha$. Early in training, before the model has developed internal structure sensitive to a given pattern, we expect $\| \chi_{\alpha\alpha} \| \approx 0$. As structure specialized for that mode emerges, the response grows: $\| \chi_{\alpha\alpha} \| \gg 0$. Empirical evidence for this picture can be found in the susceptibility-based analysis of the emergence of the induction circuit in \citet{baker2025studyingsmalllanguagemodels}. \citet{wang2025lang2pt5} shows that the \emph{per-pattern susceptibilities} (for e.g. induction patterns or delimiters) are small at the beginning of training and then grow quickly as the model develops; see \cref{appendix:modes} for the relation to $\chi_{\alpha\alpha}$.

\paragraph{Similar sequences have similar susceptibilities} The mode basis derived from $q$ provides a nontrivial sense in which $xy$ can be \emph{similar to but distinct from} $x'y'$. Let us define
\[
s(xy) := \big( s_{\alpha\beta}(xy) \big)_{\alpha,\beta}
\]
in the space $P$ of propensity vectors. This map is injective: if $s(xy) = s(x'y')$ then $x = x', y = y'$. If the distinction between $xy$ and $x'y'$ is subtle we might have $s_{\alpha\beta}(xy) \approx s_{\alpha\beta}(x'y')$ for all modes $\alpha,\beta < \gamma$ below some cutoff (arranging modes by singular value) with the distinction lying in some mode $\ge \gamma$. Thus $s(xy) \approx s(x'y')$ if $y$ follows $x$ for reasons similar to why $y'$ follows $x'$. 

If we let $\iota: P \longrightarrow \mathbb{R}^H$ be the linear transformation sending the basis element corresponding to $\alpha,\beta$ to $\chi_{\alpha\beta}$ then \cref{eq:eta-decomposition} says $\chi_{xy} = \iota( s(xy) ) - \bar{\chi}$. Thus if $y$ follows $x$ for reasons similar to why $y'$ follows $x'$ in the data distribution, we will have $s(xy) \approx s(x'y')$ (note that this has nothing to do with the model) and thus $\chi_{xy} \approx \chi_{x'y'}$. In the other direction, if the distinction between $xy$ and $x'y'$ lies only in some mode $\gamma$ the model doesn't ``understand'' (so $\chi_{\gamma\gamma} \approx 0$) then we can have $\chi_{xy} \approx \chi_{x'y'}$. The model \emph{separates token sequences in susceptibility space according to the modes it understands}.

\paragraph{Clusters as spectral lines.} This analysis clarifies the connection between clusters and the spectroscopy framing from the introduction. In physical spectroscopy, a \emph{spectral line} is a strong, consistent response observed across a family of probes -- for instance, strong absorption or emission. The natural interpretation is that these probes share a significant spectral component at that frequency, and the material has a normal mode that resonates there. Analogously, when we observe a cluster of token sequences $\{x_iy_i\}$ the decomposition \eqref{eq:eta-decomposition} provides a candidate explanation: these sequences share strong propensities $s_{\alpha\alpha}(x_iy_i)$ on some mode $\alpha$, and the model has developed a strong response $\|\chi_{\alpha\alpha}\| \gg 0$ to that mode. The cluster is the observable signature of internal structure specialized for pattern $\alpha$.

In conclusion: the $\chi_{\alpha\beta}$ measure the fundamental response of the model to mode pairs in the data distribution. These quantities are not directly observable since computing the modes would require the full conditional distribution $q(y|x)$, which is unknown. Instead, we observe the responses $\chi_{xy}$ and hypothesize that clusters correspond to dominant modes. The theory provides a principled explanation for why clustering should reveal meaningful structure in a neural network.

\section{Methodology}\label{section:methodology}

\subsection{Data and susceptibilities}\label{section:methodology-data}

We sample 60,000 token sequences from each of 13 Pile subsets \citep{gao2020pile}, yielding 780,000 context-continuation pairs $(x,y)$ for analysis. The subsets span diverse domains including code (\pilesub{github-code}), scientific text (\pilesub{arxiv}, \pilesub{pubmed\_central}), legal documents (\pilesub{freelaw}), and general web text (\pilesub{pile-cc}); the full list is given in \cref{appendix:umap}.

For each token sequence, we compute susceptibility vectors following the methodology of \citet{baker2025studyingsmalllanguagemodels}. Susceptibilities are estimated via preconditioned Stochastic Gradient Langevin Dynamics (pSGLD), which samples from a localized posterior centered at the trained model checkpoint. For each component $C$ (an attention head), the per-token susceptibility $\chi^C_{xy}$ is computed as a covariance between the component's observable and the per-token loss, estimated from posterior samples. Aggregating across all $H$ components yields the susceptibility vector $\chi_{xy} \in \mathbb{R}^H$. Full hyperparameter settings are given in \cref{appendix:susceptibilities-hparams}.

\subsection{Clustering}\label{section:methodology-clustering}

To identify interpretable clusters in the high-dimensional susceptibility data, we adapted a PageRank-based local clustering algorithm from \citet{andersen2009local} to an iterated setting.

\paragraph{Preprocessing.} The susceptibility matrix is standardized column-wise (zero mean, unit variance) and then row-wise (zero mean). The row standardization removes the uniform mode that would otherwise dominate Euclidean distances -- this is equivalent to projecting out the first principal component, which captures variation common to all token sequences rather than distinctions between them.

\paragraph{Graph construction.} We represent the preprocessed susceptibility data as a symmetrized $k$-nearest-neighbor graph, where each node corresponds to a token sequence $xy$ and edges connect nearby points in $\mathbb{R}^H$. Edge weights use a self-tuning radial basis function that adapts to local density variations, ensuring meaningful connectivity.

\paragraph{Local clustering via conductance.} The \emph{conductance} of a vertex subset $S$ measures how well-separated it is from the rest of the graph. Roughly
\[
\text{Cond}(S) = \frac{\text{(total degree of edges leaving } S\text{)}}{\text{(total degree of } S\text{)}}
\]
Low-conductance sets have many internal connections but few external ones: precisely our intuition for a cluster. Following \citet{andersen2009local}, we identify such sets using personalized PageRank: given a seed node, we rank all nodes by their PageRank score and search for a minimum-conductance prefix of this ranking.

\paragraph{Iterative discovery.} A single application of local clustering on a random seed node tends to find the dense main body of the UMAP rather than the peripheral structures we seek. We therefore apply the algorithm iteratively: after each clustering attempt, we remove the discovered nodes from consideration. Early iterations typically encounter portions of the main body (and reject as having too high a conductance). As these are progressively excluded, the algorithm identifies the genuine clusters at the boundary. Full algorithmic details and parameter settings are given in \cref{appendix:clustering}.

\section{Results}\label{section:results}

The central question for susceptibility analysis is whether the geometry of susceptibility space reflects the \emph{computational role} of tokens or merely their surface identity. If susceptibilities capture how the model computes predictions, then tokens with the same identity but different functions should be separated in susceptibility space, while tokens with different identities but the same function should be grouped together. We first test these predictions directly, then survey the hundreds of patterns we discover in Pythia-14M. Each cluster we identify is a ``spectral line'' in the sense of \cref{sec:susceptibilities-modes}: the observable signature of a mode that the model has learned to respond to.

\subsection{Susceptibilities reflect structure}\label{section:susceptibilities-structure}

We test whether susceptibility geometry reflects computational role by examining two complementary phenomena: \emph{context sensitivity} (same $y$ token, different clusters) and \emph{functional grouping} (different tokens, same cluster).

\begin{table}[tp]
\centering
\begin{tabular}{llcl}
\toprule
\textbf{Token} & \textbf{Function/Context} & \textbf{Cluster} & \textbf{Description} \\
\midrule
\multirow{2}{*}{\tok{*}} 
 & Math Operator & \clu{15}, \cluq{57} & Multiplication in algebraic expressions \\
 & Formatting & \clu{504} & Markdown emphasis marker \\
\midrule
\multirow{2}{*}{\tok{.}} 
 & Decimal Point & \clu{127} & Occurring inside numeric sequences \\
 & Sentence End & \clu{198} & Terminating a statement (after variables) \\
\midrule
\multirow{3}{*}{\tok{/}} 
 & Division & \clu{224} & Fraction bar in math problems \\
 & Date Separator & \clu{3} & Separating YYYY/MM/DD \\
 & Path/URL & \clu{322} & Separating directories or domains \\
\midrule
\multirow{2}{*}{\tok{(}} 
 & Function Call & \clu{81} & Opening arguments: $f(x)$ \\
 & List Enumeration & \clu{241} & Opening options: $(a)$ \\
\midrule
\multirow{2}{*}{\tok{\$}} 
 & Math Start & \clu{49} & Transitioning text $\to$ math mode \\
 & Math End & \clu{68} & Transitioning math $\to$ text mode \\
\midrule
\multirow{2}{*}{\tok{00}} 
 & Time & \clu{389} & Minutes field in timestamps ($12:00$) \\
 & Hexadecimal & \clu{65} & Byte value after $0x$ prefix \\
\bottomrule
\end{tabular}
\vspace{1em}
\caption{\textbf{Context sensitivity}. Identical $y$ tokens are separated into distinct clusters based on their functional role. For example, the model maps the token \tok{*} to disjoint regions of susceptibility space depending on whether it is used as a mathematical operator (multiplication) or a formatting marker (markdown italics), indicating distinct computational mechanisms.}
\label{tab:context_split}
\end{table}

\paragraph{Context sensitivity.}
As shown in \cref{tab:context_split}, the model places the same $y$ token into completely disjoint clusters depending on its syntactic role. When \tok{*} appears as a multiplication operator (e.g. in $3*4$), it falls into clusters \clu{15} and \clu{57}, which are dominated by mathematical contexts. However, when \tok{*} functions as a markdown formatting marker (e.g. \emph{*italics*}), it appears in the unrelated cluster \clu{504}.

Similarly, the model distinguishes the opening parenthesis \tok{(} when used in a function call ($f(x)$, \clu{81}) versus a list enumeration ($(a)$, \clu{241}). This suggests that these two uses of \tok{(} involve different computational mechanisms. Simpler examples of this kind of differentiation of susceptibility vectors by $y$ token sense were already noticed in a 3M parameter model by \citet[Appendix E.4]{baker2025studyingsmalllanguagemodels}.

\begin{table}[tp]
\centering
\begin{tabular}{lcl}
\toprule
\textbf{Grouped Tokens} & \textbf{Cluster} & \textbf{Context Trigger} \\
\midrule
\tok{h}, \tok{u}, \tok{y}, \tok{o}, \tok{w} & \clu{58} & Following multiplication \tok{*} \\
\midrule
\tok{~will}, \tok{~must}, \tok{~should}, \tok{~would} & \clu{63} & Following a subject noun/pronoun \\
\midrule
\tok{~second}, \tok{~third}, \tok{~remainder} & \clu{223} & Following \tok{~the} in math problems \\
\midrule
\tok{~to}, \tok{~on}, \tok{~by}, \tok{~in} & \clu{11} & Following passive verbs \\
\midrule
\tok{4}, \tok{5}, \tok{6}, \tok{7} & \clu{143} & Following exponent operator \tok{**} \\
\midrule
\tok{~In}, \tok{~It}, \tok{~With}, \tok{~Since} & \clu{509} & Capitalized words at line start \\
\bottomrule
\end{tabular}
\vspace{1em}
\caption{\textbf{Abstraction over token identity.} Distinct $y$ tokens are grouped into single clusters when they share a functional role. For instance, the model treats disparate variable names (\tok{h}, \tok{u}, \tok{y}) as computationally equivalent in the context of multiplication (\clu{58}). This indicates that the susceptibility vector captures the abstract category (e.g., ``variable'' or ``exponent'') rather than just the token.}
\label{tab:functional_grouping}
\end{table}

\paragraph{Functional Grouping.}
Conversely, \cref{tab:functional_grouping} demonstrates that the model abstracts away specific token identities when they share a functional role. For example, the tokens \tok{h}, \tok{u}, \tok{y}, \tok{o}, and \tok{w} are grouped into a single cluster (\clu{58}) specifically when they appear as variables following a multiplication operator.

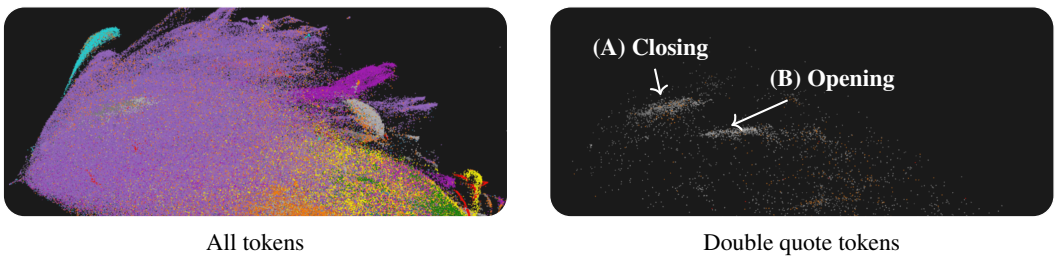
\begin{figure}[tp]
    \centering
    \begin{tikzpicture}
        \begin{scope}
            \clip[rounded corners=8pt] (0,0) rectangle (0.48\textwidth,0.2\textwidth);
            \node[anchor=south west,inner sep=0] (image1) at (0,0) 
                {\includegraphics[width=0.48\textwidth]{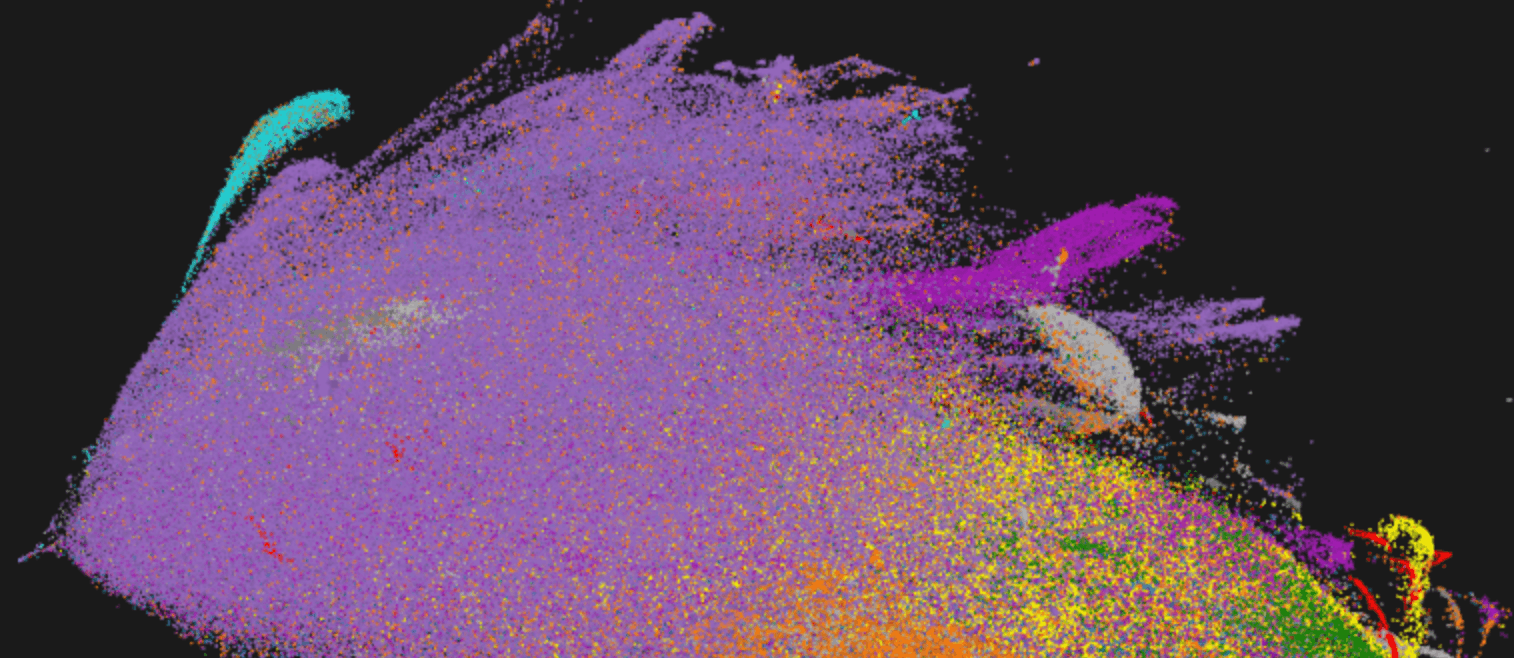}};
        \end{scope}
        \draw[white, line width=0.5pt, rounded corners=8pt] (0,0) rectangle (0.48\textwidth,0.2\textwidth);
        
        \begin{scope}[shift={(0.52\textwidth,0)}]
            \clip[rounded corners=8pt] (0,0) rectangle (0.48\textwidth,0.2\textwidth);
            \node[anchor=south west,inner sep=0] (image2) at (0,0) 
                {\includegraphics[width=0.48\textwidth]{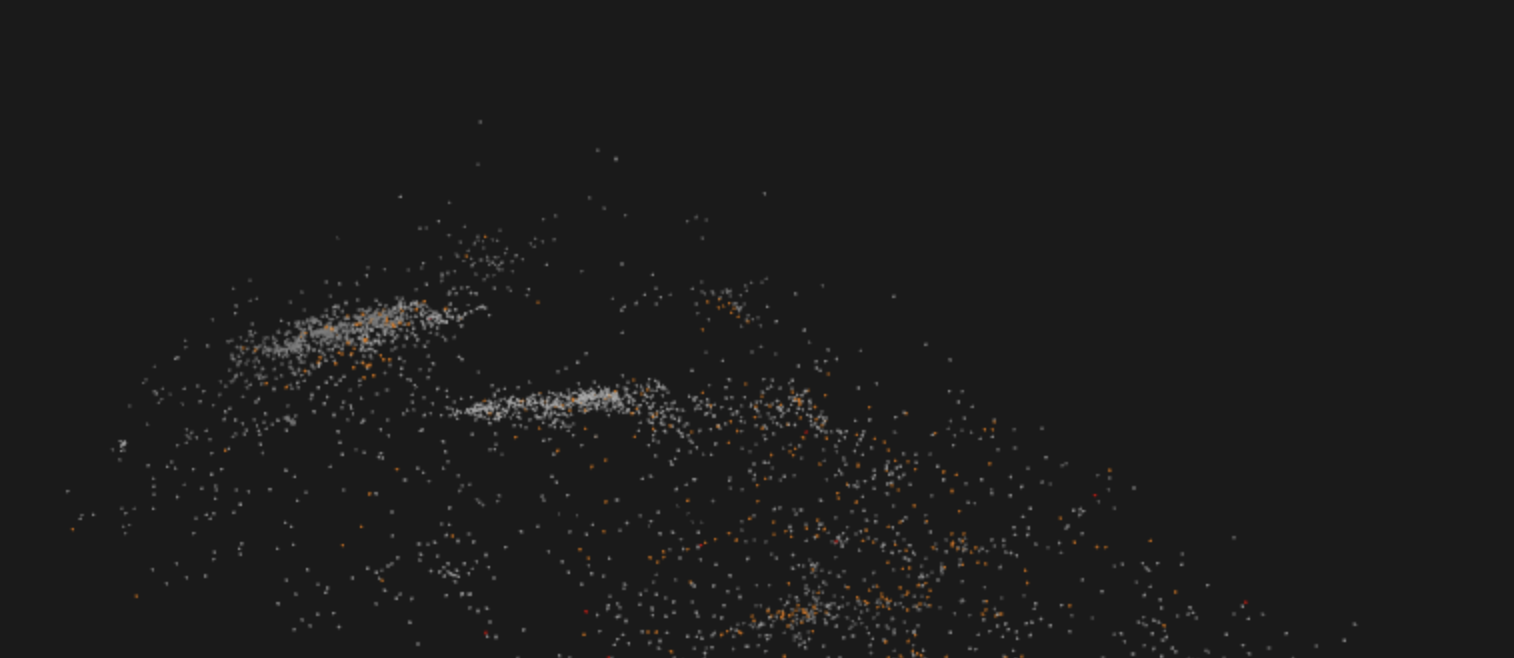}};
        \end{scope}
        \draw[white, line width=0.5pt, rounded corners=8pt] (0.52\textwidth,0) rectangle (\textwidth,0.2\textwidth);
        
        \begin{scope}[shift={(0.52\textwidth,0)}]
            \pgfmathsetmacro{\imgw}{0.48*\textwidth}
            \pgfmathsetmacro{\imgh}{0.2*\textwidth}
            
            \node[white, font=\bfseries\small] (A) at (0.2*0.48\textwidth,0.8*0.2\textwidth) {(A) Closing};
            \draw[->,thick,white] (A) -- (0.22*0.48\textwidth,0.58*0.2\textwidth);
            
            \node[white, font=\bfseries\small] (B) at (0.56*0.48\textwidth,0.65*0.2\textwidth) {(B) Opening};
            \draw[->,thick,white] (B) -- (0.36*0.48\textwidth,0.435*0.2\textwidth);
        \end{scope}
        
        \node[below=0.1cm, font=\small] at (0.24\textwidth,0) {All tokens};
        \node[below=0.1cm, font=\small] at (0.76\textwidth,0) {Double quote tokens};
        
    \end{tikzpicture}
    
\vspace{0.2em}

    \caption{
        \textbf{Double quote token clusters.} 
        \emph{Left:} Full low-dimensional representation of susceptibility vectors computed from Pythia-14M, with points colored by pattern type. 
        \emph{Right:} Filtered view showing only tokens which, when decoded, contain double quotes (e.g. \tok{"}, \tok{."}). 
        (A) Tokens containing double quotes as closing a quotation (\clu{69}, \clu{182}, \clu{372}, \clu{469}).
        (B) Tokens containing double quotes as opening a quotation (\clu{168}, \clu{174}). 
    }
    \label{fig:apostrophe_clusters}
\end{figure}

A demonstration of both phenomena can be found in \cref{fig:apostrophe_clusters}. On the one hand, \tok{"} can both open and close quotations, and we see that these uses are separated into distinct clusters (or rather, groupings of clusters). On the other hand, the closing quotation cluster \clu{182} contains multiple distinct forms of the closing quotation mark \tok{"}, \tok{,"}, \tok{'}, \tok{."}. The region marked (A) in the UMAP also contains \clu{350} which consists of \tok{”} and variations (\tok{.”}, \tok{”.} and so on). Interestingly the region (A) is mostly closing double quotes in normal text: \clu{385}, which is mostly \tok{"} tokens \emph{in code} and HTML, is located elsewhere in the UMAP.

These examples demonstrate that susceptibility geometry reflects internal structure: the model separates tokens by function, not just identity, and groups functionally equivalent tokens together.

\subsection{Patterns in Pythia-14M}\label{section:patterns-catalog}

\definecolor{tokyellow}{RGB}{255,242,170}


\begin{figure*}[p]
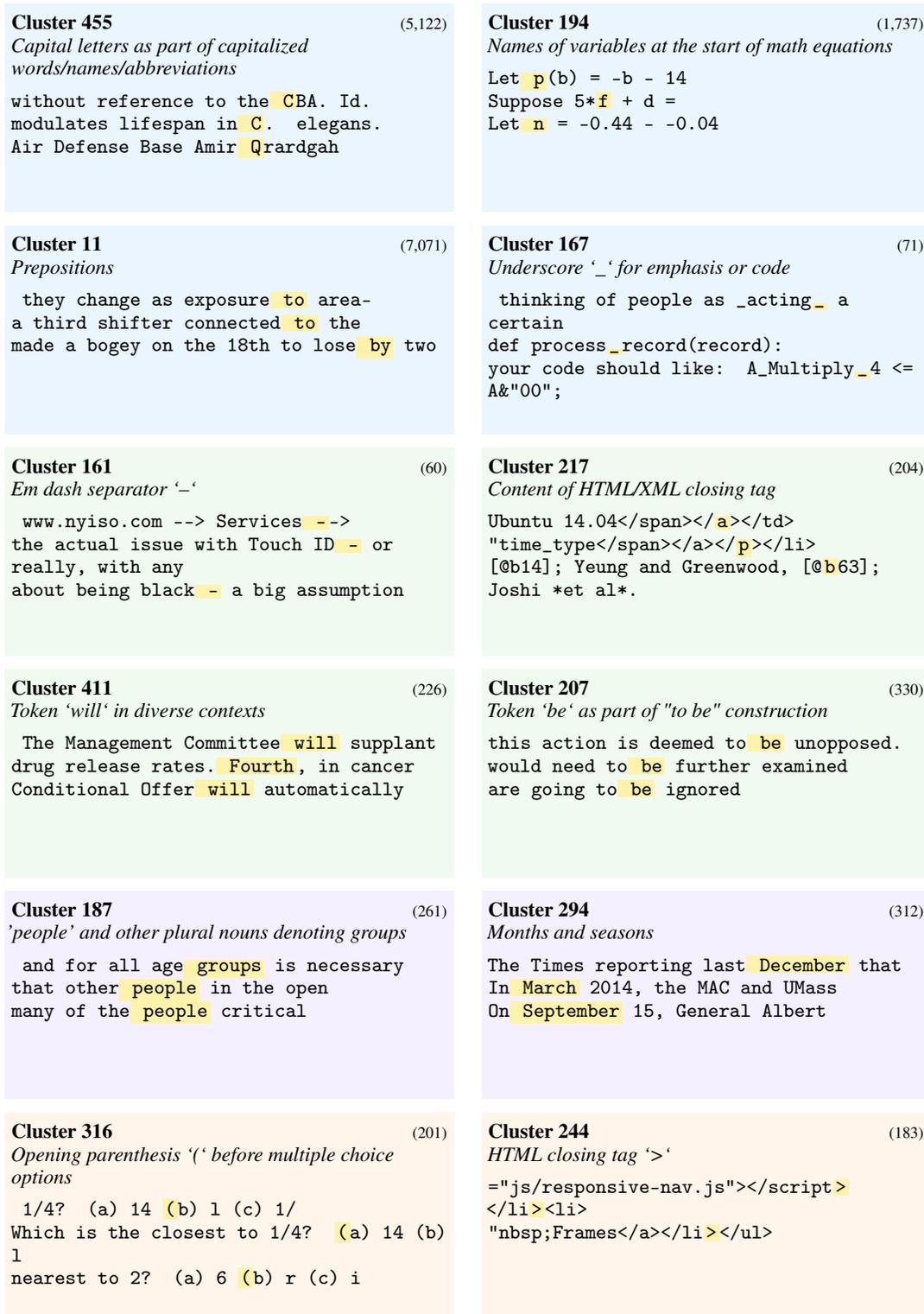

\centering
\setlength{\tabcolsep}{6pt}
\renewcommand{\arraystretch}{1.15}

\begin{tabular}{cc}

\clustercard
{Cluster 455}
{5,122}
{Capital letters as part of capitalized words/names/abbreviations}
{without reference to the\finaltok{ C}BA.  Id.}
{modulates lifespan in\finaltok{ C}. elegans.}
{Air Defense Base Amir\finaltok{ Q}rardgah}
{clusterBlue}
&
\clustercard
{Cluster 194}
{1,737}
{Names of variables at the start of math equations}
{Let\finaltok{ p}(b) = -b - 14}
{Suppose 5*\finaltok{f} + d = }
{Let\finaltok{ n} = -0.44 - -0.04}
{clusterBlue}
\\[8pt]
\clustercard
{Cluster 11}
{7,071}
{Prepositions}
{ they change as exposure\finaltok{ to} area- }
{ a third shifter connected\finaltok{ to} the}
{made a bogey on the 18th to lose\finaltok{ by} two  }
{clusterBlue}
&

\clustercard
{Cluster 167}
{71}
{Underscore `\_` for emphasis or code }
{ thinking of people as \_acting\finaltok{\_} a certain}
{def process\finaltok{\_}record(record):}
{your code should like: A\_Multiply\finaltok{\_}4 <= A\&"00";
}
{clusterBlue}
\\[8pt]
\clustercard
{Cluster 161}
{60}
{Em dash separator `--`}
{ www.nyiso.com ---> Services\finaltok{ --}->}
{ the actual issue with Touch ID\finaltok{ --} or really, with any}
{about being black\finaltok{ --} a
big assumption}
{clusterGreen}
&
\clustercard
{Cluster 217}
{204}
{Content of HTML/XML closing tag}
{Ubuntu 14.04</span></\finaltok{a}></td>}
{"time\_type</span></a></\finaltok{p}></li>}
{ [@b14]; Yeung and Greenwood, [@\finaltok{b}63]; Joshi *et al*.}
{clusterGreen}
\\[8pt]

\clustercard
{Cluster 411}
{226}
{Token `will` in diverse contexts}
{ The Management Committee\finaltok{ will} supplant }
{drug release rates.\finaltok{ Fourth}, in cancer}
{Conditional Offer\finaltok{ will} automatically}
{clusterGreen}
&
\clustercard
{Cluster 207}
{330}
{ Token `be` as part of "to be" construction}
{this action is deemed to\finaltok{ be}
unopposed.}
{  
would need to\finaltok{ be} further examined }
{ are going to\finaltok{ be} ignored }
{clusterGreen}
\\[8pt]
\clustercard
{Cluster 187}
{261}
{'people' and other plural nouns denoting groups}
{ and for all age\finaltok{ groups} is necessary}
{that other\finaltok{ people} in the open}
{ many of the\finaltok{ people} critical}
{clusterPurple}
&

\clustercard
{Cluster 294}
{312}
{Months and seasons}
{The
Times reporting last\finaltok{ December} that }
{ In\finaltok{ March} 2014, the MAC and UMass}
{On\finaltok{ September} 15, General Albert}
{clusterPurple}
\\[8pt]
\clustercard
{Cluster 316}
{201}
{Opening parenthesis `(` before multiple choice options}
{ 1/4?  (a) 14  \finaltok{(}b) l  (c) 1/}
{ Which is the closest to 1/4?  \finaltok{(}a) 14  (b) l  }
{ nearest to 2?  (a) 6  \finaltok{(}b) r  (c) i
}
{clusterOrange}
&
\clustercard
{Cluster 244}
{183}
{HTML closing tag `>`}
{="js/responsive-nav.js"></script\finaltok{>}}
{</li\finaltok{>}<li>}
{"nbsp;Frames</a></li\finaltok{>}</ul>}
{clusterOrange}
\\

\end{tabular}

\caption{
\textbf{Examples of clusters}. For each, the final token is highlighted in a selection of contexts. These clusters were selected for either having high entropy in final or penultimate tokens (blue), completely at random (green), as examples of semantic/meaning based clusters (purple), or as examples of syntactic/grammatical clusters (orange). The size of the cluster is shown in the top right.
}
\label{fig:cluster_cards}
\end{figure*}

Having established that susceptibilities reflect internal structure, we now survey the patterns Pythia-14M has learned to respond to. \cref{fig:cluster_cards} shows example token-context pairs from a selection of clusters, illustrating the range of patterns: from universal linguistic structures (prepositions, sentence boundaries) to dataset-specific conventions (LaTeX math mode, HTML tags, Python syntax). Clusters also vary in their level of abstraction, from low-level syntactic patterns (a specific token in a specific position) to higher-level patterns that abstract over token identity to capture functional roles; see \cref{appendix:cluster-levels} for discussion.

Our clustering algorithm (\cref{section:methodology-clustering}) identifies 510 clusters in the susceptibility data, listed in \cref{table:pythia14m_cluster_list}. Cluster labels were generated by an LLM (Claude Opus 4.5) and manually reviewed for correctness; for auto-interpretability methods see \citet{bills2023, paulo2025automaticallyinterpretingmillionsfeatures}. For a higher-level taxonomy organizing these clusters by type, see \cref{section:results_pythia_14m_taxonomy}.

\cref{appendix:detailed_examples} shows 30 randomly selected clusters with longer token-context pairs, as well as an evaluation of whether or not they adhere to a common theme. Based on evaluating this random sample, we find that the majority of clusters have more than 90\% of their token-context pairs well described by a single theme.

\subsection{Networks of Linked Clusters}\label{section:networks_linked}

Modes of the data distribution (\cref{sec:susceptibilities-modes}) are defined for contexts of a fixed length $k$, but modes for different $k$ can be \emph{linked}. Consider
\[
\underset{x}{\underbrace{\overset{x_{<i}}{\overbrace{x_1 x_2 \cdots x_{i-1}}} x_i \cdots x_k}} y
\]
where $x_i \in \Sigma$ and $y \in \Sigma$. The reason $y$ follows $x$ may depend on both (i) why $x_i$ follows $x_{<i}$ and (ii) the presence of $x_i$ in the context. These clusters form a kind of \emph{network} where we imagine an edge from cluster $\mathcal{C}$ to cluster $\mathcal{C'}$ if the kind of token sequences $xy$ that end up in $\mathcal{C'}$ often have contexts of the form $x = u x'y' v$ where $x'y' \in \mathcal{C}$.

For example, consider the token sequence
\[
\cdots \overset{\clu{18}, \cluq{151}, \cluq{242}, \cluq{506}}{\tok{<}}\underset{\clu{149}}{\tok{a}}\underset{\clu{291}}{\tok{~href}}\underset{\clu{251}}{\tok{="}}\tok{http}\underset{\clu{195},\cluq{286}}{\tok{://}}\underset{\clu{46}}{\tok{www}} \quad \cdots \quad \overset{\clu{190}, \cluq{266}}{\tok{</}}\underset{\clu{217}}{\tok{a}}\overset{\clu{52}, \cluq{170}, \cluq{244}}{\tok{>}}
\]
To predict \tok{</} we need to know about the early tokens that ``open'' the HTML tag, and to know the correct closing tag \tok{a} we need to know about the nature of the opening tag, which itself may be predictable from the earlier context (not shown). 

There is a sequence of clusters corresponding to different positions in the HTML syntax, distinguishing between the opening syntax, the tag identity, and its attributes. The sequence begins with the opening angle bracket \tok{<} (\clu{18}, \cluq{151}, \cluq{242}, \cluq{506}), followed by clusters for tag names like \tok{a}, \tok{div}, \tok{span}, or \tok{li} (\clu{149}, \cluq{377}, \cluq{426}). If attributes are present, there are clusters for keywords like \tok{class} (\clu{330}) or \tok{href} (\clu{197}), followed by the assignment operator \tok{="} (\clu{251}). The opening tag is completed by \tok{>} or compound tokens (\clu{52}, \cluq{56}, \cluq{170}). Finally, a separate set of clusters covers the beginning of the closing tag \tok{</} (\clu{190}, \cluq{266}), the content of the closing tag (e.g. \tok{a} in \tok{</}\tok{a}\tok{>}, \clu{217}), and then \tok{>} (\clu{52}, \cluq{170}, \cluq{244}). A very similar network of SAE features was noted in \citet{bricken2023monosemanticity}.


For additional examples of linked clusters see \cref{appendix:additional_linked_clusters}.

\subsection{Comparison to SAEs}

The susceptibility clusters shown in \cref{tab:context_split,tab:functional_grouping} are recognizably similar to the kinds of patterns found in sparse auto-encoder (SAE) features \citep{yun2021transformer, huben2024sparse, bricken2023monosemanticity}. To validate that our clusters capture structure also identified by other interpretability methods, we give a quantitative analysis of this similarity.

We compare susceptibility clusters in Pythia-14M (\cref{section:cluster_table}) to SAE features for Pythia-70M computed in \citet{lan2025sparse}. We use the residual stream features from layers $2$-$4$ and define a susceptibility cluster to \emph{match} with an SAE feature if the feature has an unusually large activation on the $y$ tokens across the cluster (and has a low baseline activation on the $y$ tokens of a random selection of other clusters). For full details of the methodology see \cref{section:sae-methodology}. \textbf{We find that out of 510 susceptibility clusters, 259 (50.8\%) have a matched SAE feature}.

We caution against overinterpreting the match rate. SAE features are designed to be sparse but not singleton -- multiple features typically fire on any given input, reflecting the hypothesis that inputs are explained by combinations of underlying features. By contrast, our clustering assigns each token pair to exactly one cluster. A token pair explained by multiple overlapping patterns may match an SAE feature for only one of those patterns, or may be assigned to a cluster capturing a different pattern entirely. The comparison validates cross-method consistency but does not establish a one-to-one correspondence between the two representations.

While our quantitative comparison is to the Pythia-70M SAEs of \citet{lan2025sparse}, we note that there is also qualitative similarity to the features in \citet{bricken2023monosemanticity} which studies a one-layer transformer with a 512-neuron transformer MLP. The 4096 features (A/1) released there closely resemble the relatively low-level syntactic and structural features of language captured by our clusters with many ``flavors'' of \tok{.}, \tok{,}, \tok{\textbackslash n}, \tok{~and}, \tok{~of} and parts of mathematics and code.

\subsection{Scaling to larger models}\label{section:scaling}

A natural question is whether the structure we identify in Pythia-14M also exists in larger Pythia models. To test this, we measured the conductance of Pythia-14M clusters when the graph structure is defined by susceptibilities from larger models. If a cluster corresponds to a genuine pattern, tokens in that cluster should remain nearby in susceptibility space even when susceptibilities are computed from a different model. We find this is indeed the case: Pythia-14M clusters have conductance significantly below 1 in all larger models tested, from 31M to 1.4B parameters (\cref{fig:conductance_scatter} in \cref{appendix:larger-models}). Random vertex sets, by contrast, have conductance $\approx 1$. This provides evidence that the clusters in Pythia-14M reflect structure in the data that larger models also learn to respond to.

Clustering on larger models yields fewer clusters (241--358 compared to 510 for Pythia-14M; see \cref{appendix:larger-models}). This may have multiple causes: increased difficulty sampling higher-dimensional parameter spaces, noisier Euclidean metrics in higher dimensions, and the fact that larger models respond to more modes may all contribute to diffusing the cluster structure. Techniques such as dictionary learning on susceptibility vectors may be necessary to recover more structure in larger models.

\section{Related work}

\paragraph{Prior susceptibility work} Susceptibility analysis was introduced by \citet{baker2025studyingsmalllanguagemodels} and applied to study development over training by \citet{wang2025lang2pt5}, both focusing on a 3M parameter attention-only transformer. Related ideas appear in the Bayesian influence function of \citet{kreer2025bayesian} and the loss kernel of \citet{adam2025losskernel}, which also use SGLD sampling to compute covariances for data attribution and interpretability. The present paper extends this work in several directions: we develop the theoretical connection between susceptibilities and modes of the data distribution (\cref{sec:susceptibilities-modes}), introduce a systematic clustering methodology based on conductance (\cref{section:methodology-clustering}), and apply these tools to Pythia-14M, which despite being only modestly larger exhibits significantly richer structure -- yielding 510 interpretable clusters compared to the handful of patterns identified in prior work.

\paragraph{Quanta vs modes} The theoretical contribution of this paper is organized around \emph{modes} of \citet{modes2}. A very similar idea is the \emph{quanta} of \citet{michaud2023the} which are defined in terms of the singular value decomposition of a matrix whose rows are gradients of the loss on a set of tokens in context $xy$. Spectral clustering based on this data is related to clustering by susceptibility vectors (in the sense that if the loss landscape were nondegenerate, there would be a formal relationship between them). While \citet{michaud2023the} do not provide the full set of clusters they do note that ``most clusters involve the prediction of the same token'' (i.e. $y$ in our notation). There are interesting exceptions, including newlines in length limited text (our \textbf{C465}).

\paragraph{Latent concepts} Given that deep neural networks learn representations, it is natural to suppose that the activations of neural language models should group words together in high-dimensional space based on syntactic and semantic relationships, in such a way that clusters represent latent concepts \citep{mikolov2013efficientestimationwordrepresentations, reif2019visualizing, hewitt2019structural}. \citet{dalvi2022discovering} find $1000$ clusters based on activations in a BERT model (110M parameters) and they give 183 labeled clusters in \citet[Appendix B.3]{dalvi2022discovering}. For a survey of methods before SAEs see \citet{10.1162/tacl_a_00519}.

\paragraph{SAEs} \citet{yun2021transformer} apply dictionary learning to the activations of a $12$-layer BERT model and find transformer factors that include low-level word features but also more complex high-level features involving multiple distinct tokens. \citet{huben2024sparse} mainly study Pythia-70M and Pythia-410M but in the paper do not give details on the SAE features discovered. \citet{marks2025sparse} make use of the clustering technique (based on gradients) from \citet{michaud2023the} and then use SAE features evaluated on these clusters to discover circuits. It is known that SAE features capture high level cross-lingual representations of grammatical concepts such as \emph{plural} and \emph{past tense} \citep{brinkmann2025largelanguagemodelsshare}. We note that our clusters are derived from 780,000 token sequences whereas e.g. the SAEs in \citet{bricken2023monosemanticity} were trained with 8B data points.

\paragraph{Limitations} Our study focuses exclusively on the Pythia model family trained on the Pile, so our findings may not generalize to other architectures or training distributions. The clustering results depend on hyperparameter choices (conductance threshold, $k$-nearest neighbors); while we find the results robust to reasonable variations, different settings could yield different cluster boundaries. Our SAE comparison uses features from Pythia-70M rather than Pythia-14M due to availability, which may affect the match rate. Finally, while we provide evidence that Pythia-14M clusters persist in larger models (\cref{section:scaling}), susceptibility analysis has not yet been demonstrated at frontier scale.

\section{Conclusion}

This paper establishes susceptibility analysis as a form of \emph{spectroscopy} for neural networks: a method for inferring internal structure from how the model responds to patterns in the data distribution. The 510 interpretable clusters we identify in Pythia-14M -- 50.8\% of which match SAE features -- reflect the model's learned responses to regularities ranging from universal linguistic structures to dataset-specific patterns. The theoretical framework (\cref{sec:susceptibilities-modes}) provides a principled connection between these patterns, formalized as modes of the data distribution, and the model's internal responses.

Beyond interpretation, susceptibilities may also enable \emph{intervention}: if susceptibilities measure the first-order response of structural coordinates to shifts in the data distribution, this relationship can be inverted to find data modifications that produce desired structural changes \citep{patterning}. We leave exploration of this direction to future work.

Our detailed results focus on Pythia-14M, but we provide evidence that the patterns identified persist in larger models. When we measure the conductance of Pythia-14M clusters using susceptibilities from models up to 1.4B parameters, the clusters remain coherent (\cref{section:scaling}, \cref{appendix:larger-models}). The question of whether susceptibility analysis scales to frontier systems remains open, but these results are encouraging. The hundreds of interpretable clusters we find, and the theoretical framework connecting them to modes of the data distribution, suggest that spectroscopy may be a useful complement to existing interpretability methods.


\bibliographystyle{abbrvnat}
\bibliography{references}

\newpage
\appendix

\section*{Appendix Overview}

The appendix provides supplementary material organized as follows:
\begin{itemize}
    \item \textbf{\cref{appendix:clustering}--\cref{appendix:susceptibilities-hparams}:} Details of the clustering algorithm and susceptibility hyperparameters.
    \item \textbf{\cref{appendix:modes}--\cref{appendix:toy-model}:} Theoretical material on susceptibilities, including the mode decomposition, per-token susceptibilities as data perturbations, per-pattern susceptibilities, and a toy model illustrating mode structure.
    \item \textbf{\cref{appendix:token-definitions}:} Definitions of token pattern categories, with distribution across datasets.
    \item \textbf{\cref{appendix:umap}:} UMAP methodology and hyperparameters.
    \item \textbf{\cref{appendix:larger-models}:} Results from applying susceptibility analysis to larger Pythia models (31M--1.4B parameters).
    \item \textbf{\cref{appendix:additional_linked_clusters}:} Additional examples of linked cluster networks including code blocks, mathematical reasoning, and LaTeX typesetting.
    \item \textbf{\cref{section:sae-methodology}:} Methodology for comparing susceptibility clusters to SAE features.
    \item \textbf{\cref{appendix:gaussian-posterior}:} Comparison of susceptibilities to a simpler Gaussian baseline.
    \item \textbf{\cref{section:cluster_table}--\cref{appendix:detailed_examples}:} Complete list of 510 Pythia-14M clusters, taxonomy, and detailed examples.
\end{itemize}

\section{Experiment Details}
\subsection{Details of the clustering algorithm}\label{appendix:clustering}

\subsubsection{Graph Construction}

The susceptibility matrix $X \in \mathbb{R}^{n \times H}$, has different scales across columns, as well high degree of correlation between columns, both of which make Euclidean distance in $n$-dimensional space a poor basis for clustering.  We preprocess it by standardizing each column to have zero mean and unit variance, then shifting each row to have zero mean. 

We then construct a symmetrized $k$-nearest-neighbor graph with $k = 45$. Edges are weighted using a self-tuning radial basis function following \citet{zelnik2004self}. For two connected points $x$ and $y$, the edge weight is
\[
w(x,y) = \exp\left(-\frac{\|x - y\|^2}{\sigma_x \cdot \sigma_y}\right)
\]
where $\sigma_x$ is the distance from $x$ to its $k$-th nearest neighbor. This self-tuning accounts for the increased distance between points in higher dimensions, and enforces a roughly equal average degree for distance graphs of datasets of varying dimensions.

\subsubsection{Conductance and Local Clustering}

For a proper subset $S$ of vertices, the \emph{conductance} is defined as
\[
\text{Cond}(S) = \frac{w(S, \bar{S})}{\min(\text{vol}(S), \text{vol}(\bar{S}))}
\]
where $w(S, \bar{S})$ is the total weight of edges crossing the boundary and $\text{vol}(S)$ is the sum of degrees within $S$. Conductance lies in $[0,1]$, with low values indicating well-separated clusters. 

The Andersen-Chung-Lang (ACL) Local clustering algorithm \citep{andersen2009local} identifies low conductance sets containing a given seed by ordering all points in the graph with personal PageRank of that seed, and identifying low conductance sets among the prefixes of that ordering.  

The personal PageRank of a seed, with teleport probability $\alpha$,  is the stable state of a random walk on the graph which, at any step, returns to the seed with probability $\alpha$.  Though this PageRank is determined by the global structure of the graph, it can be efficiently computed via the ``push'' approximation \citep{andersen2009local}, whose runtime does not depend on the number of points in the graph. The push approximation depends on a tolerance parameter $\epsilon$. It is designed to only explore nodes of the graph with predicted rank above $\epsilon$. We selected teleportation parameter $\alpha=0.001$ and $\epsilon=10^{-7}$ to cause push personal PageRank to assign nonzero rank to roughly 10k nodes.

\subsubsection{Iterative Cluster Discovery}


Direct application of local clustering suffers from several drawbacks. On the one hand seeds in the dense main body produce candidate clusters encompassing most of the graph: if the seed is central in a large (i.e. much greater than 10k) point cloud of roughly uniform density, then the potential cluster will usually include every point with positive PageRank, since the visited points form a rough ball around the seed, and the volume increases faster than the weight of outgoing edges. On the other hand, if the seed is isolated, the potential cluster will be very small.

We address this with an iterative procedure (\cref{alg:clustering}) that progressively removes visited nodes. The aim is to find reasonable sized clusters while excluding excessively large clusters and isolated points.

Early seed selection almost always chooses points within a large main body, but repeated steps of the algorithm remove such points relatively quickly. After they are gone, the algorithm categorizes the remaining points into clusters easily.

The parameter settings are summarized in \cref{tab:clustering-params}.

\begin{algorithm}[t]
\caption{Iterative Conductance-Based Clustering}
\label{alg:clustering}
\begin{algorithmic}[1]
\Require Susceptibility vectors $\{\chi_i\}_{i=1}^n \subset \mathbb{R}^H$
\Ensure List of clusters $\mathcal{C}$
\State Construct $k$-NN graph $G$ with self-tuning weights
\State $\textit{Unvisited} \gets \{1, \ldots, n\}$
\State $\mathcal{C} \gets [\,]$
\While{$|\textit{Unvisited}| > 0.001 \cdot n$}
    \State Sample seed $s$ uniformly from \textit{Unvisited}
    \State Compute personalized PageRank $\pi$ from $s$ with parameters $(\alpha, \epsilon)$
    \State Let $R \gets \{i : \pi_i > 0\}$ be nodes with positive rank
    \State Find minimum-conductance prefix $P$ of nodes sorted by $\pi$
    \If{$|P| > 0.99 \cdot |R|$} \Comment{Main body detected}
        \State $\textit{Unvisited} \gets \textit{Unvisited} \setminus R$
    \ElsIf{$|P| < 20$} \Comment{Isolated point}
        \State $\textit{Unvisited} \gets \textit{Unvisited} \setminus \{s\}$
    \Else \Comment{Valid cluster found}
        \State Append $P$ to $\mathcal{C}$
        \State $\textit{Unvisited} \gets \textit{Unvisited} \setminus P$
    \EndIf
\EndWhile
\State \Return $\mathcal{C}$
\end{algorithmic}
\end{algorithm}

The key insight is that early iterations almost always select seeds within the large, dense main body of the distribution. These attempts are rejected (the candidate cluster spans nearly all reachable nodes), but the rejection removes those nodes from future consideration. Once the main body is exhausted, subsequent seeds land in the peripheral structures, which the algorithm identifies as genuine low-conductance clusters.

\begin{table}[ht]
\centering
\begin{tabular}{@{}lll@{}}
\toprule
\textbf{Parameter} & \textbf{Value} & \textbf{Role} \\
\midrule
$k$ (neighbors) & 45 & Graph connectivity \\
$\alpha$ (teleport) & 0.001 & PageRank localization \\
$\epsilon$ (tolerance) & $10^{-7}$ & PPR approximation precision \\
Main body threshold & 0.99 & Reject if best prefix contains >99\% of ranked nodes \\
Minimum cluster size & 20 & Reject isolated points \\
Termination threshold & 0.001 & Stop when $<$0.1\% nodes remain \\
\bottomrule
\end{tabular}
\vspace{0.5em}
\caption{Clustering algorithm parameters.}
\label{tab:clustering-params}
\end{table}

\subsubsection{Local Clustering}

Its steps are
\begin{enumerate}
    \item To start, initialize \textit{UnvisitedNodes} as a set of all nodes in the graph, \textit{UnclusteredNodes} as an empty set, and \textit{Clusters} as an empty list.
    \item While \textit{UnvisitedNodes} contains more than .1\% of all nodes, choose a random node $x$ from the \textit{UnvisitedNodes}. 
    \item Determine the potential cluster of $x$ using ACL
    \item If the potential cluster contains more than 99\% of all nodes given positive rank by that round PPR, then we did not find a cluster starting from $x$.  Add all nodes with positive rank to \textit{UnclusteredNodes} and remove them all from \textit{UnvisitedNodes}. Return to step 2
    \item If the potential cluster has size less than 20, then $x$ is an isolated point. Remove $x$ from \textit{UnvisitedNodes} and return to step 2.
    \item If neither of the above is true, the potential cluster found is acceptable. Append the potential cluster (as a set) as a new element to \textit{Clusters} and remove every element in the potential cluster from \textit{UnvisitedNodes}. Return to Step 2    
\end{enumerate}

\subsection{Susceptibilities hyperparameters} \label{appendix:susceptibilities-hparams}

We compute susceptibilities similarly to \citet{baker2025studyingsmalllanguagemodels}, with a few modifications to account for the increased size of the model. 

Susceptibilities were computed using preconditioned Stochastic Gradient Langevin Dynamics (pSGLD) that used the RMSProp algorithm as a preconditioner. For Pythia-14m we used as hyperparameters $\gamma=300$, $n\beta = 3$, $\varepsilon = 1\text{e}-5$, batch size $16$, $4$ chains, and $100$ draws, with $55$ steps taken between each draw.

For the larger models, the hyperparameters are given in \cref{tab:hyperparams}.

\begin{table}[t]
\centering
\caption{Experimental hyperparameters for each model.}
\label{tab:hyperparams}
\small
\begin{tabular}{lccccccc}
\toprule
Model & $\gamma$ & $n\beta$ & $\varepsilon$ & Batch size & Chains & Draws & Steps Between Draws \\
\midrule
Pythia-14M   & 300 & 3 & 1e-5 & 16 & 4 & 100 & 55 \\
Pythia-31M   & 300 & 3 & 1e-5  & 16 & 4 & 100 & 120 \\
Pythia-70M   & 300 & 3 & 1e-5  & 16 & 4 & 100 & 200 \\
Pythia-160M  & 300 & 3 & 1e-5  & 16 & 4 & 100 & 140 \\
Pythia-410M  & 300 & 10 & 1e-5  & 16 & 4 & 100 & 160 \\
Pythia-1.4B  & 300 & 10 & 1e-5  & 16 & 4 & 100 & 160 \\
\bottomrule
\end{tabular}
\end{table}

For additional details on the theory and implementation of susceptibilities used in this paper, please refer to the appendices of \citet{baker2025studyingsmalllanguagemodels}.

\subsection{Compute Use}\label{appendix:compute}

Experiments were conducted on several models from the Pythia family (14M, 70M, 160M, 410M, and 1.4B parameters). For each model, the experiment was parallelized across multiple compute nodes, with each node equipped with four NVIDIA H200 GPUs.

The number of nodes allocated per model was as follows: 10 nodes for the 1.4B-parameter model, 10 nodes for the 410M-parameter model, 4 nodes for the 160M-parameter model, 4 nodes for the 70M-parameter model, and 4 nodes for the 14M-parameter model. Depending on model size, this corresponds to between 16 and 40 GPUs used concurrently.

Each experimental run completed within a maximum of five days of wall-clock time. Exact GPU-hour usage was not logged; however, total compute scaled approximately linearly with the number of nodes allocated per model.

\section{Susceptibilities}

We define the susceptibility $\chi^C_{xy}$ for a component $C$ of a neural network used to predict the next token $y$ given a context $x$, and explain how to think intuitively about what these scalar values mean. For full details see \citet{baker2025studyingsmalllanguagemodels}.

We consider sequence models $p(y|x,w)$ that predict tokens $y \in \Sigma$ given sequences of tokens $x \in \Sigma^k$ for various $1 \le k \le K$ (called \emph{contexts}) where $K$ is the maximum context length and $\Sigma$ is the set of tokens. The true distribution of token sequences $(x,y)$ is denoted $q(x,y)$. The sequence models we have in mind are transformer neural networks, where $w \in W$ is the vector of weights. We set $X$ to be the disjoint union of $\Sigma^k$ over $1 \le k \le K$ and $Y = \Sigma$.

Given a dataset $D_n = \{(x_i,y_i)\}_{i=1}^n$, drawn i.i.d. from $q(x,y)$ we define
\begin{align*}
\ell_{xy}(w) = - \log p(y|x,w)\,,\\
L_n(w) = \frac{1}{n} \sum_{i=1}^n \ell_{x_iy_i}(w)\,.
\end{align*}
The function $L_n(w)$ is the empirical negative log-likelihood and its average over the data distribution is denoted $L(w) = \mathbb{E}_{q(x,y)}[\ell_{xy}(w)]$. By a \emph{component} of the neural network we mean some subset of the weights $C$ associated with a product decomposition $W = U \times C$. Given a parameter $w^* = (u^*, v^*)$ and writing $w = (u,v)$ for the decomposition of a general parameter, we define a generalized function on $W$ by
\begin{equation}\label{eq:centered_deltaloss}
\phi_C(w) = \delta(u - u^*) \Big[ L(w) - L(w^*) \Big]
\end{equation}
where $\delta(u - u^*)$ is one if $u = u^*$ and zero otherwise. Given a prior $\varphi(w)$ on the parameter space, the \emph{quenched posterior} at inverse temperature $\beta > 0$ and sample size $n$ is
\begin{equation}
p_n^\beta(w) = \frac{1}{Z^\beta_n} \exp\{ -n\beta L(w) \} \varphi(w)\,\quad \text{where} \quad
Z_n^\beta = \int \exp\{-n\beta L(w)\}\varphi(w)\,dw.
\end{equation}
In practice, we use a \emph{localized} version of this posterior centered at a trained parameter $w^*$, replacing $\varphi(w)$ with a Gaussian $\exp\{-\frac{\gamma}{2}\|w - w^*\|^2\}$ and the population loss with the empirical loss $L_n(w)$. This ensures sampling remains in a neighborhood of $w^*$; see \citet{baker2025studyingsmalllanguagemodels} and \cref{appendix:susceptibilities-hparams} for details.

Given a generalized function $\phi(w)$ we define the expectation
\begin{equation}\label{eq:expectation_phi}
\langle \phi \rangle_{\beta}
= \int \phi(w) p_n^\beta(w) dw.
\end{equation}
and given a function $\psi(w)$ the covariance with respect to the quenched posterior is
\[
\operatorname{Cov}\big[ \phi, \psi \big] = \big\langle \phi \, \psi \big\rangle_\beta - \big\langle \phi \big\rangle_\beta \big\langle \psi \big\rangle_\beta\,.
\]

\begin{definition}\label{defn:per_token_susceptibility}
The \emph{per-token susceptibility} of $C$ for $(x,y) \in X \times Y$ is 
\begin{equation}\label{eq:per_sample_suscep}
\chi^C_{xy} := - \mathrm{Cov}\Bigl[ \phi_C, \ell_{xy}(w) - L(w) \Bigr]\,.
\end{equation}
\end{definition}

\subsection{Modes}\label{appendix:modes}

We briefly recall the framework of modes from \citet{modes2,baker2025studyingsmalllanguagemodels}. Fix a finite alphabet $\Sigma$ and consider the Hilbert space $\mathscr{H} = L^2(\Sigma^k, q; \mathbb{R}^{\Sigma})$ of functions from contexts $x \in \Sigma^k$ to vectors over tokens, with inner product
\[
\langle f, g \rangle_{\mathscr{H}} = \int \langle f(x), g(x) \rangle q(x) \, dx\,.
\]
The conditional distribution defines an element $\mathcal{C} \in \mathscr{H}$ by $\mathcal{C}(x) = \sum_y q(y|x) \, y$. Following \citet{modes2} we use the inner product on $\mathscr{V} = \mathbb{R}^{\Sigma^k}$ defined on basis elements $x, x' \in \Sigma^k$ by
\[
\langle x, x' \rangle_{\mathscr{V}} = q(x)^{-1} \delta_{x,x'}
\]
where $\delta_{x,x'}$ is the Kronecker delta. This weighting corresponds to the standard whitening procedure for SVD of conditional distributions: contexts are weighted inversely to their frequency, ensuring the decomposition captures structure in the joint distribution rather than being dominated by frequent contexts. We denote by $\langle -, - \rangle$ the standard inner product on $\mathbb{R}^\Sigma$. For $x \in \Sigma^k$ let $\hat{x}^*: \mathscr{V} \to \mathbb{R}$ denote the linear functional $\hat{x}^*(-) = \langle -, x \rangle_{\mathscr{V}}$.

For $y \in \Sigma$ and $x \in \Sigma^k$, define $y \circ \hat{x}^* \in \mathscr{H}$ by $(y \circ \hat{x}^*)(x') = y \cdot \hat{x}^*(x') = y \cdot q(x)^{-1} \delta_{x,x'}$. The norm of this vector is $\| y \circ \hat{x}^* \|_{\mathscr{H}} = q(x)^{-1/2}$ so the elements $\{ q(x)^{1/2} y \circ \hat{x}^*\}_{x \in \Sigma^k, y \in \Sigma}$ form an orthonormal basis of $\mathscr{H}$, which we call the \emph{token basis}.

Applying SVD to $\mathcal{C}$ yields singular values $s_\alpha$ with right singular vectors $v_\alpha$ (context patterns) and left singular vectors $u_\alpha$ (continuation patterns). Here $\Lambda$ indexes right singular vectors, $\Lambda^+$ indexes left singular vectors for nonzero singular values, and $\Lambda^{++} \supseteq \Lambda^+$ is an extension to an orthonormal basis of $\mathbb{R}^\Sigma$. Define $e_{\alpha\beta} \in \mathscr{H}$ by $e_{\alpha\beta}(x)(y) = \hat{v}_\alpha^*(x) \langle u_\beta, y \rangle$ where $\hat{v}_\alpha^*(x) = \langle v_\alpha, x \rangle_{\mathscr{V}}$. The elements $\{e_{\alpha\beta}\}_{\alpha \in \Lambda, \beta \in \Lambda^{++}}$ form an orthonormal basis of $\mathscr{H}$, which we call the \emph{modes basis}.

\paragraph{Transfer coefficients.} The \emph{propensity} $s_{\alpha\beta}(xy)$ is the coefficient relating these two bases:
\begin{equation}\label{eq:propensity}
s_{\alpha\beta}(xy) := \langle y \circ \hat{x}^*, e_{\alpha\beta} \rangle_{\mathscr{H}} = e_{\alpha\beta}(x)(y) = \hat{v}_\alpha^*(x) u_\beta^*(y)
\end{equation}
where $u_\beta^*(y) = \langle u_\beta, y \rangle$. This gives the expansion
\[
y \circ \hat{x}^* = \sum_{\alpha,\beta} s_{\alpha\beta}(xy) e_{\alpha\beta}
\]
where sums over $\alpha, \beta$ mean $\alpha \in \Lambda$ and $\beta \in \Lambda^{++}$.

\begin{lemma}\label{lemma:decomp_ellxyminusL} As functions of $w$ we have
\[
\ell_{xy}(w) - L(w) = \sum_{\alpha, \beta} \big[ s_{\alpha\beta}(xy) - \delta_{\alpha,\beta} s_\alpha\big] \Phi_{\alpha \beta}(w)
\]
where $\Phi(w)(x) = \sum_y \ell_{xy}(w) y$ and $\Phi_{\alpha\beta}(w) = \langle \Phi(w), e_{\alpha\beta} \rangle_{\mathscr{H}}$.
\end{lemma}
\begin{proof}
Using the expansion $y \circ \hat{x}^* = \sum_{\alpha, \beta} s_{\alpha\beta}(xy) e_{\alpha\beta}$ from \eqref{eq:propensity}, we repeat the calculation of \citet[Lemma D.4]{baker2025studyingsmalllanguagemodels} with $\Delta L = \ell_{xy} - L$
\begin{align*}
\Delta L &= - \int (\delta_{x,x'}\delta_{y,y'} q(x)^{-1} - q(y'|x')) q(x') \log p(y'|x',w) dx' dy'\\
&= - \int (y \circ \hat{x}^* - \mathcal{C})(x')(y') q(x') \log p(y'|x',w) dx' dy'\\
&= - \sum_{\alpha,\gamma} \big[ s_{\alpha\gamma}(xy) - \delta_{\alpha,\gamma} s_\alpha \big] \int e_{\alpha \gamma}(x')(y') q(x') \log p(y'|x',w) dx' dy'\\
&= \sum_{\alpha,\gamma} \big[ s_{\alpha\gamma}(xy) - \delta_{\alpha,\gamma} s_\alpha \big] \Phi_{\alpha\gamma}(w)
\end{align*}
as claimed.
\end{proof}

\begin{lemma}\label{lemma:formula_chixy_chialphabeta} Hence
\begin{equation}\label{eq:chi_mode_decomp}
\chi_{xy} = \sum_{\alpha,\beta} s_{\alpha\beta}(xy) \chi_{\alpha\beta} - \bar{\chi}
\end{equation}
where $\chi_{\alpha\beta} \in \mathbb{R}^H$ is the vector with components $\chi^C_{\alpha\beta} = -\operatorname{Cov}[\phi_C, \Phi_{\alpha\beta}]$, the susceptibility of component $C$ for mode pair $(\alpha,\beta)$, and
\[
\bar{\chi} = \sum_{\alpha} s_\alpha \chi_{\alpha\alpha}\,.
\]
\end{lemma}

\paragraph{Sparsity.} The decomposition \eqref{eq:chi_mode_decomp} is useful when the propensity profile $\{s_{\alpha\beta}(xy)\}_{\alpha,\beta}$ is sparse. This requires: (i) $x$ aligns with few context patterns $v_\alpha$, (ii) $y$ aligns with few continuation patterns $u_\beta$, and (iii) diagonal dominance, meaning $s_{\alpha\alpha}(xy)$ dominates over off-diagonal terms. This occurs when $(x,y)$ ``follows for a clear reason'' -- a small number of patterns coherently explain why $y$ follows $x$.

\subsection{Per-token susceptibilities as data perturbations}\label{appendix:per-token-perturbation}

We explain how to interpret the per-token susceptibility $\chi_{xy}$ in terms of perturbations of the data distribution. This clarifies the connection between per-token susceptibilities and the mode decomposition. Following \citet{baker2025studyingsmalllanguagemodels}, the susceptibility for a perturbation $q \to q'$ of the data distribution decomposes as
\begin{equation}\label{eq:tokenwise-suscep}
\chi = \int q'(x,y) \chi_{xy} dx dy\,.
\end{equation}
That is, the susceptibility is the $q'$-weighted average of per-token susceptibilities. Consider a perturbation that slightly upweights a single pair $(x_0, y_0)$
\[
q' = (1-\epsilon) q + \epsilon \, \delta_{(x_0, y_0)}
\]
where $\delta_{(x_0, y_0)}$ is the point mass at $(x_0, y_0)$. Using \eqref{eq:tokenwise-suscep}:
\begin{align*}
\chi &= (1-\epsilon) \int q(x,y) \, \chi_{xy} \, dx \, dy + \epsilon \, \chi_{x_0 y_0}\\
&= (1-\epsilon) \, \mathbb{E}_q[\chi_{xy}] + \epsilon \, \chi_{x_0 y_0}\\
&= \epsilon \, \chi_{x_0 y_0}
\end{align*}
where we used $\mathbb{E}_q[\chi_{xy}] = 0$ (the per-token susceptibilities are centered). The per-token susceptibility $\chi_{xy}$ is proportional to the susceptibility for a perturbation that upweights the pair $(x,y)$.

\subsection{Inverting the decomposition}

Using the orthonormality $\langle e_{\alpha\beta}, e_{\gamma\delta} \rangle_{\mathscr{H}} = \delta_{\alpha\gamma}\delta_{\beta\delta}$, we can invert \eqref{eq:chi_mode_decomp}. Observe
\begin{align*}
\int s_{\alpha\beta}(xy) q(x) \chi_{xy} \, dx \, dy &= \sum_{\gamma,\delta} \chi_{\gamma\delta} \int e_{\alpha\beta}(x)(y) \big( e_{\gamma\delta}(x)(y) - \delta_{\gamma,\delta} s_\gamma \big) q(x) \, dx \, dy\\
&= \sum_{\gamma, \delta} \chi_{\gamma \delta} \Big[ \langle e_{\alpha\beta}, e_{\gamma\delta} \rangle_{\mathscr{H}} - \int q(x) \delta_{\gamma,\delta} s_\gamma \, dx \, dy \Big]\\
&= \chi_{\alpha\beta} - |\Sigma| \bar{\chi}\,.
\end{align*}

\subsection{Per-pattern susceptibilities}\label{appendix:per-pattern}

The mode structure of a data distribution is not known \emph{a priori}, and one may instead work with patterns defined by interpretable token properties. Given a pattern category $\mathcal{P}$ (e.g., word-start tokens, induction patterns, spacing tokens), \citet{wang2025lang2pt5} define the empirical \emph{per-pattern susceptibility}
\begin{equation}\label{eq:per-pattern-susceptibility}
\hat\chi(\mathcal{P}) = \frac{1}{|\mathcal{P}|}\sum_{(x,y) \in \mathcal{P}} \chi_{xy}
\end{equation}
where the sum is over token pairs classified as following pattern $\mathcal{P}$. This is the average susceptibility vector over tokens in the pattern class. The per-pattern susceptibility can be understood as an empirical approximation to the pure susceptibility $\chi_{\alpha\alpha}$. To see this, substitute \eqref{eq:chi_mode_decomp} into \eqref{eq:per-pattern-susceptibility}:
\[
\hat\chi(\mathcal{P}) = \sum_{\gamma,\delta} \chi_{\gamma\delta} \cdot \frac{1}{|\mathcal{P}|}\sum_{(x,y) \in \mathcal{P}} s_{\gamma\delta}(xy) - \bar{\chi}\,.
\]
The coefficient $\frac{1}{|\mathcal{P}|}\sum_{(x,y) \in \mathcal{P}} s_{\gamma\delta}(xy)$ measures the average propensity of mode $(\gamma,\delta)$ within pattern $\mathcal{P}$. If pattern $\mathcal{P}$ is well-aligned with a mode $\alpha$ in the sense that:
\begin{enumerate}
\item[(i)] tokens $(x,y) \in \mathcal{P}$ have large diagonal propensity $s_{\alpha\alpha}(xy)$, and
\item[(ii)] propensities $s_{\gamma\delta}(xy)$ with $\gamma \neq \alpha$ and $\delta \neq \alpha$ are small relative to $|\mathcal{P}|$,
\end{enumerate}
then $\hat\chi(\mathcal{P}) \approx c \cdot \chi_{\alpha\alpha} - \bar{\chi}$ for some constant $c > 0$ and with the vector $\bar{\chi}$ independent of $\mathcal{P}, \alpha$. In this case, the per-pattern susceptibility captures the same information as the pure susceptibility for the corresponding mode.

The per-pattern approach of \citet{wang2025lang2pt5} uses indicator functions for pattern membership (uniform weights within each class), while the mode decomposition uses theoretically-derived weights $s_{\alpha\beta}(xy)$ from the SVD of the conditional distribution. The latter is more principled but requires knowledge of the mode structure; the former is practical when patterns are defined by interpretable token properties that happen to align with the underlying modes.

\subsection{Toy model of modes}\label{appendix:toy-model}

In this section we study a simplified data distribution and the modes that it determines. As our starting point we take an observation about the clusters in Pythia-14M. The model ``understands'' two distinct ways that a sentence, concluded with a full stop, can be continued: with a capitalized word or a newline. The relevant clusters:
\begin{itemize}
    \item \textbf{Capitalized clusters}: there are a number of clusters consisting of token sequences $xy$ where $x$ ends in a period token (i.e. $x = x'\tok{.}$ for some $x'$) and the token $y$ de-tokenises to a space followed by a capital letter. In short, the clusters capture \emph{capitalized words following full stops}. The examples: \clu{180} (sentence-initial \tok{~These}), \clu{280} (conjunctive adverbs like \tok{~However}, \tok{~Thus}, \tok{~Therefore}), \clu{305} (sentence-initial words like \tok{~In}, \tok{~At}, \tok{~Since}), along with related clusters \clu{36}, \clu{160}, \clu{351}, \clu{363}, and \clu{365} which capture other sentence-initial capitalized words after periods. Most of these appear in the UMAP near \clu{180} as shown in \cref{fig:pythia14m_umap1}.
    \item \textbf{Newline cluster}: the cluster \clu{189} consists mostly of \tok{\textbackslash n} following \tok{.}. This cluster appears in the UMAP with the other \tok{\textbackslash n} tokens as shown in \cref{fig:pythia14m_umap1}.
\end{itemize}

Fix some $k > 0$ and let three distinct contexts $x_C, x_N, x_E \in \Sigma^k$ be chosen. Mathematically we make no further assumptions on these contexts, but informally, we think of these as all ending in \tok{.} with the following distinctions:
\begin{itemize}
    \item $x_C$ is the kind of context that is usually continued with a Capitalized word.
    \item $x_N$ is the kind of context that is usually continued with a Newline.
    \item $x_E$ can be continued Either way.
\end{itemize}

We let $y_C, y_N \in \Sigma$ denote continuations that begin with a capitalized word or a newline, respectively. In practice of course $x_C, x_N, x_E$ and $y_C, y_N$ should be replaced by sets of tokens with their own distribution, but we treat only the simplest case (for the general idea see the collective bigrams of \citet{modes2}). Then as our conditional distributions we take
    \begin{align*}
        q(y_C|x_C) &= 1, \quad q(y_N|x_C) = 0 & &\text{($x_C$ prefers capitalized)}\\
        q(y_C|x_N) &= 0, \quad q(y_N|x_N) = 1 & &\text{($x_N$ prefers newline)}\\
        q(y_C|x_E) &= a, \quad q(y_N|x_E) = b & &\text{($x_E$ is ambiguous)}
    \end{align*}
where $a + b = 1$ and $a, b > 0$, with $q(x_C) = q(x_N) = q(x_E) = \frac{1}{3}$. The operator $\mathcal{C}$ is
\[
M = \begin{array}{c@{\hspace{2pt}}c}
& \begin{array}{ccc} \scriptstyle x_C & \scriptstyle x_N & \scriptstyle x_E \end{array} \\[2pt]
\begin{array}{c} \scriptstyle y_C \\[4pt] \scriptstyle y_N \end{array} &
\begin{pmatrix} 1 & 0 & a \\[4pt] 0 & 1 & b \end{pmatrix}
\end{array}
\]
where columns are contexts and rows are tokens. We compute the SVD of $\frac{1}{\sqrt{3}} M$ \citep[Remark 4.11]{modes2}. The eigenvalues of $MM^T$ are found from the characteristic polynomial
\[
\lambda^2 - (3 - 2ab)\lambda + 2(1 - ab) = 0
\]
which has discriminant $(1 - 2ab)^2$, giving $\lambda_1 = 2(1-ab), \lambda_2 = 1$. The singular values of $\frac{1}{\sqrt{3}}M$ are
\[
s_1 = \sqrt{\frac{2(1-ab)}{3}}, \qquad s_2 = \frac{1}{\sqrt{3}}\,.
\]
The left singular vectors (token patterns) are
\begin{align*}
u_1 &\propto a \cdot y_C + b \cdot y_N & &\text{(weighted average of continuations)}\\
u_2 &\propto b \cdot y_C - a \cdot y_N & &\text{(capitalize vs.\ newline contrast)}
\end{align*}
and the right singular vectors (context patterns) are
\begin{align*}
v_1 &\propto a \cdot x_C + b \cdot x_N + (1-2ab) \cdot x_E\\
v_2 &\propto b \cdot x_C - a \cdot x_N + 0 \cdot x_E
\end{align*}
The key observation is that the second mode $v_2$ has zero weight on $x_E$ regardless of the value of $a$. The ambiguous context does not participate in the discriminating mode -- it has nothing to contribute to the distinction between capitalizing and not capitalizing.

In the special case $a = b = 1/2$ we have $s_1 = 1/\sqrt{2}$, $s_2 = 1/\sqrt{3}$, and $u_1 \propto y_C + y_N$, $u_2 \propto y_C - y_N$, $v_1 \propto x_C + x_N + x_E$, $v_2 \propto x_C - x_N$.

\section{Token pattern definitions}\label{appendix:token-definitions}

Following \citet{baker2025studyingsmalllanguagemodels,wang2025lang2pt5}, we classify tokens according to patterns that capture syntactic and structural properties. \Cref{tab:token-patterns} defines the pattern categories used throughout this paper, with the corresponding colors used in \cref{fig:pythia14m_umap1} and other figures. \Cref{fig:distribution_pattern} shows how these patterns are distributed across the 13 Pile subsets used in our analysis; note that patterns are not mutually exclusive, so percentages need not sum to 100\%.

\begin{figure}[t]
    \centering
    \includegraphics[width=\textwidth]{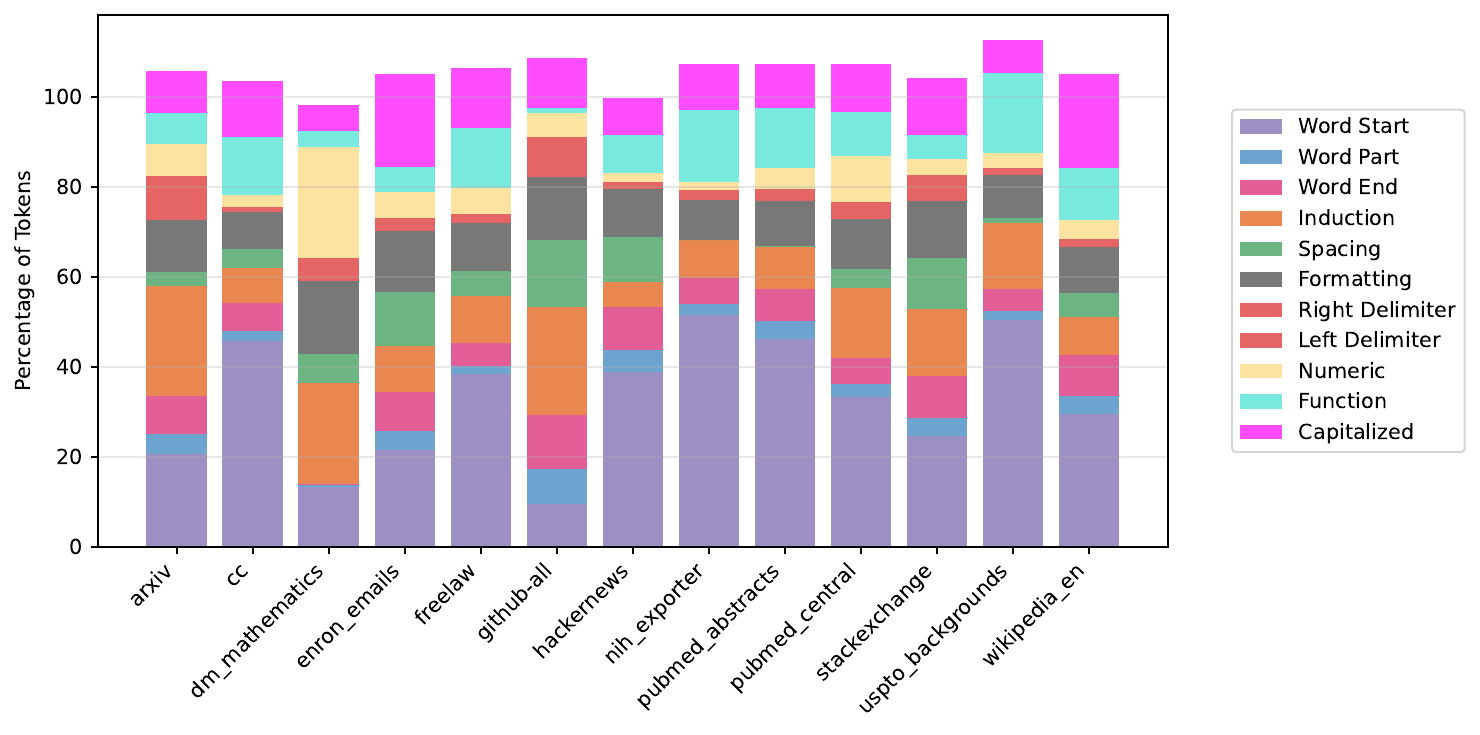}
    \caption{Percentages of tokens in each dataset which follow a given pattern. Note that not all patterns are mutually exclusive.}
    \label{fig:distribution_pattern}
\end{figure}

Note the distinctions with respect to the tokenizer in \citet{wang2025lang2pt5}. The tokenizer used for the Pythia models contains dedicated tokens for spaces of various lengths. Let us denote by $\tokenbox{~} \times n$ a token consisting of $n$ consecutive spaces. The token ID of \tokenbox{~} is $209$ while for $1 < n \le 24$ the token ID of $\tokenbox{~} \times n$ is $50278 - n$. In particular, the token consisting of $24$ consecutive spaces (the largest number encoded by a single token) has ID $50254$.

\begin{table}[t]
\centering
\begin{tabular}{@{}p{2cm}p{6.5cm}p{4cm}@{}}
\toprule
\textbf{Pattern} & \textbf{Definition} & \textbf{Examples} \\
\midrule
\legendbox{patternWordStart}{\textcolor{white}{\textbf{Word start}}} & 
A token that decodes to a space followed by lower or upper case letters & 
\tokenboxline{~be}, \tokenboxline{~R}\tokenbox{ose}, \tokenboxline{~The} \\
\addlinespace[0.5em]

\legendbox{patternWordPart}{\textcolor{white}{\textbf{Word part}}} & 
A non-word-end token that decodes to upper or lower case letters & 
\tokenbox{~S}\tokenboxline{ne}\tokenbox{ed}, \tokenboxline{th}\tokenbox{at}, \tokenbox{st}\tokenboxline{em}\tokenbox{ed} \\
\addlinespace[0.5em]

\legendbox{patternWordEnd}{\textcolor{white}{\textbf{Word end}}} & 
A token that decodes to upper or lower case letters followed by a formatting token, delimiter or space & 
\tokenbox{~el}\tokenbox{im}\tokenboxline{inate}\tokenbox{~}, \tokenbox{~differe}\tokenboxline{nces}\tokenbox{)}, \tokenbox{al}\tokenboxline{bum}\tokenbox{~} \\
\addlinespace[0.5em]

\legendbox{patternInduction}{\textcolor{white}{\textbf{Induction}}} & 
A sequence of tokens $uvUuv$ where $U$ is any sequence, $u,v$ are individual tokens, and $uv$ is not a common bigram ($q(v|u) \leq 0.05$) & 
\small\tokenbox{the}\tokenbox{ cat} $\ldots$ \tokenbox{the}\tokenboxline{ cat} \\
\addlinespace[0.5em]

\legendbox{patternSpacing}{\textcolor{white}{\textbf{Spacing}}} & 
A token made up of one or more characters from space, newline, tab, carriage return, or form feed & 
\tokenbox{~}, \tokenbox{\textbackslash n}, \tokenbox{\textbackslash t}, \tokenbox{\textbackslash{}n\textbackslash{}n} \\
\addlinespace[0.5em]

\legendbox{patternRightDelimiter}{\textcolor{white}{\textbf{Delimiter}}} & 
Brackets and composite tokens including parentheses, brackets, and their combinations & 
\tokenbox{)}, \tokenbox{~)}, \tokenbox{]}, \tokenbox{);}, \tokenbox{(} \\
\addlinespace[0.5em]

\legendbox{patternFormatting}{\textcolor{white}{\textbf{Formatting}}} & 
Tokens used for document structure and formatting beyond simple spacing & 
\tokenbox{.}, \tokenbox{,}, \tokenbox{~//} \\
\addlinespace[0.5em]

\legendbox{patternNumeric}{\textcolor{white}{\textbf{Numeric}}} & 
Tokens containing numerical digits & 
\tokenbox{123}, \tokenbox{~14}, \tokenbox{~2024} \\
\addlinespace[0.5em]

\legendbox{patternFunction}{\textcolor{white}{\textbf{Function}}} & 
Function words & 
\tokenbox{~the}, \tokenbox{~and}, \tokenbox{~to} \\
\addlinespace[0.5em]

\legendbox{patternCapitalized}{\textcolor{white}{\textbf{Capitalized}}} & 
Capitalized words and acronyms & 
\tokenbox{~Denver}, \tokenbox{CBRN}, \tokenbox{~Enron} \\
\bottomrule
\end{tabular}
\vspace{0.5em}
\caption{Token pattern categories and their definitions. Throughout the text we apply the indicated colors to tokens that follow a particular pattern. Boxed tokens indicate the pattern being illustrated.}
\label{tab:token-patterns}
\end{table}

\section{UMAP}\label{appendix:umap}

The data matrix $X$ has $l \times h$ columns (where $l$ is the number of layers and $h$ the number of heads per layer) and approximately $780{,}000$ rows. Each row is the susceptibility vector $\chi_{xy}$ for a fixed neural network parameter $w$ where $(x,y) \sim q^l(x,y)$ as $1 \le l \le 13$ ranges over subsets of the Pile \citep{gao2020pile}: \pilesub{github-code}, \pilesub{pile-cc}, \pilesub{pubmed\_abstracts}, \pilesub{uspto\_backgrounds}, \pilesub{pubmed\_central}, \pilesub{stackexchange}, \pilesub{wikipedia\_en}, \pilesub{freelaw}, \pilesub{arxiv}, \pilesub{dm\_mathematics}, \pilesub{enron\_emails}, \pilesub{hackernews}, and \pilesub{nih\_exporter}. We sample $60{,}000$ token sequences from each dataset. The data matrix $X$ is standardized (that is, the columns have the mean subtracted and are rescaled to have unit standard deviation) before applying the UMAP algorithm.

\subsection{UMAP hyperparameters} \label{appendix:umap-hparams}

The UMAP algorithm depends fundamentally on the choice of \texttt{n\_neighbors} hyperparameter. 
The images in this paper were computed with \texttt{n\_neighbors} equal to $45$.
  
This hyperparameter governs how many nearest neighbors are taken into consideration when computing the local distances in the original embedding that the learned embedding tries to match. 
The value being too low can cause misleading clusters of data points in the visualization. 

\section{Larger Models}\label{appendix:larger-models}

\begin{table}[t]
    \centering
    \begin{tabular}{|c|c|c|}
        \hline
        \textbf{Model Name} & \textbf{Number of Components} & \textbf{Number of Clusters Found} \\
        \hline
\texttt{pythia-14m}  & 32 & 510 \\
\texttt{pythia-31m}  & 56 & 358 \\
\texttt{pythia-70m}  & 56 & 254 \\
\texttt{pythia-160m} & 158 & 311 \\
\texttt{pythia-410m} & 410 & 241 \\
\texttt{pythia-1.4b} & 410 & 249\\
        \hline
    \end{tabular}
    \vspace{0.5em}
    \caption{Pythia models from 14M-1.4B parameters and the results of clustering.}
    \label{larger-models-table}
\end{table}

We repeated this analysis on a collection of five more models with as many as 1.4B parameters (\cref{fig:pythia_family_umaps}). This list of models, as well as number of components per model and clusters found using the clustering algorithm in \cref{appendix:clustering} is shown in \cref{larger-models-table}.

As discussed in \cref{section:scaling}, clustering performed better on Pythia-14M than on the larger models. Nonetheless the clusters found in higher models were highly coherent. They almost all corresponded to observable, interpretable patterns in the data. However, there were fewer overall clusters, and most of the patterns found were already seen in the Pythia-14M clusters. We plan to improve our data collection and analysis methods and hope to achieve similar or higher quality cluster detection models of this scale. See \cref{appendix:compute} for details about the cost of computing susceptibilities on these models.


We took the clusters found for Pythia-14M and measured their conductance where the graph structure on the data is given by the susceptibility values found in the larger models instead, as shown in \cref{fig:conductance_scatter}.


\begin{figure}
    \centering
    \includegraphics[width=0.9\linewidth]{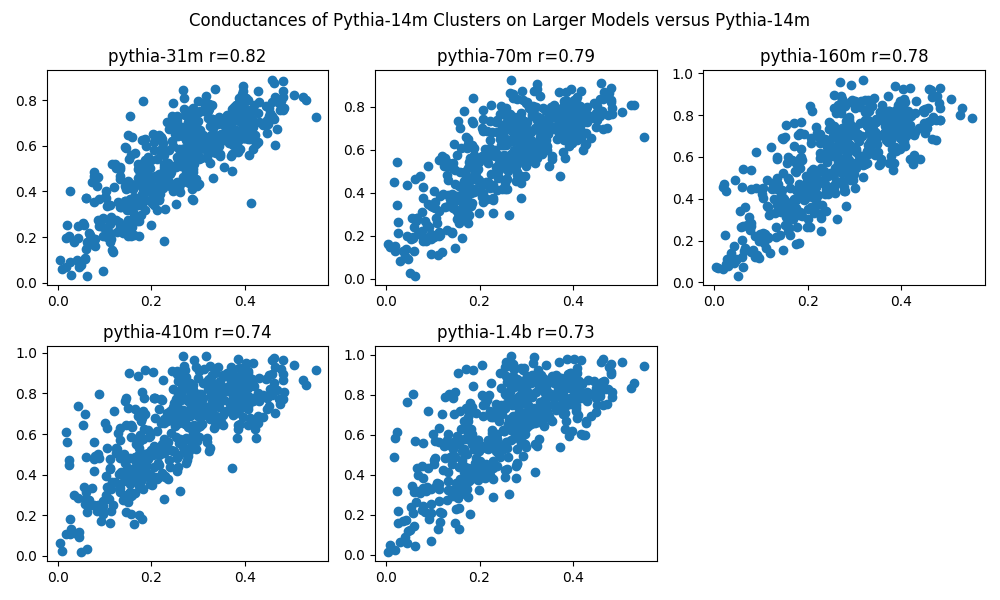}
    \caption{\textbf{Conductance of Pythia-14M clusters in larger models.} Each point represents one of the 510 clusters, with its conductance measured on the susceptibility distance graph of a larger Pythia model (y-axis) versus Pythia-14M (x-axis). Clusters with low conductance in both models correspond to patterns that persist across scale.}
    \label{fig:conductance_scatter}
\end{figure}

\begin{figure}[p]
    \centering
    \begin{tikzpicture}
        \def\imgwidth{0.48\textwidth}
        \def\imgheight{0.38\textwidth}
        \def\hgap{0.04\textwidth}
        \def\vgap{0.06\textwidth}

        \begin{scope}
            \clip[rounded corners=8pt] (0,0) rectangle (\imgwidth,0.9*\imgheight);
            \node[anchor=south west,inner sep=0] at (0,0)
                {\includegraphics[width=\imgwidth]{figures-new/pythia-14m/pythia14_img3.png}};
        \end{scope}
        \draw[white, line width=0.5pt, rounded corners=8pt] (0,0) rectangle (\imgwidth,\imgheight);
        \node[below=0.1cm, font=\small\bfseries] at (0.5*\imgwidth,0) {Pythia-14M};

        \begin{scope}[shift={(\imgwidth+\hgap,0)}]
            \clip[rounded corners=8pt] (0,0) rectangle (\imgwidth,0.9*\imgheight);
            \node[anchor=south west,inner sep=0] at (0,0)
                {\includegraphics[width=\imgwidth]{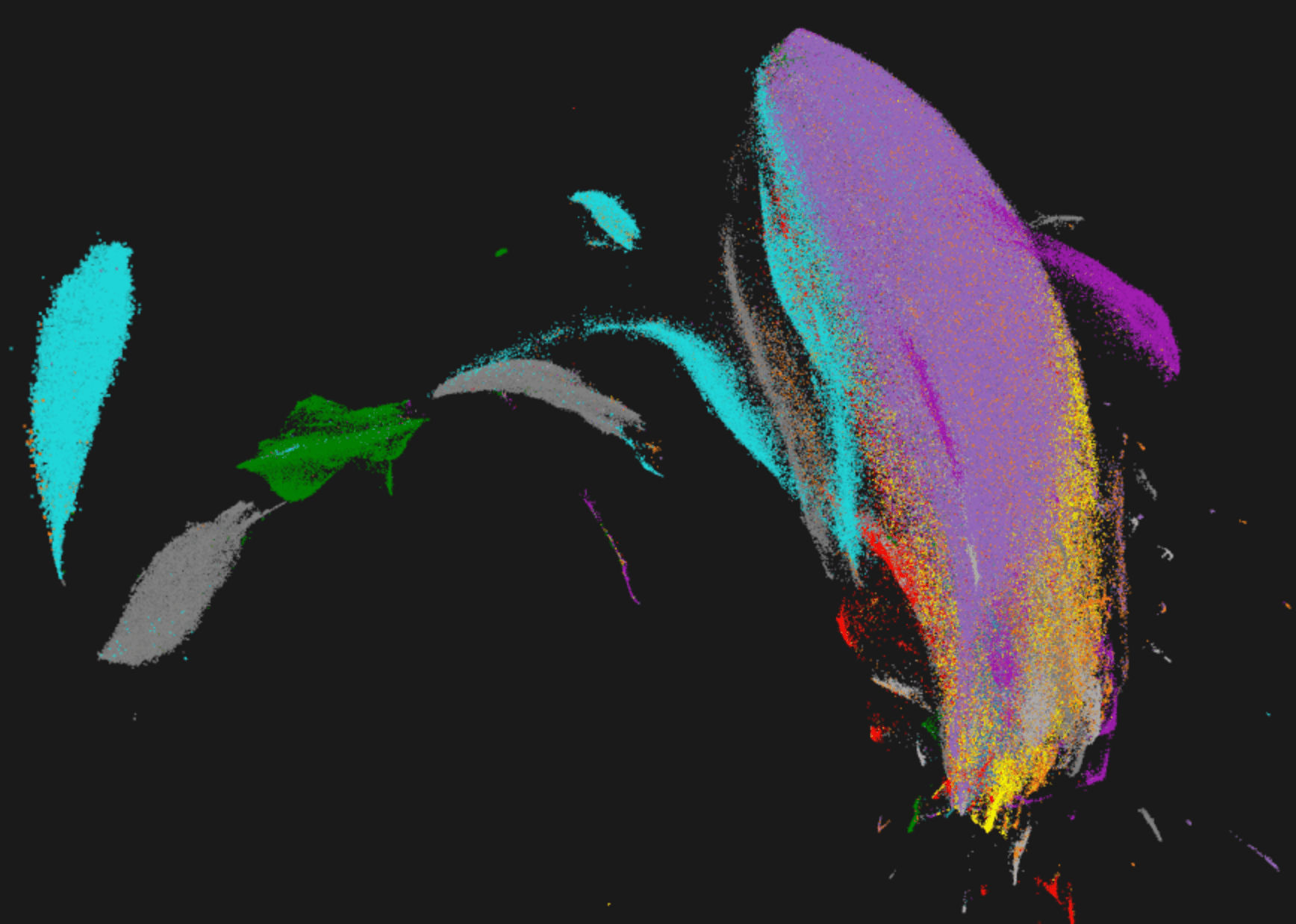}};
        \end{scope}
        \draw[white, line width=0.5pt, rounded corners=8pt] (\imgwidth+\hgap,0) rectangle (2*\imgwidth+\hgap,\imgheight);
        \node[below=0.1cm, font=\small\bfseries] at (1.5*\imgwidth+\hgap,0) {Pythia-31M};

        \begin{scope}[shift={(0,\imgheight+\vgap)}]
            \clip[rounded corners=8pt] (0,0) rectangle (\imgwidth,0.9*\imgheight);
            \node[anchor=south west,inner sep=0] at (0,0)
                {\includegraphics[width=\imgwidth]{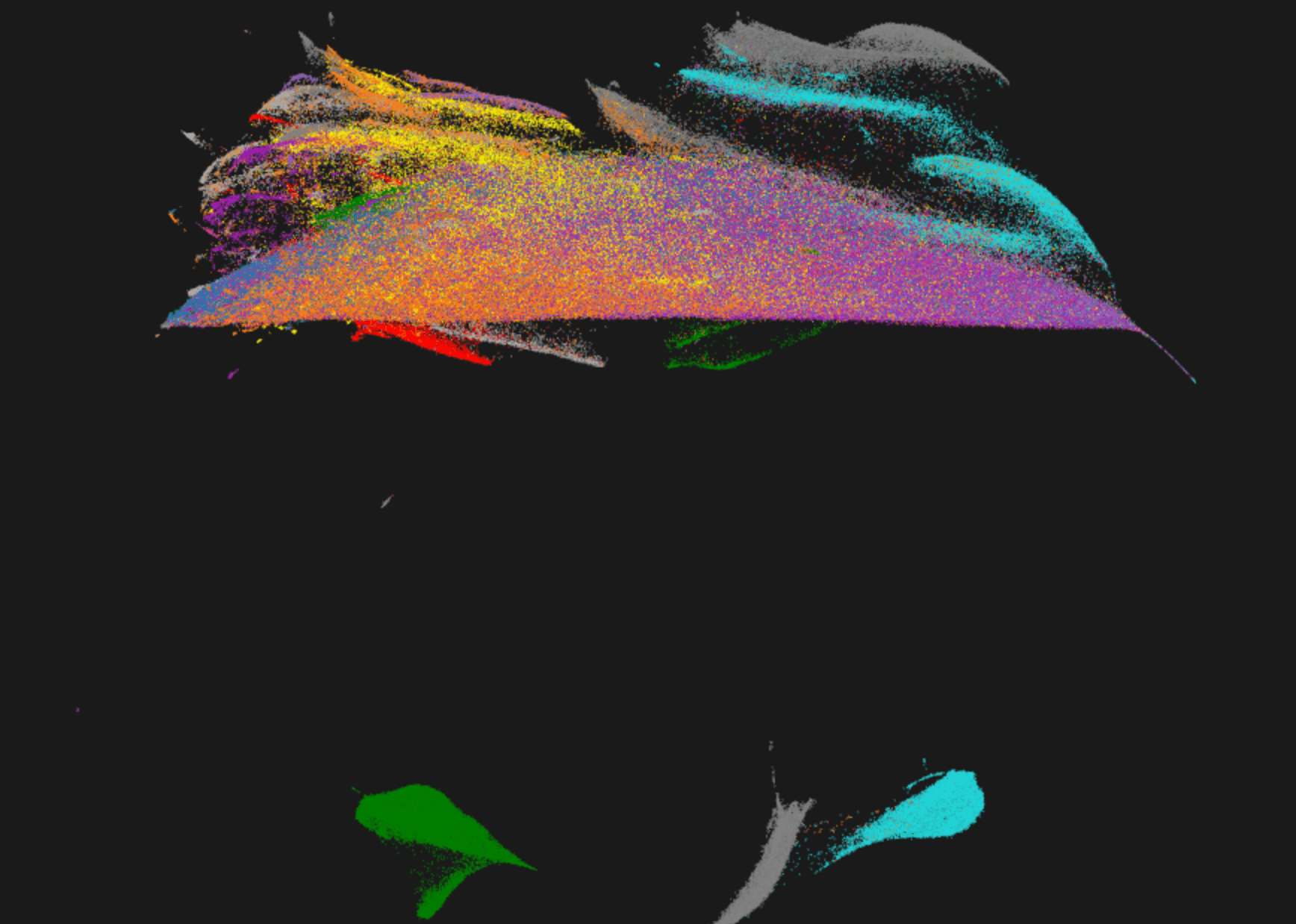}};
        \end{scope}
        \draw[white, line width=0.5pt, rounded corners=8pt] (0,\imgheight+\vgap) rectangle (\imgwidth,2*\imgheight+\vgap);
        \node[below=0.1cm, font=\small\bfseries] at (0.5*\imgwidth,\imgheight+\vgap) {Pythia-70M};

        \begin{scope}[shift={(\imgwidth+\hgap,\imgheight+\vgap)}]
            \clip[rounded corners=8pt] (0,0) rectangle (\imgwidth,0.9*\imgheight);
            \node[anchor=south west,inner sep=0] at (0,0)
                {\includegraphics[width=\imgwidth]{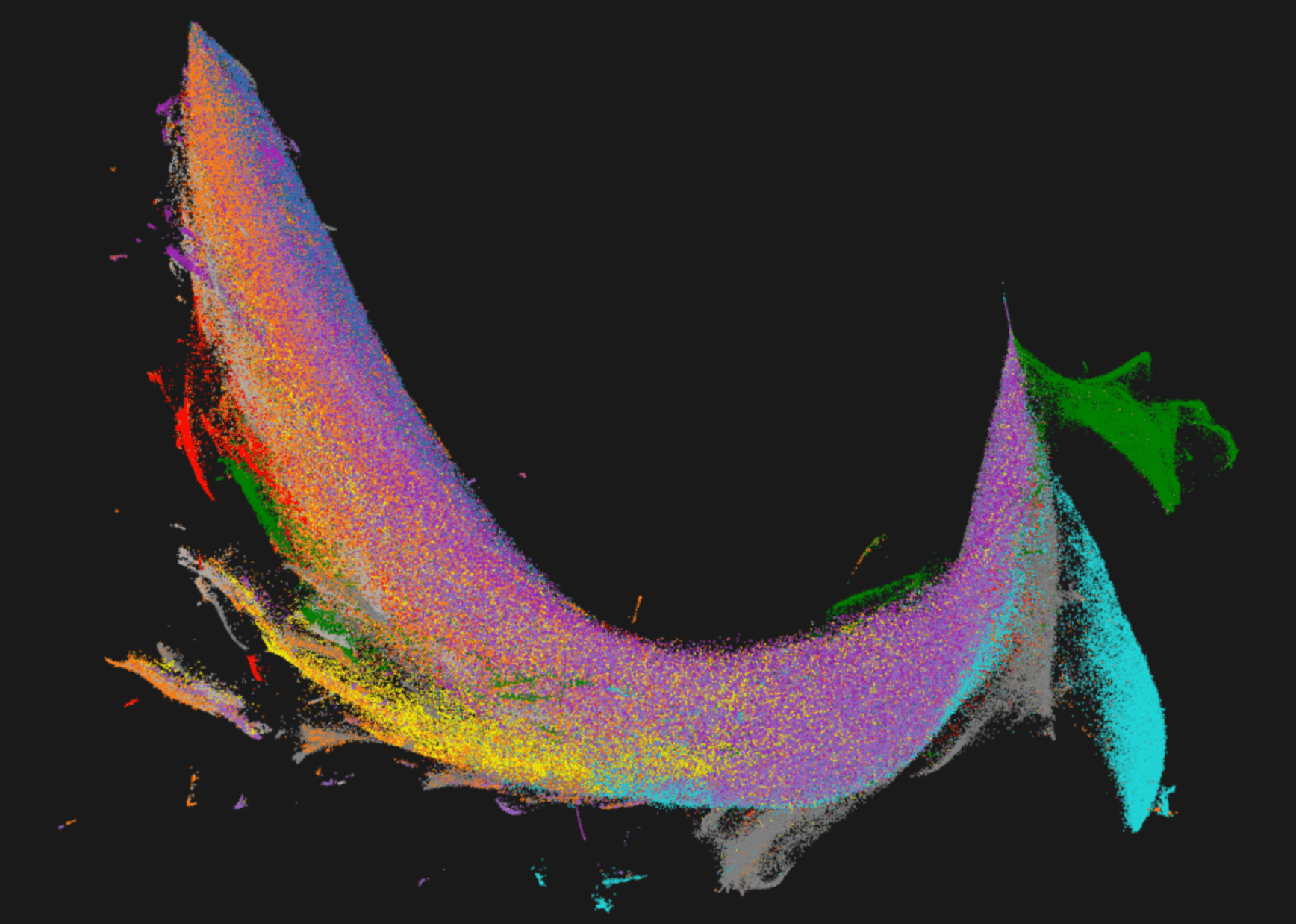}};
        \end{scope}
        \draw[white, line width=0.5pt, rounded corners=8pt] (\imgwidth+\hgap,\imgheight+\vgap) rectangle (2*\imgwidth+\hgap,2*\imgheight+\vgap);
        \node[below=0.1cm, font=\small\bfseries] at (1.5*\imgwidth+\hgap,\imgheight+\vgap) {Pythia-160M};

        \begin{scope}[shift={(0,2*\imgheight+2*\vgap)}]
            \clip[rounded corners=8pt] (0,0) rectangle (\imgwidth,0.9*\imgheight);
            \node[anchor=south west,inner sep=0] at (0,0)
                {\includegraphics[width=\imgwidth]{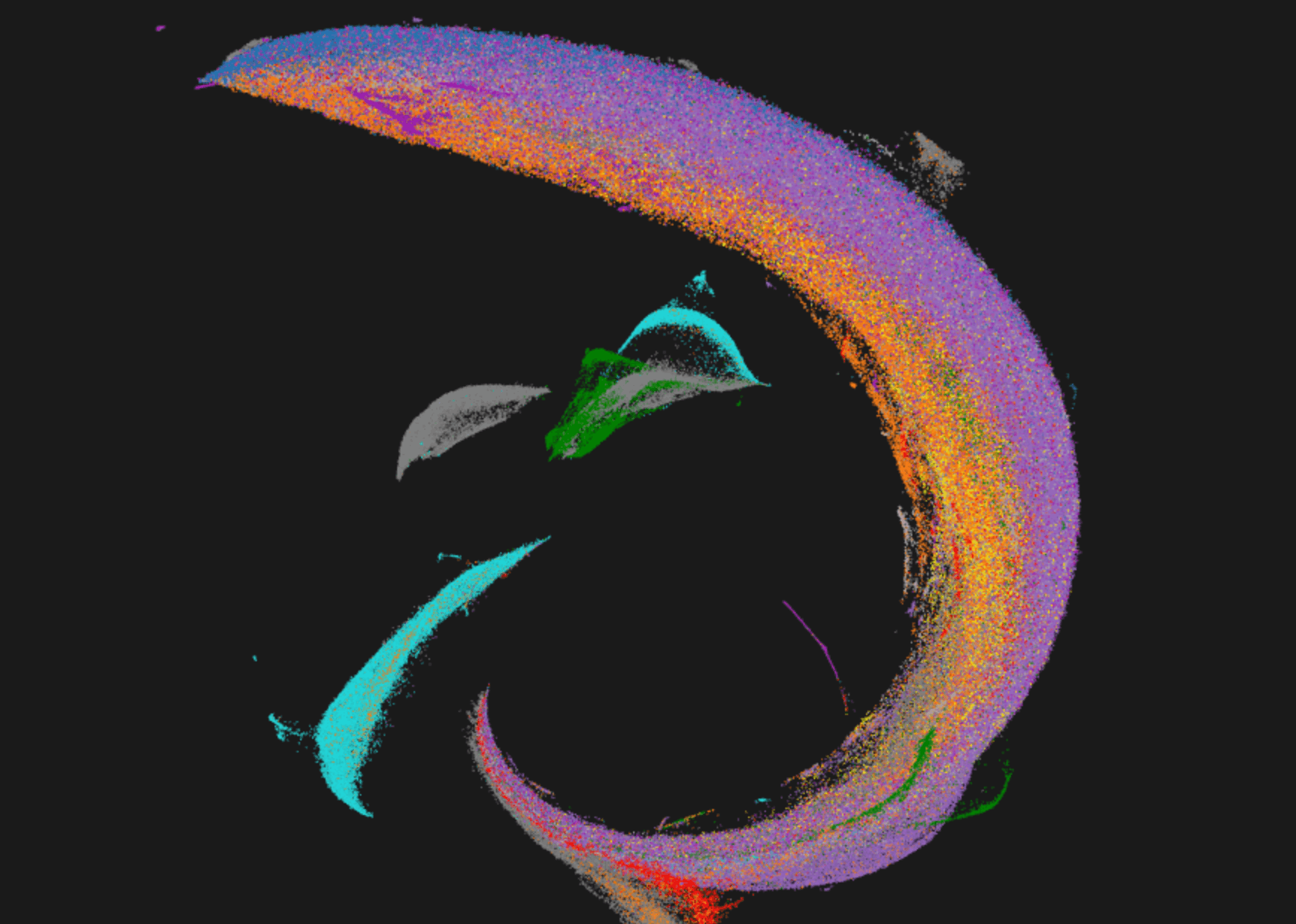}};
        \end{scope}
        \draw[white, line width=0.5pt, rounded corners=8pt] (0,2*\imgheight+2*\vgap) rectangle (\imgwidth,3*\imgheight+2*\vgap);
        \node[below=0.1cm, font=\small\bfseries] at (0.5*\imgwidth,2*\imgheight+2*\vgap) {Pythia-410M};

        \begin{scope}[shift={(\imgwidth+\hgap,2*\imgheight+2*\vgap)}]
            \clip[rounded corners=8pt] (0,0) rectangle (\imgwidth,0.9*\imgheight);
            \node[anchor=south west,inner sep=0] at (0,0)
                {\includegraphics[width=\imgwidth]{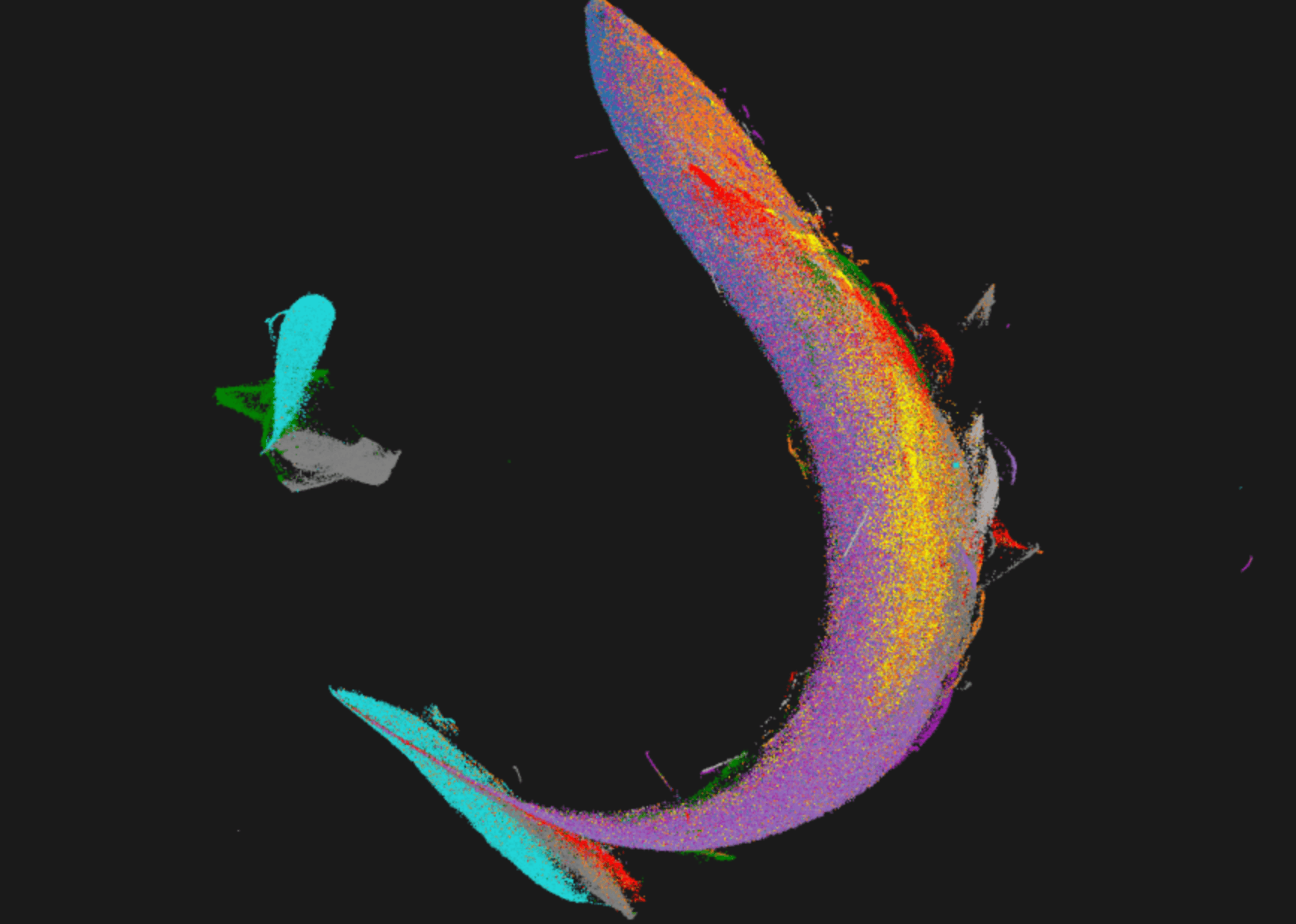}};
        \end{scope}
        \draw[white, line width=0.5pt, rounded corners=8pt] (\imgwidth+\hgap,2*\imgheight+2*\vgap) rectangle (2*\imgwidth+\hgap,3*\imgheight+2*\vgap);
        \node[below=0.1cm, font=\small\bfseries] at (1.5*\imgwidth+\hgap,2*\imgheight+2*\vgap) {Pythia-1.4B};
    \end{tikzpicture}

    \caption{\textbf{UMAP visualizations of susceptibility vectors across the Pythia family.} Each panel shows the low-dimensional representation of susceptibility vectors for the same set of token sequences, computed from models ranging from 14M to 1.4B parameters. Points are colored by token pattern type. The overall structure -- including the separation of spacing tokens, the clustering of similar patterns, and the broad organization into regions -- persists across model scales, though the precise geometry varies.}
    \label{fig:pythia_family_umaps}
\end{figure}

\section{Additional Linked Clusters}\label{appendix:additional_linked_clusters}

In \cref{section:networks_linked} we introduced the notion of \emph{linked clusters}: clusters that correspond to modes at different context lengths which are related because the reason $y$ follows $x$ depends on earlier structure in the context. The main text presents the HTML tag network as a detailed example. Here we provide additional examples from code, mathematics, and scientific typesetting. These networks illustrate how the model develops coordinated cluster structure to handle sequential, nested, or multi-phase patterns in the training data.

\paragraph{Code Block network}
Similarly, the model has distinct clusters related to nesting depth and scope of programming logic (specifically in C-style languages). The sequence often starts with conditional keywords like \tok{if (} (\clu{97}). The opening brace \tok{\{} marks the start of the block (\clu{86}, \cluq{267}, \cluq{337}). The structure of the body of the code block is governed by clusters such as 4-space indent (\clu{106}) versus deep 8-space indent (\clu{17}) and more subtle spacing (\clu{48}, \cluq{60}, \cluq{103}, \cluq{340}, \cluq{392}, \cluq{422}, \cluq{449}).

Note that in both cases the primary role of the clusters seems to be ``structural'' in the sense that it is the \emph{form} rather than the \emph{content} of both HTML tags and code blocks that seems to attract dedicated structure.

\subsubsection{Mathematical Reasoning and Typesetting}

\paragraph{Math Problem Network}
The \tok{dm\_mathematics} dataset consists of contexts with highly repetitive question formats generated by an algorithm, and provides a rich source of patterns for language models to learn. The model appears to treat math problems as a three-act structure. The problem begins with clusters for definition keywords like \tok{Let} or \tok{Suppose} (\textbf{C333, 386, 458}), followed by variable assignments like ``Let $x$ be...'' (\textbf{C70, 129, 194, 360, 428, 476}). The model pivots to the question phase with imperative clusters like \tok{Calculate} (\textbf{C438, 439}) or \tok{What is} (\textbf{C169}). The sequence concludes with clusters specialized for answer formats, such as boolean \tok{True}/\tok{False} tokens (\textbf{C250}).

\paragraph{Multiple Choice Network}
A network of clusters seems related to the patterns involved in the rigid formatting of multiple-choice options, ensuring consistent delimiters and spacing. Options are enclosed by specific opening parentheses \tok{(} (\textbf{C241, 316}) and closing parentheses \tok{)} (\textbf{C37, 437}). The labels themselves form a sequence, with specific clusters for \tok{a} (\textbf{C285}) and \tok{b} (\textbf{C315}). A unique cluster is dedicated to the double-space \tok{~~} separator (\textbf{C100, 264}) that visually isolates options from one another.

\paragraph{LaTeX Math Mode Network}
For scientific typesetting, the model toggles between text mode and math mode using dedicated clusters. The dollar sign delimiter \tok{\$} is handled by \textbf{C49} (opening) and \textbf{C68} (closing). Within math mode, clusters track backslash commands like \tok{\textbackslash frac} (\textbf{C150}) and the mandatory opening brace \tok{\{} that follows them (\textbf{C173, 249}).

\section{SAE Comparison}\label{section:sae-methodology}

To validate that our susceptibility clusters capture structure also identified by other interpretability methods, we compare them to sparse autoencoder (SAE) features. We evaluate pre-trained SAEs from \citet{lan2025sparse} on Pythia-70M against the token-in-context examples that populate each cluster. We use Pythia-70M because pre-trained SAEs were not available for Pythia-14M; the conductance analysis in \cref{section:scaling} provides evidence that Pythia-14M clusters remain coherent in larger models. The SAEs were trained on the residual stream activations and are available for each layer; we use layers $2$--$4$ (middle layers typically capture the most interpretable features, and we use three layers for redundancy).

The SAEs follow a standard architecture with pre-centering. Given a residual stream activation $a \in \mathbb{R}^{d}$ where $d = 512$ for Pythia-70M, the encoder produces a sparse latent representation $z \in \mathbb{R}^{D}$ with $D = 32768$ latent dimensions:
\begin{align}
z &= \mathrm{ReLU}\big( W_{\mathrm{enc}} (a - b_{\mathrm{dec}}) + b_{\mathrm{enc}} \big) \label{eq:sae-encode}
\end{align}
where $W_{\mathrm{enc}} \in \mathbb{R}^{D \times d}$ is the encoder weight matrix, $b_{\mathrm{enc}} \in \mathbb{R}^D$ is the encoder bias, and $b_{\mathrm{dec}} \in \mathbb{R}^d$ is the decoder bias which serves as a pre-centering term. The decoder reconstructs the input as
\begin{align}
\hat{a} &= W_{\mathrm{dec}} z + b_{\mathrm{dec}} \label{eq:sae-decode}
\end{align}
where $W_{\mathrm{dec}} \in \mathbb{R}^{d \times D}$. The ReLU activation in \eqref{eq:sae-encode} induces sparsity in $z$, with each component $z_i$ corresponding to a learned ``feature.''

For each context belonging to a susceptibility cluster, we extract a window of 60 tokens centered at the target position. Specifically, if the target token $y$ appears at position $p$ in the full context, we extract tokens from position $\max(0, p - 30)$ to $\min(L, p + 30)$ where $L$ is the sequence length. We then run Pythia-70M on this window and extract residual stream activations at each position, obtaining SAE latent activations $z^{(t)} \in \mathbb{R}^D$ for each token position $t$ in the window.

\begin{figure}[t]
      \centering
      \includegraphics[width=0.8\textwidth]{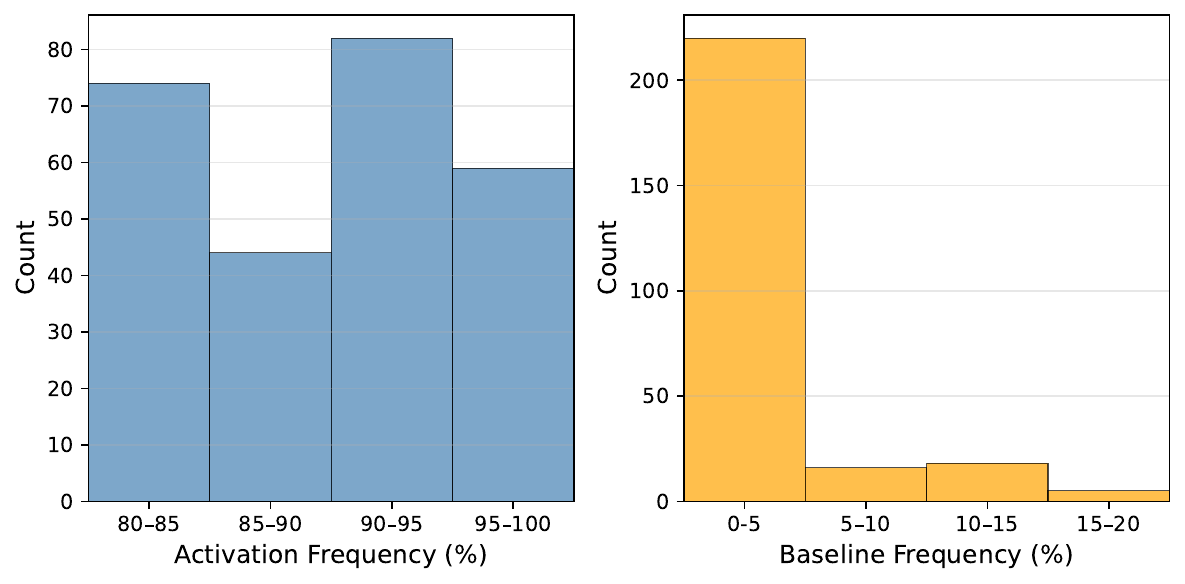}
      \caption{\textbf{Distribution of activation and baseline frequencies for matched SAE features}.
      \emph{Left}: Activation frequencies for 259 matched features.
      \emph{Right}: Baseline frequencies for the same features when measured across random clusters,
      demonstrating specificity with most baselines below 5\%.
      Matched features are defined as those with activation frequency $\geq 80\%$.}
      \label{fig:activation_frequencies}
  \end{figure}

\begin{figure}[p]
    \centering
    \includegraphics[width=\textwidth]{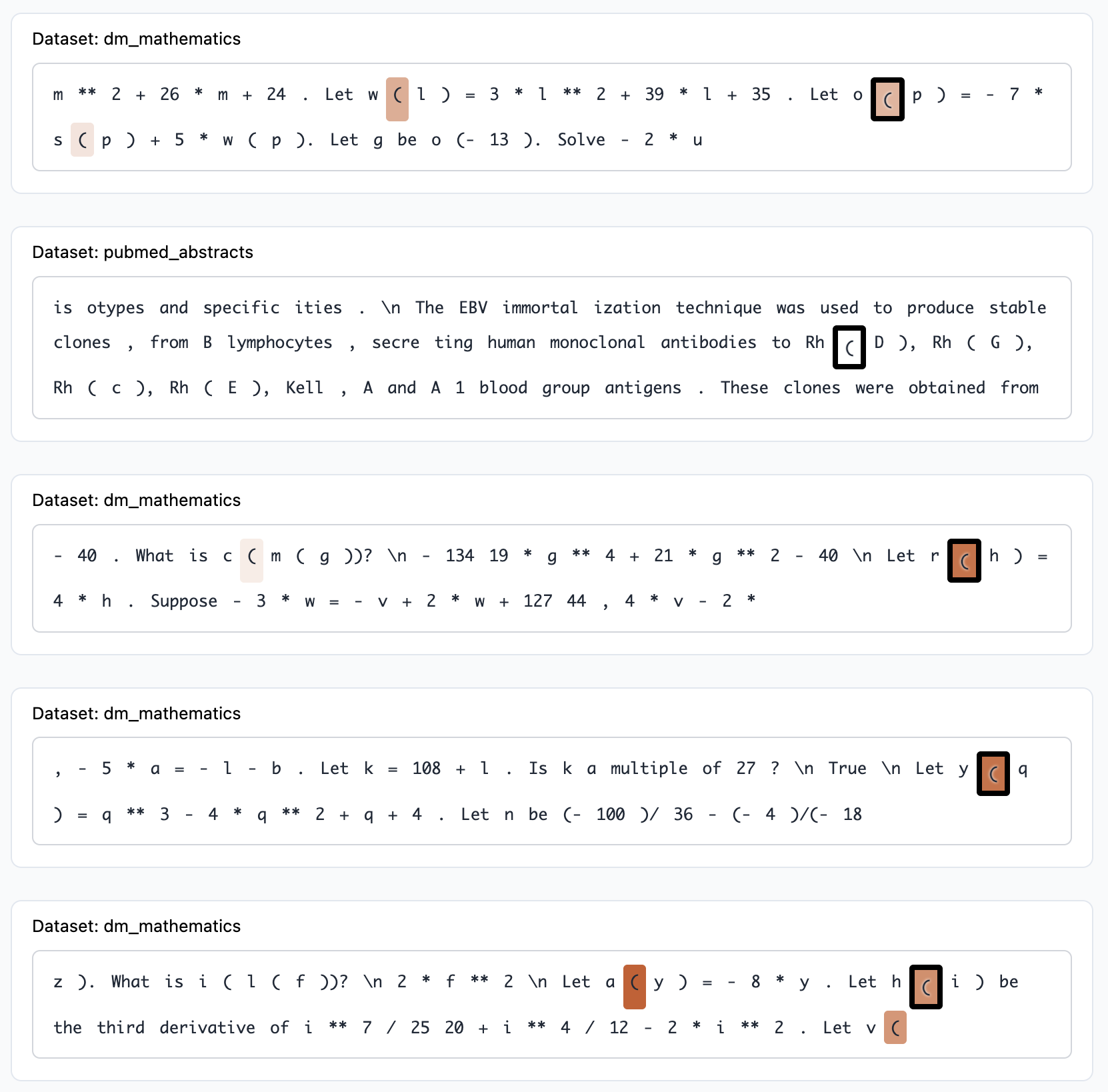}
    \vspace{1em}
    \caption{\textbf{Matching SAE for cluster \clu{95}}. Shown are five contexts from across \pilesub{dm\_mathematics} and \pilesub{pubmed\_abstracts} appearing in \clu{95}, with the outlined token being $y$. We show the activation of SAE feature 834 from layer $3$ of Pythia-70M, which is strongly activated on these $y$ tokens. The activation ranges from $0.0$ (white) to $6.33$ (orange). The label for this cluster is: opening parenthesis \tok{(} after function names in \dataset{dm\_mathematics} algebra problems. The activation frequency of this feature on this cluster is 93.3\% and baseline frequency 4.3\%.}
    \label{fig:cluster_95_matching}
\end{figure}

\begin{figure}[p]
    \centering
    \includegraphics[width=\textwidth]{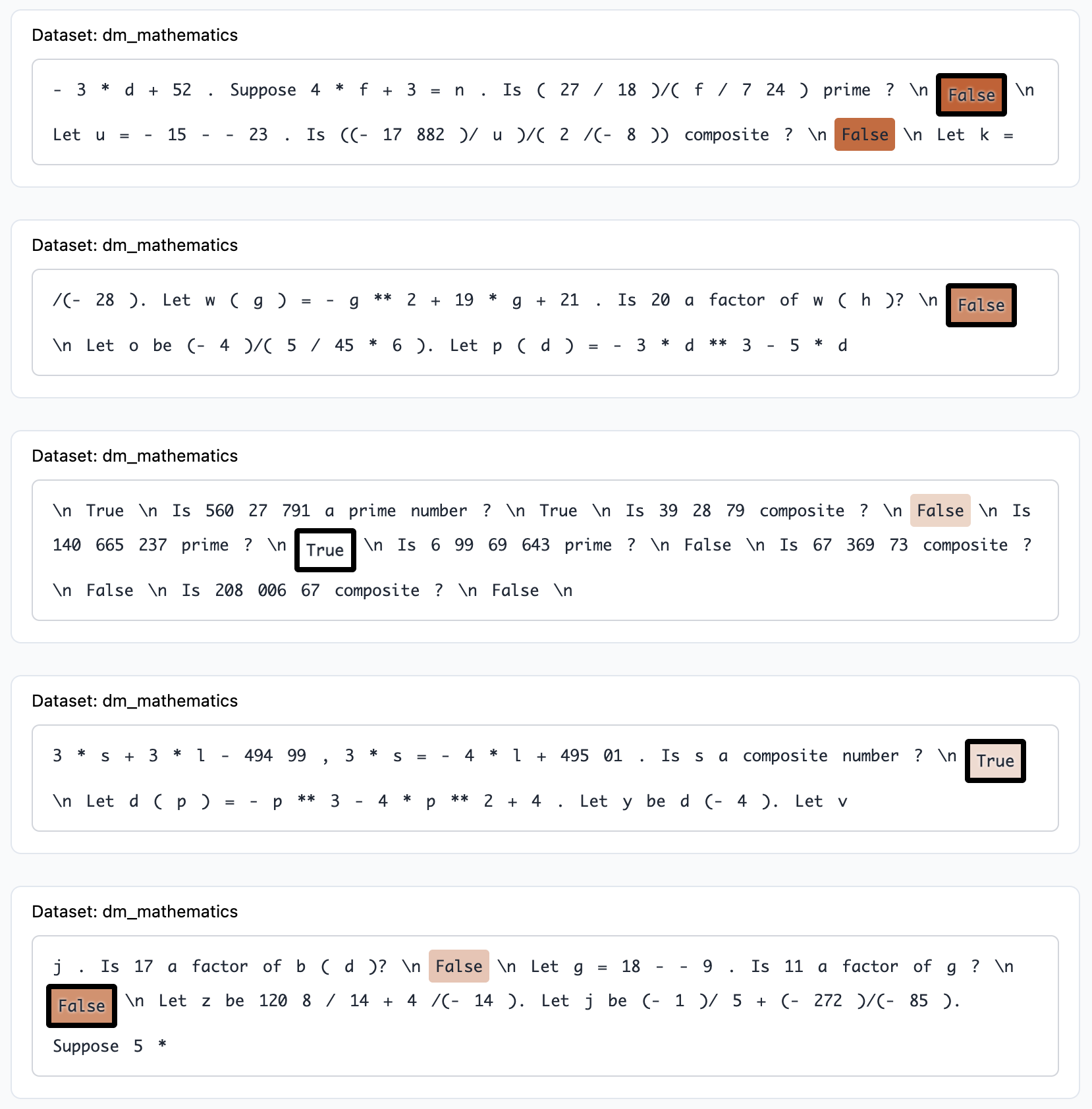}
    \vspace{1em}
    \caption{\textbf{Matching SAE for cluster \clu{250}}. Shown are five contexts from across \pilesub{dm\_mathematics} appearing in \clu{250}, with the outlined token being $y$. We show the activation of SAE feature 16089 from layer $4$ of Pythia-70M, which is strongly activated on these $y$ tokens. The activation ranges from $0.0$ (white) to $5.95$ (orange). The label for this cluster is: \tok{True}/\tok{False} boolean answers after newlines in \dataset{dm\_mathematics} math problems. The activation frequency of this feature on this cluster is 90\% and baseline frequency 0.2\%.}
    \label{fig:cluster_250_matching}
\end{figure}

To determine whether an SAE feature $i$ ``matches'' a susceptibility cluster, we introduce a binarized activation metric. For a given context, let $z_i^{(t)}$ denote the activation of feature $i$ at position $t$, and let $z_i^{(*)}$ denote its activation at the target token position. We compute:
\begin{enumerate}
    \item The 90th percentile $\tau_{90}$ of $\{z_i^{(t)}\}_{t=1}^{T}$ across all positions in the window
    \item An effective threshold $\tau = \max(\tau_{90}, 0.1)$
    \item A binary indicator $\mathbf{1}[z_i^{(*)} \geq \tau]$
\end{enumerate}
The \emph{activation frequency} of feature $i$ for a cluster $\mathcal{C}$ is then the fraction of contexts in $\mathcal{C}$ for which this indicator equals 1:
\begin{equation}
\mathrm{ActFreq}_i(\mathcal{C}) = \frac{1}{|\mathcal{C}|} \sum_{(x,y) \in \mathcal{C}} \mathbf{1}\big[z_i^{(*)}(x,y) \geq \tau(x,y)\big]\,.
\end{equation}
Intuitively, this measures how often feature $i$ is ``unusually active'' (relative to other positions in the same context) precisely at the target token. We say the SAE feature $i$ is a \emph{match} for cluster $\mathcal{C}$ if $\mathrm{ActFreq}_i(\mathcal{C}) \geq 0.8$.

To assess whether a matching feature is specific to a cluster rather than generically active, we compute a baseline activation frequency. For a given feature $i$ identified as matching cluster $\mathcal{C}$, we sample 20 clusters $\mathcal{C}_1, \ldots, \mathcal{C}_{20}$ uniformly at random from all clusters excluding $\mathcal{C}$. For each baseline cluster, we sample up to 30 contexts $\mathcal{C}'_j \subseteq \mathcal{C}_j$ and compute the same binarized activation measure. The baseline is then:
\begin{equation}
\mathrm{Baseline}_i = \frac{1}{\sum_j |\mathcal{C}'_j|} \sum_{j=1}^{20} \sum_{(x,y) \in \mathcal{C}'_j} \mathbf{1}\big[z_i^{(*)}(x,y) \geq \tau(x,y)\big]
\end{equation}
A feature with high activation frequency on its matched cluster but low baseline frequency provides evidence that the SAE has learned a representation corresponding to the same pattern captured by the susceptibility cluster. 

For each cluster, we analyze up to 30 sampled contexts and rank SAE features by their average activation at the target position. From the top-ranked features, we identify the first (if any) with activation frequency $\geq 0.8$ as the candidate match, and compute its baseline.

We find that out of 510 susceptibility clusters, 259 (50.8\%) have a matched SAE feature. Almost all the baseline frequencies for our matched SAE features are less than 5\% (\cref{fig:activation_frequencies}).

\clearpage

\newpage

\section{Gaussian Posterior Susceptibilities}\label{appendix:gaussian-posterior}

The SLT-based approach to interpretability developed in \citet{baker2025studyingsmalllanguagemodels,wang2025lang2pt5} and this paper rests on a key hypothesis: that the internal structure of neural networks is encoded in the geometry of the loss landscape, and can be accessed by estimating posterior expectation values. It is therefore crucial that our methods actually probe the \emph{posterior} -- which is shaped by the loss landscape -- rather than just a generic distribution around the trained weights.

In this appendix we test this by comparing susceptibilities to a simpler baseline that uses Gaussian noise instead of posterior samples. If the Gaussian baseline produced similar results, it would undermine the claim that loss landscape geometry matters; the fact that it performs substantially worse validates that the posterior structure is doing real work.

The per-token \emph{Gaussian Posterior Susceptibility} of a component $C$ for a token context pair $(x,y)$ is
\begin{equation}\label{GSP-def}
-\text{Cov}_{N(w^*, \lambda I)}\left[\phi_C, \ell_{xy}(w) - L(w)\right].
\end{equation}

This is identical to susceptibilities as usually defined, but the role of the quenched posterior has been replaced with a Gaussian distribution with uniform covariance.

Gaussian posterior susceptibilities are easily computed by modifying the experiment details in section \ref{appendix:susceptibilities-hparams} to have $n\beta=0$. They have a straightforward interpretation as a tool of mechanistic interpretability. Up to scale and constant term \eqref{GSP-def} measures the covariance between $\ell_{xy}$ and $L$ when the weights in component $C$ are perturbed by Gaussian noise.

We contrast the approach taken in the main text with this one by computing Gaussian posterior susceptibilities for Pythia-14M on the same set of data points, and repeating our analysis. \cref{fig:apostrophe_clusters_nbeta0} shows these new susceptibilities.
\begin{figure}
    \centering
    \includegraphics[width=0.9\linewidth]{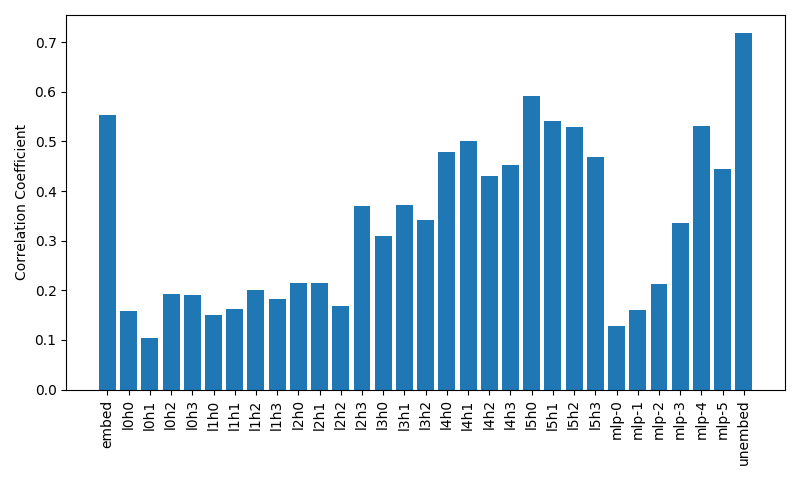}
    \caption{\textbf{Component-wise correlation between susceptibility types.} Pearson correlation between Gaussian posterior and usual susceptibilities, computed separately for each of the 32 components.}
    \label{fig:nb0_corr}
\end{figure}
We observe low to moderate correlation between the two datasets, as described in Figure \ref{fig:nb0_corr}.
\begin{figure}
    \centering
    \includegraphics[width=0.9\linewidth]{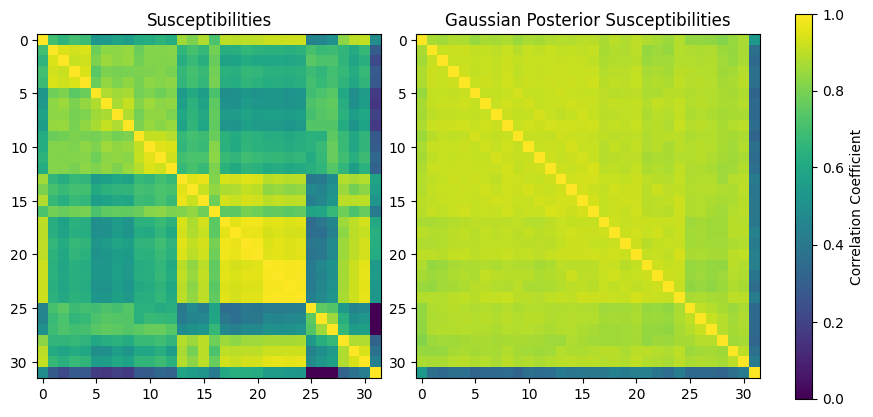}
    \caption{\textbf{Inter-component correlation structure.} Correlation matrices for usual susceptibilities (left) and Gaussian posterior susceptibilities (right) across the 32 components. The higher off-diagonal correlations on the right indicate that Gaussian posterior susceptibilities are less sensitive to differences between components.}
    \label{fig:nb0_corr_heatmap}
\end{figure}
We also observe that the component susceptibilities for Gaussian posterior susceptibilities are extremely well correlated with each other, as visible in Figure \ref{fig:nb0_corr_heatmap}. This is also seen in PCA analysis. The top PC for Pythia-14M susceptibility data explains 71\% of the variance, while the top PC for Gaussian posterior susceptibilities on the same dataset explains 87\% of the variance.

From this, we conclude that Gaussian posterior susceptibilities are \emph{meaningfully different} from susceptibilities done using the localized Gibbs posterior, and in many ways inferior.  The high inter-component correlations suggest that the Gaussian posterior susceptibilities are less sensitive to difference between components.

This affects cluster analysis as well, since the clustering algorithm implicitly depends on the Euclidean metric on susceptibility space, having multiple highly correlated coefficients can distort distances between points and make low conductance sets harder to identify.

Running the clustering algorithm on the Gaussian posterior susceptibilities for Pythia-14M identified 288 clusters, compared to the 510 found for the normal data.  There isn't a clear metric for cluster quality, but clusters found on normal susceptibility data were $72\%$ larger on average (mean size $976$ vs $567$). Further the clusters for normal susceptibilities contained, on average,  more diverse sets of tokens as measured by entropy. The clusters found on normal data had an average token entropy $31\%$ higher ($1.23$ vs $0.94$).

In \cref{fig:apostrophe_clusters_nbeta0} we see that in contrast to \cref{fig:apostrophe_clusters} there is no separation of tokens containing double quotes into opening and closing clusters. Nor do these tokens appear to have any particular organization in the body of the UMAP.

In \cref{fig:bracket_clusters_nbeta0} we compare the susceptibility data with the original hyperparameters, which we refer to as $n \beta > 0$, to that for $n \beta = 0$. We perform this comparison for sequences $xy$ where $y$ contains ``<'' when decoded. In the $n \beta > 0$ data we see a separation into two regions (A) for \tok{<} from opening tags (\clu{18}, \cluq{151}, \cluq{242}, \cluq{506}) and (B) for \tok{</} from closing tags (\clu{190},\cluq{266}). In the $n \beta = 0$ data there is a single \tok{</} cluster (\textbf{D87}) but no \tok{<} cluster for opening HTML tags, which is consistent with the more uniform appearance of these tokens in the susceptibility UMAP.

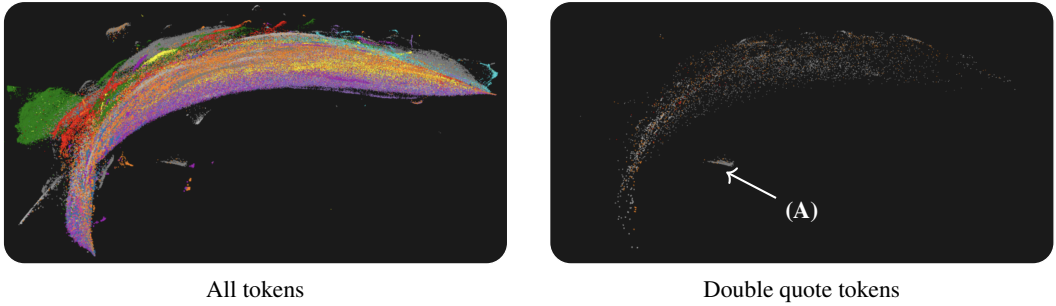
\begin{figure}[p]
    \centering
    \begin{tikzpicture}
        \begin{scope}
            \clip[rounded corners=8pt] (0,0) rectangle (0.48\textwidth,0.25\textwidth);
            \node[anchor=south west,inner sep=0] (image1) at (0,0) 
                {\includegraphics[width=0.48\textwidth]{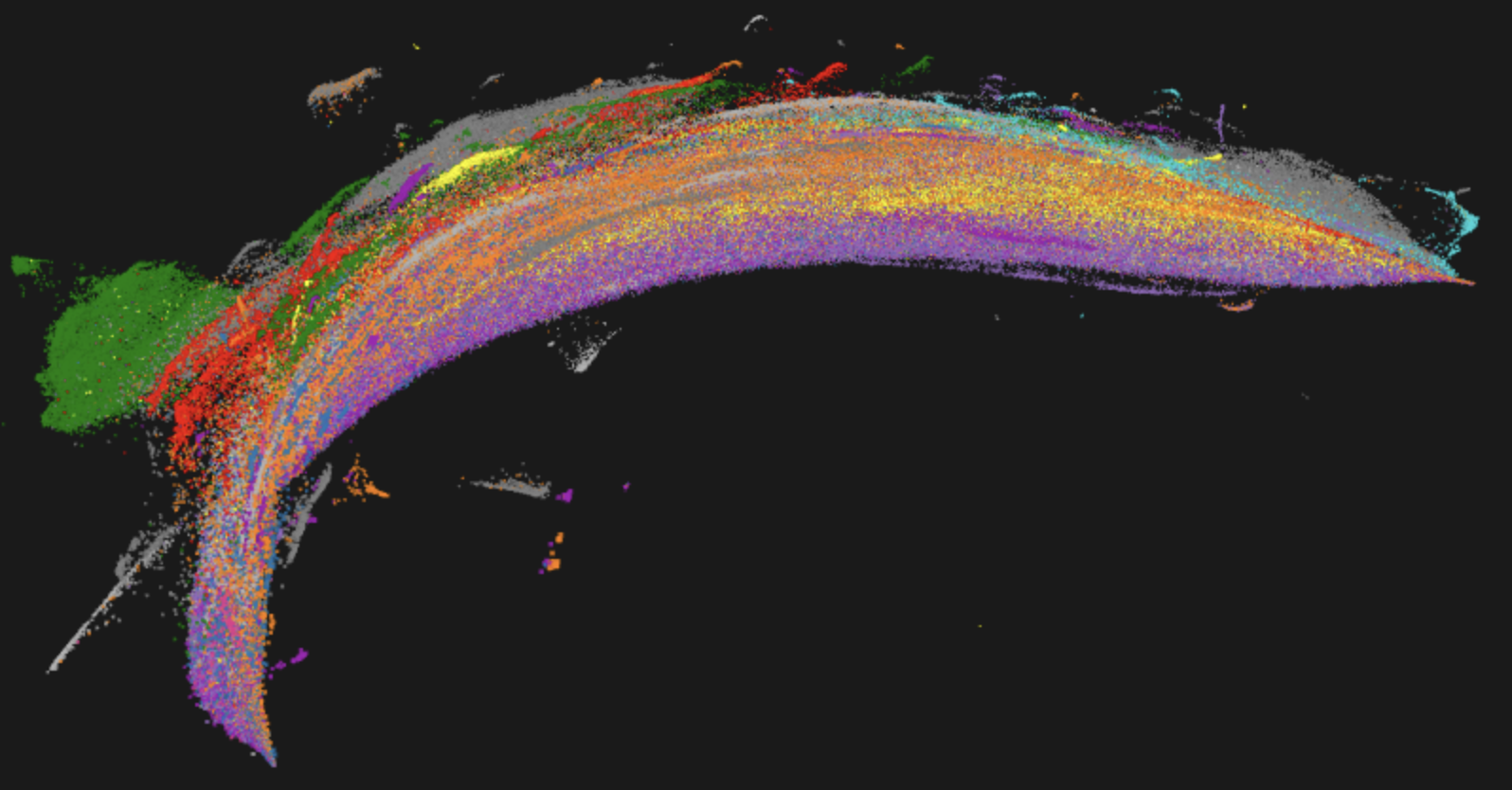}};
        \end{scope}
        \draw[white, line width=0.5pt, rounded corners=8pt] (0,0) rectangle (0.48\textwidth,0.25\textwidth);
        
        \begin{scope}[shift={(0.52\textwidth,0)}]
            \clip[rounded corners=8pt] (0,0) rectangle (0.48\textwidth,0.25\textwidth);
            \node[anchor=south west,inner sep=0] (image2) at (0,0) 
                {\includegraphics[width=0.48\textwidth]{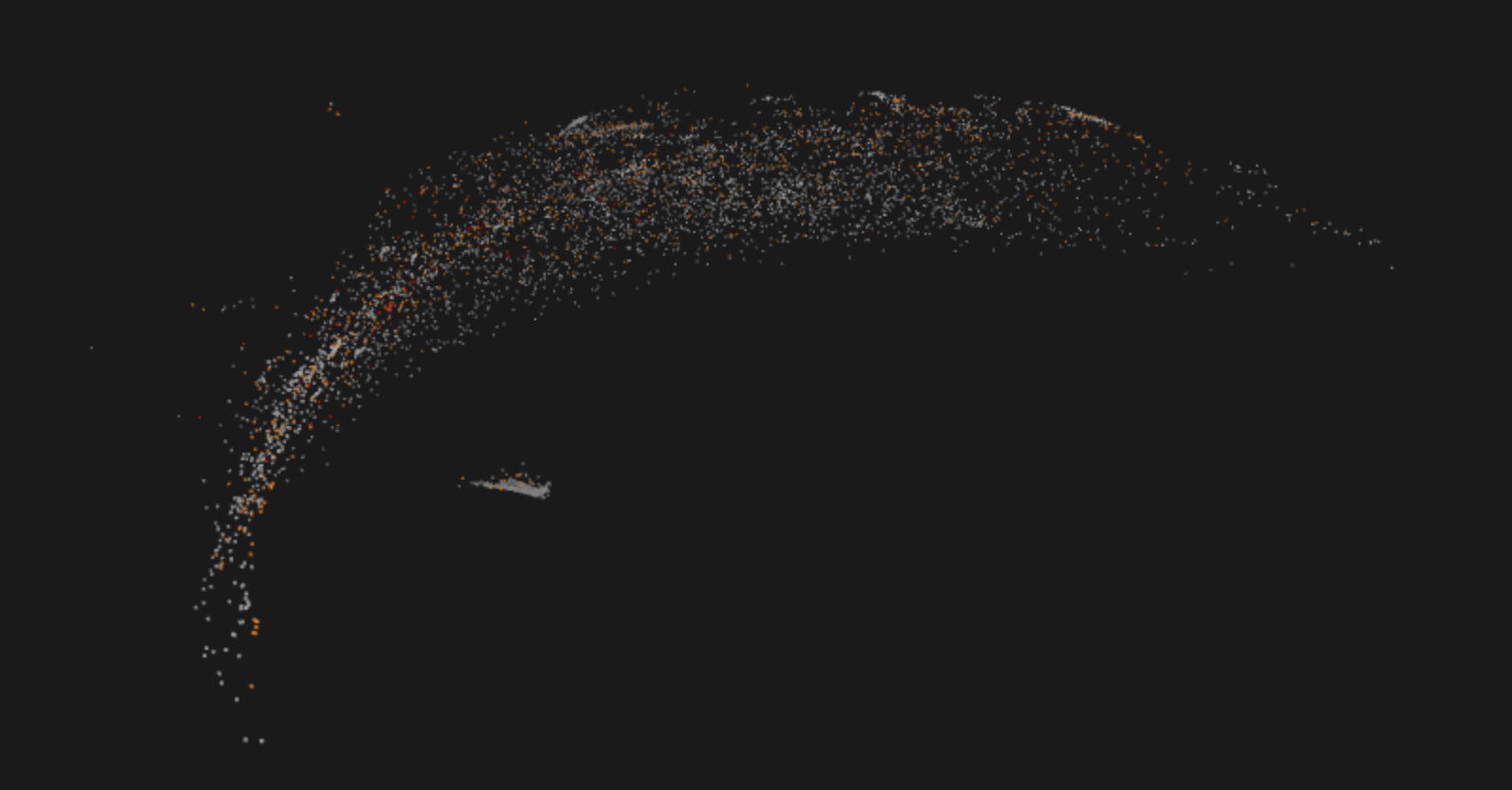}};
        \end{scope}
        \draw[white, line width=0.5pt, rounded corners=8pt] (0.52\textwidth,0) rectangle (\textwidth,0.25\textwidth);

        \begin{scope}[shift={(0.52\textwidth,0)}]
            \pgfmathsetmacro{\imgw}{0.48*\textwidth}
            \pgfmathsetmacro{\imgh}{0.25*\textwidth}
            
            \node[white, font=\bfseries\small] (A) at (0.5*0.48\textwidth,0.2*0.25\textwidth) {(A)};
            \draw[->,thick,white] (A) -- (0.35*0.48\textwidth,0.35*0.25\textwidth);
        \end{scope}
        
        \node[below=0.1cm, font=\small] at (0.24\textwidth,0) {All tokens};
        \node[below=0.1cm, font=\small] at (0.76\textwidth,0) {Double quote tokens};
        
    \end{tikzpicture}
    
\vspace{0.2em}

    \caption{
        \textbf{Double quote token clusters in $n \beta = 0$ data.} 
        \emph{Left:} Full low-dimensional representation of susceptibility vectors computed from Pythia-14M, with points colored by pattern type. 
        \emph{Right:} Filtered view showing only tokens which, when decoded, contain double quotes (e.g. \tok{"}, \tok{."}). The cluster (A) consists of \tok{="} tokens. We do not see separate clusters for \tok{"} with the sense ``opening'' vs ``closing'' in this data. Compare with \cref{fig:apostrophe_clusters}.
    }
    \label{fig:apostrophe_clusters_nbeta0}
\end{figure}

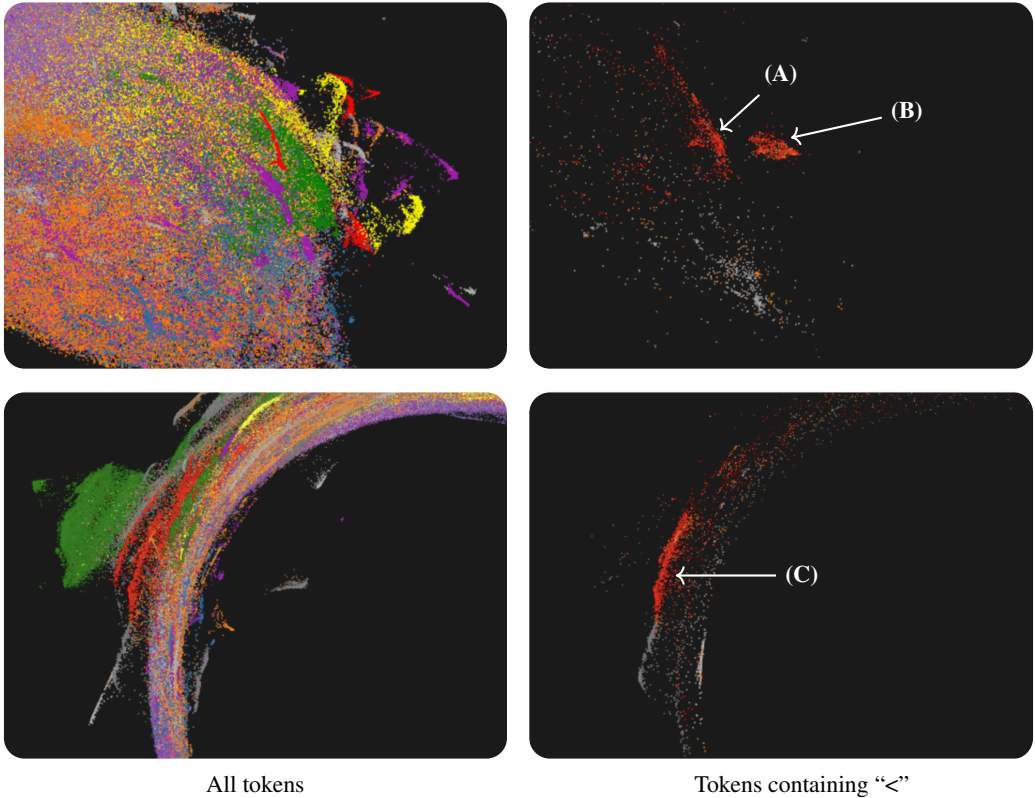
\begin{figure}[p]
    \centering
    \begin{tikzpicture}
        \begin{scope}
            \clip[rounded corners=8pt] (0,0) rectangle (0.48\textwidth,0.35\textwidth);
            \node[anchor=south west,inner sep=0] (image1) at (0,0) 
                {\includegraphics[width=0.48\textwidth]{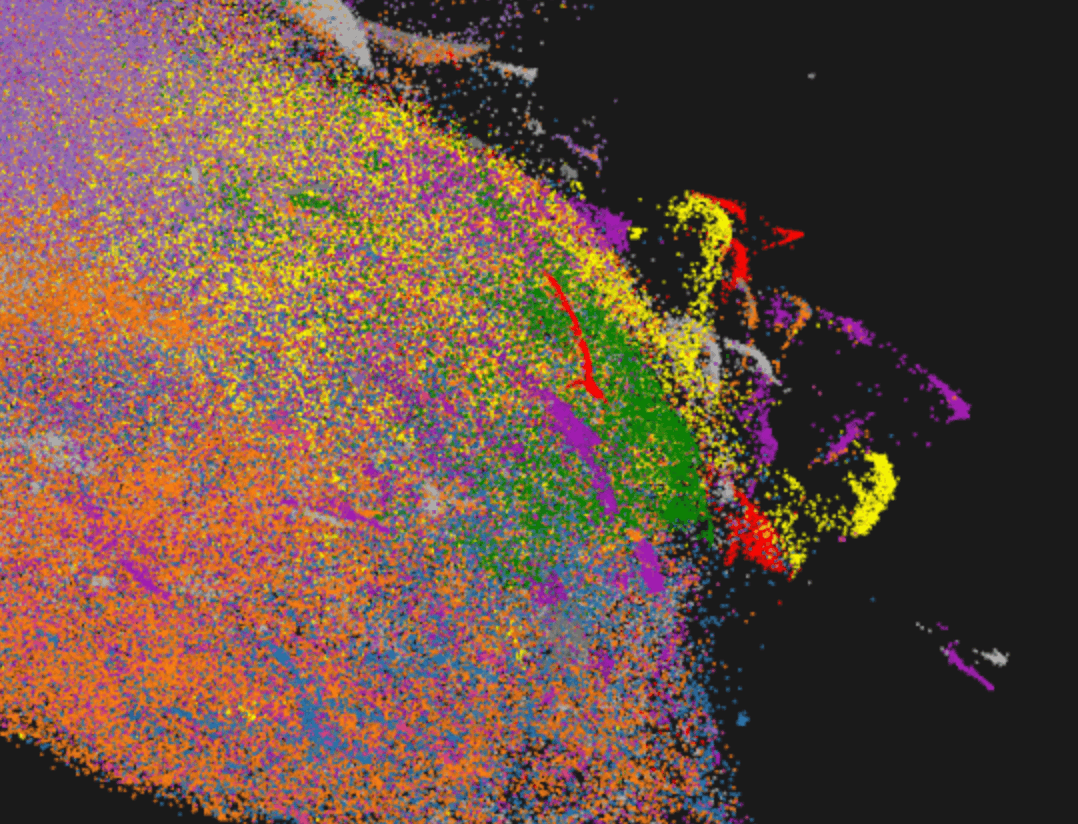}};
        \end{scope}
        \draw[white, line width=0.5pt, rounded corners=8pt] (0,0) rectangle (0.48\textwidth,0.35\textwidth);
        
        \begin{scope}[shift={(0.5\textwidth,0)}]
            \clip[rounded corners=8pt] (0,0) rectangle (0.48\textwidth,0.35\textwidth);
            \node[anchor=south west,inner sep=0] (image2) at (0,0) 
                {\includegraphics[width=0.48\textwidth]{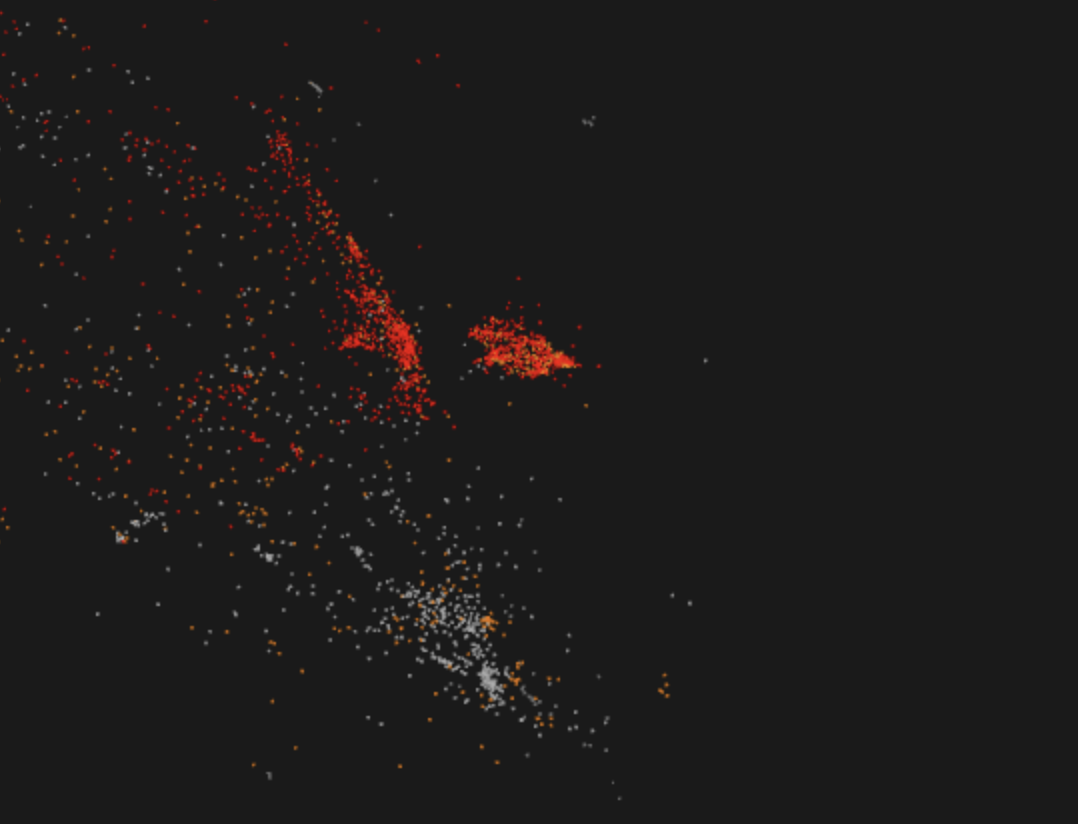}};
        \end{scope}
        \draw[white, line width=0.5pt, rounded corners=8pt] (0.5\textwidth,0) rectangle (\textwidth,0.35\textwidth);

        \begin{scope}[shift={(0.5\textwidth,0)}]
            \pgfmathsetmacro{\imgw}{0.48*\textwidth}
            \pgfmathsetmacro{\imgh}{0.35*\textwidth}
            
            \node[white, font=\bfseries\small] (A) at (0.5*0.48\textwidth,0.8*0.35\textwidth) {(A)};
            \draw[->,thick,white] (A) -- (0.38*0.48\textwidth,0.65*0.35\textwidth);

            \node[white, font=\bfseries\small] (B) at (0.75*0.48\textwidth,0.7*0.35\textwidth) {(B)};
            \draw[->,thick,white] (B) -- (0.52*0.48\textwidth,0.63*0.35\textwidth);
        \end{scope}        
    \end{tikzpicture}

    \vspace{0.7em}

    \begin{tikzpicture}
        \begin{scope}
            \clip[rounded corners=8pt] (0,0) rectangle (0.48\textwidth,0.35\textwidth);
            \node[anchor=south west,inner sep=0] (image1) at (0,0) 
                {\includegraphics[width=0.48\textwidth]{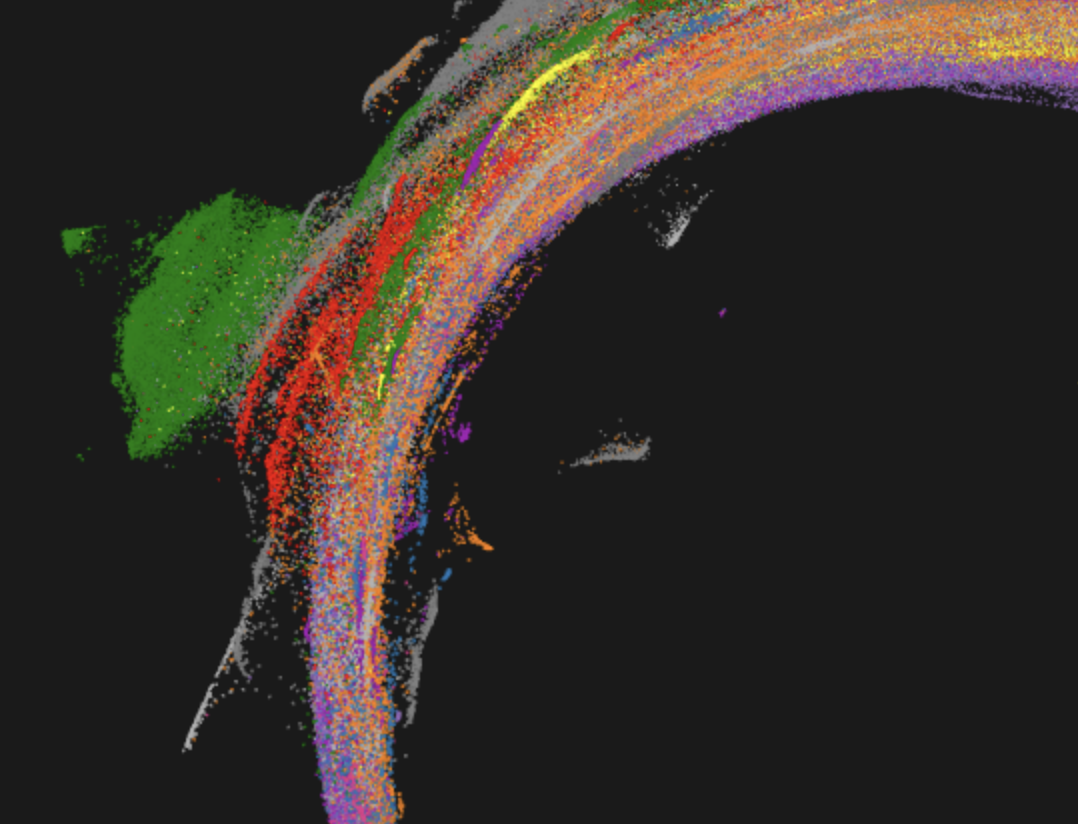}};
        \end{scope}
        \draw[white, line width=0.5pt, rounded corners=8pt] (0,0) rectangle (0.48\textwidth,0.35\textwidth);
        
        \begin{scope}[shift={(0.5\textwidth,0)}]
            \clip[rounded corners=8pt] (0,0) rectangle (0.48\textwidth,0.35\textwidth);
            \node[anchor=south west,inner sep=0] (image2) at (0,0) 
                {\includegraphics[width=0.48\textwidth]{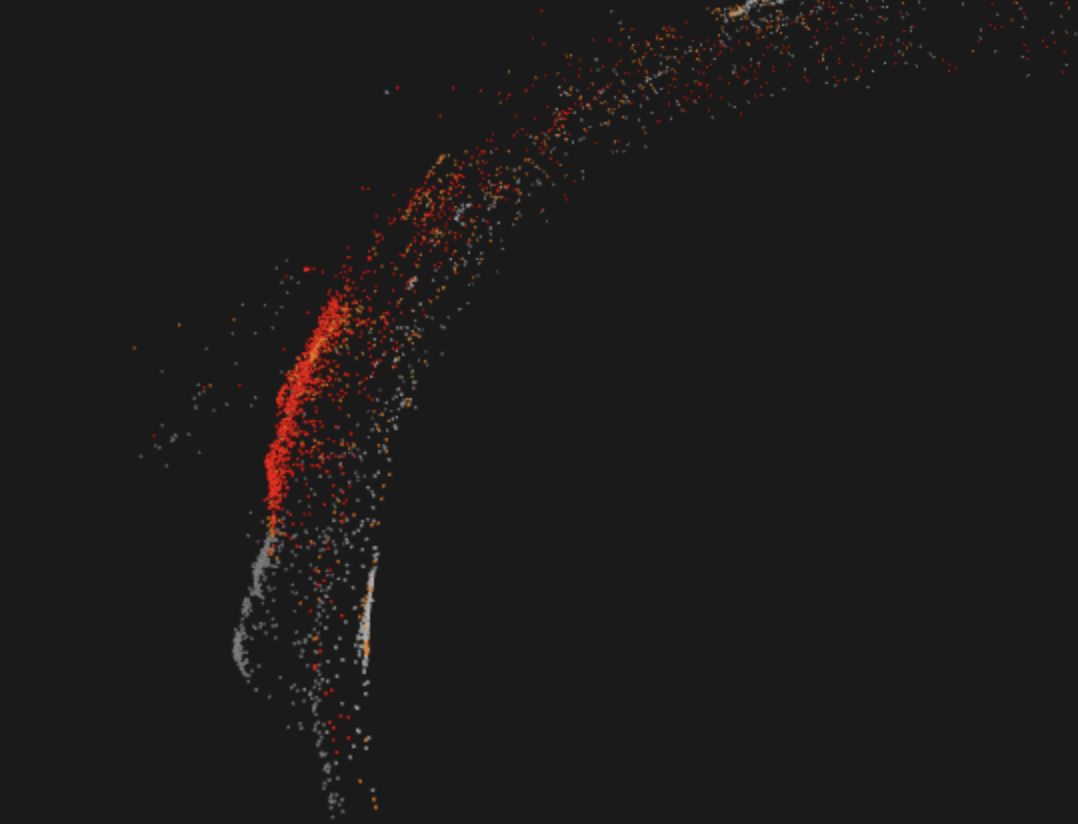}};
        \end{scope}
        \draw[white, line width=0.5pt, rounded corners=8pt] (0.5\textwidth,0) rectangle (\textwidth,0.35\textwidth);

        \begin{scope}[shift={(0.52\textwidth,0)}]
            \pgfmathsetmacro{\imgw}{0.48*\textwidth}
            \pgfmathsetmacro{\imgh}{0.35*\textwidth}
            
            \node[white, font=\bfseries\small] (C) at (0.5*0.48\textwidth,0.5*0.35\textwidth) {(C)};
            \draw[->,thick,white] (C) -- (0.25*0.48\textwidth,0.5*0.35\textwidth);
        \end{scope}
        
        \node[below=0.1cm, font=\small] at (0.24\textwidth,0) {All tokens};
        \node[below=0.1cm, font=\small] at (0.76\textwidth,0) {Tokens containing ``<''};
        
    \end{tikzpicture}
    
\vspace{0.2em}

    \caption{
        \textbf{Tokens containing ``<'' in both $n \beta = 0$ and $n \beta > 0$ data.} \emph{First row:} data from Pythia-14M, usual hyperparameters. \emph{Second row:} data from Pythia-14M with hyperparameter $n\beta = 0$. \emph{Left:} Full low-dimensional representation of susceptibility vectors colored by pattern type. 
        \emph{Right:} Filtered view showing only tokens when decoded contain ``<'' (e.g. \tok{<}, \tok{</}). Region (A) contains \tok{<} from opening tags while (B) contains tokens \tok{</} from closing HTML tags. Region (C) contains a mix of both \tok{</}, \tok{<} with more of the latter towards the bottom of the region and more of the former towards the top.
    }
    \label{fig:bracket_clusters_nbeta0}
\end{figure}

\clearpage

\section{Clusters}

\subsection{Pythia-14M clusters}\label{section:cluster_table}

We divide token clusters into four main categories (and implicitly ``Other''). These are
\begin{itemize}
\item Grammar: Syntax and grammatical patterns
\item Math: Mathematical and scientific notation
\item Code: Code, markup and technical syntax
\item Formatting: Document structure and formatting
\item Bigrams: Fixed phrases and bigrams (examples in \cref{tab:bigram_clusters})
\end{itemize}
For details of this taxonomy see \cref{section:results_pythia_14m_taxonomy}. The complete list of clusters is given in \cref{table:pythia14m_cluster_list}. Note that \clu{x} stands for the cluster with index $1 \le x \le 510$ and \sae{a}{b} stands for the layer $a$ Pythia-70M SAE feature with index $b$. When an SAE feature appears in the row for a cluster it means they are a match in the sense of \cref{section:sae-methodology}.

\subsection{Bigrams and fixed phrases}

The simplest example of a mode in \citet{modes2} is an ``absolute'' bigram, where a pair of tokens only ever occur together. Given the conjectured relation between modes and clusters we therefore expect that bigrams will appear in classifications of clusters in susceptibility UMAPs.

And indeed, the set of clusters for Pythia-14M contains numerous examples of bigrams and fixed phrases, by which we mean a cluster $\mathcal{C}$ consisting of token sequences $xy$ where $y$ is the same across almost all token sequences in the cluster, and $x = x'x''$ for some common sequence $x'' \in \Sigma^{k'}$. In \cref{tab:bigram_clusters} we enumerate all the bigrams and fixed phrases found among our clusters (see \cref{section:cluster_table} for the complete list).

\begin{table}[t]
\centering
\begin{tabular}{llcl}
\toprule
\textbf{Category} & \textbf{Example Bigram} & \textbf{Cluster} & \textbf{$y$ Token} \\
\midrule
\multirow{4}{*}{\textbf{Fixed Idioms}} 
 & ``United \textbf{States}'' & \clu{373} & \tok{~States} \\
 & ``based \textbf{on}'' & \clu{99} & \tok{~on} \\
 & ``associated \textbf{with}'' & \clu{346}, \cluq{355} & \tok{~with} \\
 & ``such \textbf{as}'' & \clu{236} & \tok{~as} \\
 & ``in \textbf{order}'' & \clu{432} & \tok{~order} \\
 & ``a \textbf{lot}'' & \clu{146} & \tok{~lot} \\
 & ``of \textbf{these}'' & \clu{357} & \tok{~these} \\
 & ``the \textbf{following}'' & \clu{359} & \tok{~following} \\
 & ``a \textbf{new}'' & \clu{468} & \tok{~new} \\
\midrule
\multirow{4}{*}{\textbf{Conversational}} 
 & ``Thank \textbf{you}'' & \clu{457} & \tok{~you} \\
 & ``I \textbf{think}'' & \clu{262} & \tok{~think} \\
 & ``I \textbf{have}'' & \clu{433} & \tok{~have} \\
 & ``to \textbf{do}'' & \clu{376} & \tok{~do} \\
 
\midrule
\multirow{4}{*}{\textbf{Math Setup}} 
 & ``derivative \textbf{of}'' & \clu{77} & \tok{~of} \\
 & ``Let [...] \textbf{be}'' & \clu{425} & \tok{~be} \\
 & ``Collect the \textbf{terms}'' & \clu{291} & \tok{~terms} \\
 & ``In \textbf{base}'' & \clu{361} & \tok{~base} \\
 & ``How \textbf{many}'' & \clu{470} & \tok{~many} \\
\midrule
\multirow{3}{*}{\textbf{Syntactic ``The''}} 
 & ``that \textbf{the}'' & \clu{53} & \tok{~the} \\
 & ``to \textbf{the}'' & \clu{408} & \tok{~the} \\
 & ``is \textbf{the}'' & \clu{270} & \tok{~the} \\
 & ``be \textbf{the}'' & \clu{364} & \tok{~the} \\
\bottomrule
\end{tabular}
\vspace{1em}
\caption{\textbf{Bigrams and fixed phrases.} A significant minority of discovered clusters correspond to rigid bigrams and idiomatic phrases where the continuation is highly deterministic. This suggests the model processes common multi-token sequences (like ``United States'' or ``based on'') as atomic units.}
\label{tab:bigram_clusters}
\end{table}

\definecolor{tableShade}{gray}{0.98}
\rowcolors{3}{white}{tableShade}  
\renewcommand{\arraystretch}{1.3}

\begin{longtable}{lp{0.6\textwidth}ll}
\toprule
Cluster & Description & Category & SAE \\
\midrule
\endfirsthead

\toprule
Cluster & Description & Category & SAE\\
\midrule
\endhead

\midrule
\multicolumn{4}{r}{\textit{Continued on next page}} \\
\endfoot

\bottomrule
\endlastfoot

\textbf{C1} & Auxiliary verbs (\tok{~is}, \tok{~have}, \tok{~has}, \tok{~are}) after subject pronouns/nouns & Grammar & --\\
\textbf{C2} & \tok{**} exponentiation operator after variables in \dataset{dm\_mathematics} and code & Math & \sae{3}{23541}\\
\textbf{C3} & Date separator \tok{/} in formatted dates and paths & Formatting & --\\
\textbf{C4} & Indefinite article \tok{~a} following prepositions in technical/formal text & Grammar & \sae{4}{6774}\\
\textbf{C5} & Underscore \tok{\_} in code/LaTeX formatting (\dataset{github-all}, \dataset{arxiv}, \dataset{stackexchange}) & Code & --\\
\textbf{C6} & Opening bracket of parenthetical in text & Formatting & \sae{2}{7291}\\
\textbf{C7} & Minus sign in mathematical expressions, opening brackets in math and combinations of the two & Math & --\\
\textbf{C8} & Hyphen \tok{-} after prefixes and in compound terms across scientific/technical text & Grammar & --\\
\textbf{C9} & Token \tok{:} in "Category:" formatting in \dataset{wikipedia\_en} & Formatting & \sae{4}{22002}\\
\textbf{C10} & Single-letter variable names in \dataset{dm\_mathematics} equations & Math & --\\
\textbf{C11} & Prepositions (\tok{~to}, \tok{~on}, \tok{~by}, \tok{~in}, \tok{~for}) after passive/descriptive phrases & Grammar & --\\
\textbf{C12} & Plural verb forms (\tok{~are}, \tok{~have}, \tok{~were}) after relative clauses and subjects & Grammar & --\\
\textbf{C13} & Mathematical operators and numbers following arithmetic symbols in \dataset{dm\_mathematics}, \dataset{arxiv}, and \dataset{github-all} & Math & --\\
\textbf{C14} & Modal and auxiliary verbs after subjects in technical/legal text & Grammar & --\\
\textbf{C15} & Multiplication operator \tok{*} and formatting tokens in mathematical expressions across \dataset{dm\_mathematics}, \dataset{github-all}, and \dataset{arxiv} & Math & --\\
\textbf{C16} & Closing parentheses and brackets after numbers in technical/scientific text & Formatting & --\\
\textbf{C17} & Indentation spacing (8 spaces) after newlines in \dataset{github-all}, \dataset{enron\_emails}, \dataset{stackexchange} & Formatting & \sae{2}{30019}\\
\textbf{C18} & \tok{<} XML/HTML opening bracket after whitespace in \dataset{github-all} and \dataset{stackexchange} & Code & \sae{2}{20843}\\
\textbf{C19} & Comparison words (\tok{~smaller}, \tok{~greater}, \tok{~bigger}) after "is" in \dataset{dm\_mathematics} & Math & \sae{2}{2185}\\
\textbf{C20} & Closing parentheses after citations, references, and expressions in \dataset{pubmed\_abstracts}, \dataset{stackexchange}, \dataset{arxiv}, \dataset{github-all}, and \dataset{pubmed\_central} & & --\\
\textbf{C21} & End of text token after double newlines in \dataset{stackexchange} and \dataset{hackernews} & Formatting & \sae{2}{3257}\\
\textbf{C22} & Double newline paragraph breaks in \dataset{cc}, \dataset{wikipedia\_en}, \dataset{hackernews}, \dataset{arxiv}, and \dataset{github-all} & Formatting & --\\
\textbf{C23} & \tok{.} period after numbers in \dataset{dm\_mathematics} & Formatting & --\\
\textbf{C24} & Preposition \tok{~for} after purpose/use descriptions in technical and discussion contexts & Grammar & --\\
\textbf{C25} & Token \tok{:} in "Q:" formatting pattern in \dataset{stackexchange} question headers & Formatting & \sae{2}{25505}\\
\textbf{C26} & Article \tok{~an} following prepositions and verbs in technical/legal text & Grammar & \sae{4}{4098}\\
\textbf{C27} & Hyphen after \tok{~non} prefix in scientific/technical text & Formatting & \sae{4}{32155}\\
\textbf{C28} & Prepositions after commas (\tok{~as}, \tok{~from}, \tok{~with}, \tok{~into}, \tok{~using}) & Grammar & --\\
\textbf{C29} & Token \tok{.} playing various roles except the end of a sentence & & --\\
\textbf{C30} & \tok{~of} after nouns/phrases indicating possession, quantity, or part-whole relationships & Grammar & --\\
\textbf{C31} & Modal verb followed by \tok{~be} in technical/scientific text & Grammar & \sae{3}{5228}\\
\textbf{C32} & Token \tok{in} from "\textbackslash{}in" set membership symbol in \dataset{arxiv} and \dataset{stackexchange} & Math & \sae{2}{11825}\\
\textbf{C33} & \tok{fig} after \tok{="} in \dataset{pubmed\_central} figure references & Code & --\\
\textbf{C34} & Pipe \tok{|} delimiter in \dataset{wikipedia\_en} tables and formatting & Formatting & --\\
\textbf{C35} & Assignment operator after closing parenthesis in \dataset{dm\_mathematics} & Math & \sae{4}{27936}\\
\textbf{C36} & Sentence-initial capitalized words after periods in formal/technical text & Math & --\\
\textbf{C37} & Closing parenthesis \tok{)} after function arguments in \dataset{dm\_mathematics} and code, as well as closing parenthesis on multiple choice questions & Math & --\\
\textbf{C38} & Colons and semicolons following dates and times and email headers & Formatting & --\\
\textbf{C39} & \tok{this} keyword induction in JavaScript code (\dataset{github-all}) & Code & \sae{4}{22721}\\
\textbf{C40} & Exponents (3-5) after \tok{**} in \dataset{dm\_mathematics} polynomial expressions & Math & \sae{3}{31787}\\
\textbf{C41} & Passive voice construction (\tok{~by} after past participles) & Grammar & --\\
\textbf{C42} & Double newline after closing brace \tok{\}} in \dataset{github-all} and \dataset{stackexchange} code & Formatting & \sae{2}{17680}\\
\textbf{C43} & Sentence-initial tokens after periods in expository text & & --\\
\textbf{C44} & LaTeX/code closing delimiters and formatting tokens in \dataset{arxiv} and \dataset{stackexchange} & Math & --\\
\textbf{C45} & End of document/abstract in scientific and technical text (\tok{<|endoftext|>} after \tok{.} in \dataset{pubmed\_abstracts}, \dataset{nih\_exporter}, \dataset{uspto\_backgrounds}) & Formatting & --\\
\textbf{C46} & \tok{www} after \tok{://} in URLs across \dataset{hackernews}, \dataset{enron\_emails}, and \dataset{stackexchange} & Code & \sae{4}{12160}\\
\textbf{C47} & Closing bracket \tok{]} with various compounded tokens & Formatting & --\\
\textbf{C48} & Whitespace indentation after newlines in code/technical documents & Code & --\\
\textbf{C49} & LaTeX math delimiter \tok{\$} opening a math environment in \dataset{arxiv} and \dataset{stackexchange} & Math & --\\
\textbf{C50} & Preposition \tok{~to}/\tok{~by} after "divided" or "with respect" in \dataset{dm\_mathematics} and technical texts & Grammar & --\\
\textbf{C51} & Variable names \tok{s} and \tok{t} after \tok{*} in \dataset{dm\_mathematics} & Math & \sae{2}{29842}\\
\textbf{C52} & HTML/XML closing angle bracket \tok{>} after tag names in \dataset{github-all} and \dataset{stackexchange} & Code & \sae{2}{28829}\\
\textbf{C53} & Token \tok{~the} as part of "that the" in technical/legal documents & Bigrams & \sae{4}{18242}\\
\textbf{C54} & Equals sign \tok{~=} after numbers in \dataset{dm\_mathematics} equations & Math & --\\
\textbf{C55} & Conjunction \tok{~and} after commas and list items & Grammar & \sae{2}{1172}\\
\textbf{C56} & HTML closing angle bracket \tok{">} after \tok{.html} in href attributes & Code & \sae{3}{24737}\\
\textbf{C57} & Multiplication operator \tok{*} in \dataset{dm\_mathematics} algebraic expressions & Math & \sae{4}{23363}\\
\textbf{C58} & Variable names (\tok{h}, \tok{u}, \tok{y}, \tok{o}, \tok{w}) after \tok{*} in \dataset{dm\_mathematics} equations & Math & --\\
\textbf{C59} & Conjunction \tok{~and} after comma in scientific/technical lists & Grammar & \sae{2}{1172}\\
\textbf{C60} & Indentation whitespace after newlines in code/structured text & Formatting & --\\
\textbf{C61} & Newline \tok{\textbackslash n} after punctuation/delimiters in code and structured text & Code & --\\
\textbf{C62} & Whitespace and indentation continuation after newlines in \dataset{github-all} and \dataset{stackexchange} & Code & --\\
\textbf{C63} & Modal verbs (\tok{will}, \tok{is}, \tok{should}, \tok{would}, \tok{must}) following subjects in technical/professional text & Grammar & --\\
\textbf{C64} & Prepositions after verbs/nouns (\tok{~with}, \tok{~as}, \tok{~between}, \tok{~using}) in technical/scientific text & Grammar & --\\
\textbf{C65} & Hexadecimal \tok{00} after \tok{0x} prefix in \dataset{github-all} & Code & --\\
\textbf{C66} & \tok{~other}, \tok{~others}, \tok{~them} after quantifiers/conjunctions in diverse technical/legal texts & Grammar & --\\
\textbf{C67} & Variable name \tok{q} in \dataset{dm\_mathematics} algebra problems & Math & \sae{2}{29842}\\
\textbf{C68} & LaTeX math mode closing delimiter \tok{\$} in \dataset{arxiv} and \dataset{stackexchange} & Math & \sae{2}{8199}\\
\textbf{C69} & Closing quotation marks of various kinds after quoted phrases in \dataset{freelaw} and \dataset{cc} & Grammar & --\\
\textbf{C70} & \tok{Let} followed by \tok{~be} in \dataset{dm\_mathematics} variable definitions & Math & \sae{3}{7888}\\
\textbf{C71} & Tabs and whitespace after newlines in code/structured text & Code & --\\
\textbf{C72} & Variable names after \tok{*} or \tok{(} in \dataset{dm\_mathematics} equations & Math & --\\
\textbf{C73} & Document/sentence start tokens after \tok{<|endoftext|>} in \dataset{pubmed\_abstracts} and \dataset{nih\_exporter} & Formatting & --\\
\textbf{C74} & Pronouns \tok{~his}/\tok{~he} following references to male individuals & Grammar & --\\
\textbf{C75} & Modal verb completions (\tok{~be}, \tok{~=}, \tok{~not}, \tok{~a}) following auxiliary verbs (\tok{can}, \tok{may}, \tok{would}, \tok{should}, \tok{will}) in technical/scientific text & Grammar & --\\
\textbf{C76} & Percentage symbols and repeated zeros in scientific/technical text & Math & --\\
\textbf{C77} & Token \tok{~of} as part of the bigram "derivative of" in \dataset{dm\_mathematics} calculus problems & Math & \sae{2}{32107}\\
\textbf{C78} & Contractions and possessives after pronouns/nouns (\tok{'s}, \tok{'ll}, \tok{'d}, \tok{'ve}) & Grammar & --\\
\textbf{C79} & Capitalized words/names following punctuation or prepositions across \dataset{wikipedia\_en}, \dataset{cc}, \dataset{enron\_emails}, \dataset{pubmed\_abstracts}, \dataset{nih\_exporter} & Grammar & --\\
\textbf{C80} & Closing parentheses and brackets after expressions & Formatting & --\\
\textbf{C81} & Opening parenthesis \tok{(} after function names in \dataset{dm\_mathematics} equations & Math & --\\
\textbf{C82} & Preposition \tok{~to} and numbers after hyphens/references & Grammar & --\\
\textbf{C83} & Newlines in a variety of contexts. Sometimes the first of two newlines & Formatting & --\\
\textbf{C84} & \textit{No description available} & --\\
\textbf{C85} & LaTeX/math notation delimiters and escaped formatting in \dataset{arxiv} and \dataset{stackexchange} & Math & --\\
\textbf{C86} & Opening brace \tok{\{} after function/control structure in code & Code & \sae{2}{19285}\\
\textbf{C87} & Question marks ending math problems in \dataset{dm\_mathematics} & Math & \sae{2}{5385}\\
\textbf{C88} & \tok{~get} and other common verbs after infinitive markers in \dataset{hackernews}, \dataset{cc}, \dataset{stackexchange}, \dataset{enron\_emails} & Grammar & --\\
\textbf{C89} & Numeric identifiers in \dataset{pubmed\_central} reference codes (e.g., \tok{molecules-22-00821}, \tok{pone.0196349}) & & --\\
\textbf{C90} & \tok{HOU} after \tok{/} in Enron email address paths & & \sae{4}{27393}\\
\textbf{C91} & \tok{~such} after comma/preposition introducing examples in technical/legal text & Grammar & \sae{3}{21249}\\
\textbf{C92} & Period after U.S. state abbreviations and legal citations & --\\
\textbf{C93} & Newline after question mark or sentence end in \dataset{dm\_mathematics} and structured text & Math & --\\
\textbf{C94} & Sentence-ending \tok{.} after numbers in \dataset{dm\_mathematics} and technical text & Math & --\\
\textbf{C95} & Opening parenthesis \tok{(} after function names in \dataset{dm\_mathematics} algebra problems & Formatting & \sae{3}{834}\\
\textbf{C96} & Sentence-initial capitalized words after newlines in \dataset{dm\_mathematics} & Math & --\\
\textbf{C97} & \tok{if~(} conditional statement opening in \dataset{github-all} code & Code & \sae{2}{4280}\\
\textbf{C98} & Quantifiers and determiners after prepositions (\tok{~two}, \tok{~different}, \tok{~we}, \tok{~three}, \tok{~many}) & Grammar & --\\
\textbf{C99} & Token \tok{~on} from bigram "based on" in technical/scientific text & Bigrams & \sae{4}{22796}\\
\textbf{C100} & Double space separator between multiple choice options in \dataset{dm\_mathematics} & Math & \sae{4}{9348}\\
\textbf{C101} & \tok{~also} after auxiliary/linking verbs in technical/scientific text & Grammar & \sae{2}{18885}\\
\textbf{C102} & \tok{Category} token after newline in \dataset{wikipedia} article footers & & \sae{4}{28622}\\
\textbf{C103} & Newline followed by tab indentation in code/legal documents & Code & --\\
\textbf{C104} & Relative clause verbs after \tok{~which}/\tok{~that} in technical/formal text & Grammar & --\\
\textbf{C105} & Underscore \tok{\_} after identifiers in code/LaTeX (\dataset{github-all}, \dataset{arxiv}, \dataset{stackexchange}) & Math & --\\
\textbf{C106} & Code indentation (4 spaces) after newline in \dataset{github-all} and \dataset{stackexchange} & Code & \sae{2}{30019}\\
\textbf{C107} & Domain suffix \tok{org} after \tok{.} in URLs (\dataset{hackernews}, \dataset{stackexchange}) & Code & \sae{3}{9056}\\
\textbf{C108} & Exponentiation operator \tok{**} after variables in \dataset{dm\_mathematics} polynomial expressions & Math & \sae{3}{23541}\\
\textbf{C109} & Comma after clause-final \tok{~that} and \tok{~and} in technical/legal text & Grammar & --\\
\textbf{C110} & Whitespace before line breaks in \dataset{enron\_emails} and \dataset{stackexchange} & Formatting & --\\
\textbf{C111} & Delimiters and punctuation in code/math contexts (\tok{(}, \tok{~(}, \tok{;}, \tok{\}}, \tok{]}) & Code & --\\
\textbf{C112} & Conjunction \tok{~and} after list items in scientific/technical text & Grammar & \sae{2}{1172}\\
\textbf{C113} & Conjunctions after commas (\tok{~or}, \tok{~but}, \tok{~and}) in technical/legal text & Grammar & --\\
\textbf{C114} & \tok{:} after \tok{A} in \dataset{stackexchange} answer headers & & \sae{4}{3827}\\
\textbf{C115} & Newline after math problem answers in \dataset{dm\_mathematics} & Formatting & \sae{4}{11424}\\
\textbf{C116} & Relative pronouns (\tok{~that}, \tok{~which}, \tok{~who}) introducing clauses in technical/legal text & Grammar & --\\
\textbf{C117} & Newline characters in \dataset{hackernews}, \dataset{freelaw}, and \dataset{stackexchange} discussions & Formatting & --\\
\textbf{C118} & Repeated \tok{0} in comma-separated numeric arrays in \dataset{github-all} & Code & --\\
\textbf{C119} & Minus sign/subtraction operator in mathematical expressions & Math & --\\
\textbf{C120} & Indefinite article \tok{~a} after "as/is/have" across diverse texts & Grammar & --\\
\textbf{C121} & Article \tok{~an} & Grammar & --\\
\textbf{C122} & Newline after question in \dataset{dm\_mathematics} & Math & \sae{2}{26511}\\
\textbf{C123} & Apostrophe \tok{'} after pronouns/nouns for contractions and possessives & Grammar & \sae{3}{22772}\\
\textbf{C124} & Third-person plural pronouns (\tok{~their}, \tok{~they}, \tok{~themselves}) following entity references & Grammar & --\\
\textbf{C125} & Superlative/ordinal adjectives (\tok{~first}, \tok{~most}, \tok{~best}, \tok{~following}) after \tok{~the} in natural language & Grammar & --\\
\textbf{C126} & Colon after labels/headers in technical and structured text & Formatting & \sae{2}{25505}\\
\textbf{C127} & Decimal point \tok{.} after \tok{0} in numeric sequences in \dataset{arxiv} & Math & \sae{2}{32622}\\
\textbf{C128} & Preposition \tok{~in} after phrases requiring location/context specification & Grammar & \sae{2}{24529}\\
\textbf{C129} & Variable name introduction after \tok{Let} in \dataset{dm\_mathematics} & Math & \sae{3}{6264}\\
\textbf{C130} & Plural/past tense verbs (\tok{~are}, \tok{~were}, \tok{~have}, \tok{~was}) following plural nouns in scientific/technical text & Grammar & --\\
\textbf{C131} & Question marks ending sentences in \dataset{hackernews}, \dataset{stackexchange}, \dataset{cc}, and \dataset{enron\_emails} & Formatting & --\\
\textbf{C132} & Token \tok{~same} as part of the bigram "the same" across technical/legal texts & Bigrams & \sae{2}{22352}\\
\textbf{C133} & First-person pronoun \tok{~I} after sentence or clause boundaries in informal text & Grammar & --\\
\textbf{C134} & \tok{\&} after \tok{0} in LaTeX matrix notation in \dataset{arxiv} & Math & --\\
\textbf{C135} & Hyphens and dashes as formatting separators in technical/academic text & Formatting & --\\
\textbf{C136} & \tok{~which} after comma in relative clauses & Grammar & \sae{2}{5279}\\
\textbf{C137} & \tok{~we} after comma in academic/technical writing & Grammar & \sae{3}{6738}\\
\textbf{C138} & \tok{~of} after nouns in technical/scientific contexts & Grammar & \sae{3}{30829}\\
\textbf{C139} & Email address closing brackets in \dataset{enron\_emails} (\tok{.com>}, \tok{.edu>} followed by \tok{,} or \tok{;}) & & \sae{3}{17319}\\
\textbf{C140} & Auxiliary verbs after subjects (especially \tok{~is}) in technical/scientific text & Grammar & --\\
\textbf{C141} & Mathematical expression closing parentheses followed by operators in \dataset{dm\_mathematics} & Math & --\\
\textbf{C142} & Second-person pronouns (\tok{~you}, \tok{~your}) after conversational connectors in informal web text & Grammar & --\\
\textbf{C143} & Polynomial exponents (\tok{4}, \tok{5}, \tok{6}, \tok{7}) after \tok{**} in \dataset{dm\_mathematics} & Math & \sae{3}{31787}\\
\textbf{C144} & Hyphens and dashes in scientific/technical text & & --\\
\textbf{C145} & Fraction \tok{/} and minus sign \tok{-} in mathematical expressions & Math & --\\
\textbf{C146} & Token \tok{~lot} as part of bigram "a lot" & Bigrams & \sae{4}{9348}\\
\textbf{C147} & Tabs and code formatting tokens in \dataset{github-all} and technical contexts & Formatting & --\\
\textbf{C148} & The number \tok{2} as both an exponent and in other math expressions & Math & --\\
\textbf{C149} & HTML anchor tag \tok{a} after angle brackets in \dataset{github-all} & Code & \sae{4}{2942}\\
\textbf{C150} & LaTeX math commands after backslash in \dataset{arxiv} and \dataset{stackexchange} & Math & --\\
\textbf{C151} & HTML/XML opening angle bracket \tok{<} after newlines/whitespace in code & Code & \sae{4}{16761}\\
\textbf{C152} & \tok{~it} after subordinating conjunctions (\tok{~because}, \tok{~if}, \tok{~when}, \tok{~since}) & Grammar & \sae{4}{26404}\\
\textbf{C153} & Possessive \tok{'s} after "it" and proper nouns & Grammar & \sae{2}{15540}\\
\textbf{C154} & Comma after \tok{ECT} in \dataset{enron\_emails} email address lists & Formatting & \sae{4}{11343}\\
\textbf{C155} & Multiplication operator \tok{*} after numbers or variables in \dataset{dm\_mathematics} equations & Math & --\\
\textbf{C156} & Time units after "many" in \dataset{dm\_mathematics} and general text & Math & --\\
\textbf{C157} & Parenthetical expressions and clarifications in scientific/technical text & Grammar & \sae{2}{7291}\\
\textbf{C158} & \tok{**} exponentiation operator after variable names in mathematical expressions & Math & --\\
\textbf{C159} & Numeric tokens after arithmetic operators in \dataset{dm\_mathematics} and mixed contexts & Math & --\\
\textbf{C160} & Sentence-initial capitalized words after periods in academic/formal text & Grammar & --\\
\textbf{C161} & Em dash separator \tok{~--} in informal/technical text & & \sae{2}{20383}\\
\textbf{C162} & Newline after closing brace \tok{\}} in code (\dataset{github-all}, \dataset{stackexchange}) & Code & \sae{2}{26486}\\
\textbf{C163} & URL path separators and formatting slashes & Code & --\\
\textbf{C164} & \tok{pone} journal identifier in \dataset{pubmed\_central} citation references & & \sae{4}{6666}\\
\textbf{C165} & Document-initial tokens, mostly "Q", after \tok{<|endoftext|>} in \dataset{stackexchange} and emails & Grammar & \sae{2}{28741}\\
\textbf{C166} & Comma after numbers in lists and sequences & Formatting & --\\
\textbf{C167} & Underscore \tok{\_} for emphasis/italics formatting in \dataset{hackernews} and LaTeX subscripts in \dataset{arxiv} & Formatting & --\\
\textbf{C168} & Opening double quote after punctuation/text & Formatting & --\\
\textbf{C169} & \tok{What} followed by \tok{~is} in \dataset{dm\_mathematics} question patterns & Math & --\\
\textbf{C170} & HTML/XML closing angle bracket \tok{>} after tag names in \dataset{github-all} and \dataset{stackexchange} & Code & --\\
\textbf{C171} & First token after \tok{<|endoftext|>} in scientific/technical documents & Formatting & --\\
\textbf{C172} & Whitespace before line breaks in \dataset{enron\_emails} and other text corpora & Formatting & --\\
\textbf{C173} & LaTeX command opening brace \tok{\{} after macro names in \dataset{arxiv} and \dataset{stackexchange} & Math & --\\
\textbf{C174} & Opening quotation marks and special characters that behave in a similar way, after articles/prepositions & Formatting & --\\
\textbf{C175} & Sentence-initial tokens in math problems in \dataset{dm\_mathematics} & Math & --\\
\textbf{C176} & Question marks ending sentences in \dataset{hackernews}, \dataset{stackexchange}, \dataset{enron\_emails}, and \dataset{cc} & Formatting & \sae{2}{5385}\\
\textbf{C177} & Closing brace \tok{\}} in code blocks (\dataset{github-all}, \dataset{stackexchange}) & Code & \sae{4}{31580}\\
\textbf{C178} & \tok{enron} after \tok{@} in email addresses in \dataset{enron\_emails} && \sae{4}{10333} \\
\textbf{C179} & \tok{Category:} formatting pattern in \dataset{wikipedia\_en} && \sae{4}{22002}\\
\textbf{C180} & Sentence-initial \tok{~These} after period in scientific/technical text & Grammar & \sae{4}{28272}\\
\textbf{C181} & Conjunction \tok{~or} after alternative options in technical/scientific text & Grammar & --\\
\textbf{C182} & Closing quotation marks after quoted text && --\\
\textbf{C183} & Superlative adjectives (\tok{~most}, \tok{~best}, \tok{~last}, \tok{~fastest}) after \tok{~the}  & Grammar & --\\
\textbf{C184} & \tok{term} after hyphen in "long-term" and "short-term" compounds & & \sae{4}{14946}\\
\textbf{C185} & \tok{~=} often after closing parenthesis in function definition in \dataset{dm\_mathematics} equations & Math & \sae{3}{23065}\\
\textbf{C186} & \tok{~been} after auxiliary verbs (\tok{~have}/\tok{~has}/\tok{~had}) in perfect tense constructions & Grammar & \sae{2}{9904}\\
\textbf{C187} & \tok{~people} and related plural nouns after quantifiers/determiners in \dataset{hackernews} and \dataset{cc} & Grammar & --\\
\textbf{C188} & Opening parenthesis \tok{(} after conjunctions and punctuation in scientific/legal text & Formatting & --\\
\textbf{C189} & Newline after sentence-ending punctuation & Formatting & \sae{4}{18943}\\
\textbf{C190} & XML/HTML closing tag delimiter \tok{</} in code & Code & \sae{4}{2503}\\
\textbf{C191} & \tok{Pat} after "U.S." in patent citations in \dataset{uspto\_backgrounds} & Bigrams & \sae{4}{7189}\\
\textbf{C192} & Indefinite article \tok{~a} following prepositions/verbs in technical and narrative text & Grammar & \sae{4}{6774}\\
\textbf{C193} & Relative clause and complement clause introducers (\tok{~that}, \tok{~which}, \tok{~who}) in technical/legal text & Grammar & --\\
\textbf{C194} & Variable definitions after \tok{Let} in \dataset{dm\_mathematics} & Math & --\\
\textbf{C195} & \tok{://} after \tok{http}/\tok{https} in URLs across \dataset{hackernews}, \dataset{stackexchange}, \dataset{enron\_emails} & Code & \sae{4}{16482}\\
\textbf{C196} & \tok{Invention} after "Field of the" in \dataset{uspto\_backgrounds} patent documents & & \sae{4}{27614}\\
\textbf{C197} & HTML anchor tag \tok{href} attribute after \tok{<a} in \dataset{github-all} & Code & \sae{3}{2324}\\
\textbf{C198} & Sentence-ending \tok{.} after variable names in \dataset{dm\_mathematics} algebra problems & Math & --\\
\textbf{C199} & Hyphen after compound word prefixes (self-, high-, low-) & Formatting & --\\
\textbf{C200} & Consecutive newlines in \dataset{freelaw} legal documents & Formatting & \sae{4}{18262}\\
\textbf{C201} & Not clear & & --\\
\textbf{C202} & Mathematical expressions after closing parentheses in \dataset{dm\_mathematics} & Math & -- \\
\textbf{C203} & Numeric digits after division operators or with minus signs in \dataset{dm\_mathematics} & Math & \sae{3}{9056}\\
\textbf{C204} & Minus sign after newline in \dataset{dm\_mathematics} & Math & \sae{3}{6699}\\
\textbf{C205} & Sentence-initial \tok{~The} after newline in expository text & Grammar & \sae{4}{18242}\\
\textbf{C206} & 24-space indentation in \dataset{pubmed\_central} table formatting & Formatting & --\\
\textbf{C207} & Token \tok{~be} as part of "to be" infinitive construction across diverse datasets & Bigrams & \sae{2}{2297}\\
\textbf{C208} & 16-space indentation after newline in LaTeX \dataset{pubmed\_central} documents & Formatting & --\\
\textbf{C209} & Time duration units after numeric quantifiers & & --\\
\textbf{C210} & Legal citation suffix \tok{d} after F.2d/F.3d case reporters in \dataset{freelaw} & & --\\
\textbf{C211} & Contraction apostrophe \tok{'t} after negation words (don't, doesn't, didn't, isn't, can't) in \dataset{hackernews} and \dataset{stackexchange} & Grammar & \sae{2}{32558}\\
\textbf{C212} & Indentation and whitespace after newlines in \dataset{github-all}, \dataset{stackexchange}, and \dataset{freelaw} & Formatting & --\\
\textbf{C213} & Closing parentheses \tok{)} after numbers/expressions in \dataset{dm\_mathematics}, \dataset{arxiv}, \dataset{github-all}, and \dataset{pubmed\_central} & Math & --\\
\textbf{C214} & Hexadecimal byte values \tok{x} after \tok{0} in \dataset{github-all} & Code & --\\
\textbf{C215} & \tok{S} after \tok{U.} in legal citations (U.S.C., U.S. Pat.) & Bigrams & \sae{4}{6880}\\
\textbf{C216} & Indentation spaces after newlines in code (\dataset{github-all}, \dataset{stackexchange}) & Formatting & \sae{2}{30019} \\
\textbf{C217} & Content (\tok{a}, \tok{b}, \tok{td}, \tok{ul}) of HTML/XML closing tags (\tok{</a>}, \tok{</b>}, \tok{</td>}, \tok{</ul>}) & Code & --\\
\textbf{C218} & Pronoun + \tok{~was} after subject pronouns (\tok{~he}, \tok{~it}, \tok{~I}, \tok{~He}, \tok{~It}) & Grammar & \sae{2}{26870}\\
\textbf{C219} & Variable name \tok{z} in \dataset{dm\_mathematics} algebra expressions after \tok{*} & Math & \sae{2}{29842}\\
\textbf{C220} & Hyphen \tok{-} in compound words and formatting across technical/scientific text & Formatting & --\\
\textbf{C221} & \tok{~the} after "is" in \dataset{dm\_mathematics} math questions & Bigrams & \sae{3}{28171}\\
\textbf{C222} & \tok{~the} after prepositions in technical/scientific text & Grammar & \sae{2}{13230}\\
\textbf{C223} & Ordinal position words (\tok{~second}, \tok{~third}, \tok{~remainder}, \tok{~smallest}, \tok{~biggest}) after \tok{~the} in \dataset{dm\_mathematics} & Grammar & --\\
\textbf{C224} & Denominator digits after \tok{/} in fractions in \dataset{dm\_mathematics} & Math & \sae{3}{19310}\\
\textbf{C225} & Document-initial capitalized words (\tok{~The}, \tok{~This}) after \tok{<|endoftext|>} in scientific/technical texts & Grammar & --\\
\textbf{C226} & \tok{ENRON} after \tok{@} in email addresses from \dataset{enron\_emails} & & \sae{4}{20721}\\
\textbf{C227} & Female pronoun references (\tok{~her}/\tok{~she}) in narrative contexts & Grammar & --\\
\textbf{C228} & Contraction \tok{'t} (don't, can't, isn't) in \dataset{cc} and \dataset{hackernews} & Grammar & \sae{2}{32558}\\
\textbf{C229} & Newline after section dividers (\tok{~~~}, \tok{------}, \tok{======}) in \dataset{hackernews} and technical documents & Formatting & --\\
\textbf{C230} & \tok{Category} after newline in \dataset{wikipedia} category lists & & \sae{4}{28622}\\
\textbf{C231} & \tok{S} after \tok{U.} in \tok{U.S.} abbreviation & Bigrams & \sae{2}{14080}\\
\textbf{C232} & Article \tok{~an} after prepositions and verbs in technical/legal text & Grammar & --\\
\textbf{C233} & Contraction \tok{'s} after "it" (it's) & Grammar & --\\
\textbf{C234} & \tok{.html} file extension in URLs/paths in \dataset{github-all} & Code & \sae{4}{11710}\\
\textbf{C235} & Negation phrases after pronouns ("don't", "haven't") in informal text & Grammar & \sae{4}{18262}\\
\textbf{C236} & Token \tok{~as} as part of "such as" phrase completion in technical/scientific text & Grammar & --\\
\textbf{C237} & Token \tok{g} completing "e.g." & Bigrams & \sae{4}{3985}\\
\textbf{C238} & Contractions after \tok{I} (\tok{'m}, \tok{'ve}, \tok{'s}) in informal text & Grammar & --\\
\textbf{C239} & Token \tok{~in} in various contexts & Grammar & --\\
\textbf{C240} & Comparative phrases with \tok{~than} following \tok{~more}, \tok{~less}, \tok{~better}, \tok{~higher}, \tok{~greater} & Grammar & \sae{2}{6584}\\
\textbf{C241} & Opening parenthesis \tok{(} for multiple choice options in \dataset{dm\_mathematics} & Math & \sae{3}{31725}\\
\textbf{C242} & HTML/XML opening angle bracket \tok{<} after newlines/whitespace in \dataset{github-all} and \dataset{stackexchange} & Code & --\\
\textbf{C243} & Question marks ending math problems and queries in \dataset{dm\_mathematics}, \dataset{hackernews}, and \dataset{stackexchange} & Formatting & \sae{2}{5385}\\
\textbf{C244} & HTML closing tag \tok{>} after \tok{li} and other elements in \dataset{github-all} & Code & \sae{2}{28829}\\
\textbf{C245} & \tok{16} after \tok{DSP} in DSP16 function calls in \dataset{github-all} & & --\\
\textbf{C246} & First token after \tok{<|endoftext|>} token & Formatting & --\\
\textbf{C247} & Common verbs after plural subjects (\tok{~have}, \tok{~are}, \tok{~know}, \tok{~to}) in explanatory text & Grammar & --\\
\textbf{C248} & \tok{~present} after \tok{~the}/\tok{~The} in \dataset{uspto\_backgrounds} patent descriptions & Grammar & --\\
\textbf{C249} & LaTeX command opening brace \tok{\{} after math commands in \dataset{arxiv} & Math & \sae{2}{463}\\
\textbf{C250} & \tok{True}/\tok{False} boolean answers after newlines in \dataset{dm\_mathematics} math problems & Math & \sae{4}{16089}\\
\textbf{C251} & HTML/XML attribute value delimiter \tok{="} after attribute names in \dataset{github-all} and \dataset{stackexchange} & Code & \sae{2}{15267}\\
\textbf{C252} & Digits after minus sign in \dataset{dm\_mathematics} equations & Math & --\\
\textbf{C253} & Ordinal suffix \tok{'t} as in "'th" after variable names in \dataset{dm\_mathematics} sequence term problems & Grammar & \sae{4}{2426}\\
\textbf{C254} & Document-initial tokens after \tok{<|endoftext|>} in scientific/professional text & Formatting & --\\
\textbf{C255} & Preposition \tok{~for} after purpose/use descriptions in technical and scientific text & Grammar & --\\
\textbf{C256} & \tok{)\{\#} figure/table reference syntax in \dataset{pubmed\_central} markdown & Formatting & \sae{2}{11800}\\
\textbf{C257} & Semicolon \tok{;} after list items in technical/scientific text & Formatting & \sae{2}{31196}\\
\textbf{C258} & \tok{~of} after math terms like "term", "digit", "multiple", "factors", "derivative" in \dataset{dm\_mathematics} & Grammar & \sae{4}{7086}\\
\textbf{C259} & Sentence-initial capitalized words after newlines & Grammar & --\\
\textbf{C260} & Top-level domain extensions (\tok{.com}, \tok{.org}, \tok{.edu}) after domain names in URLs & Code & --\\
\textbf{C261} & Token \tok{~is} as part of "which is" pattern in \dataset{dm\_mathematics} comparison questions & Grammar & --\\
\textbf{C262} & Token \tok{~think} as part of "I think" in informal text (\dataset{hackernews}, \dataset{cc}, \dataset{enron\_emails}) & Bigrams & \sae{4}{8133}\\
\textbf{C263} & Token \tok{e} as part of "i.e." & Bigrams & \sae{2}{16862}\\
\textbf{C264} & Double space separator between multiple choice options in \dataset{dm\_mathematics} & Formatting & --\\
\textbf{C265} & \tok{A} after newline in \dataset{stackexchange} Q\&A answer markers && \sae{2}{30166}\\
\textbf{C266} & XML/HTML closing tag delimiter \tok{</} in code and markup & Code & \sae{2}{9024}\\
\textbf{C267} & Opening brace \tok{\{} after indentation in code blocks (\dataset{github-all}, \dataset{stackexchange}) & Code & \sae{4}{20517}\\
\textbf{C268} & Prepositions after verbs/phrases (\tok{~from}, \tok{~as}, \tok{~into}, \tok{~by}, \tok{~to}) & Grammar & --\\
\textbf{C269} & Variable names \tok{j} and \tok{o} in \dataset{dm\_mathematics} & Math & --\\
\textbf{C270} & Token \tok{~the} as part of "is the" bigram in \dataset{dm\_mathematics} comparison questions & Bigrams & \sae{4}{13470}\\
\textbf{C271} & \tok{~are} after "minutes" in time calculation questions in \dataset{dm\_mathematics} & Grammar & \sae{4}{6666}\\
\textbf{C272} & Token ` example as part of "for example" in technical/patent text & Bigrams & \sae{4}{18943}\\
\textbf{C273} & Sentence-ending punctuation (\tok{!}, \tok{?}, \tok{...}) in informal text & Formatting & --\\
\textbf{C274} & \tok{~for} after variable name in \dataset{dm\_mathematics} solve equations & Grammar & \sae{4}{18262}\\
\textbf{C275} & Sentence-initial \tok{I} after newline in \dataset{hackernews} and \dataset{stackexchange} discussions & Grammar & \sae{4}{4232}\\
\textbf{C276} & Email quote continuation marker \tok{>} after newline in \dataset{enron\_emails} and \dataset{hackernews} & Formatting & \sae{2}{18642}\\
\textbf{C277} & \tok{------} comment separator in \dataset{hackernews} discussions & Formatting & \sae{4}{1060}\\
\textbf{C278} & \tok{com} after \tok{.} in URLs across \dataset{hackernews}, \dataset{enron\_emails}, and \dataset{stackexchange} & Code & --\\
\textbf{C279} & Multiplication operator \tok{*} after numbers in \dataset{dm\_mathematics} equations & Math & --\\
\textbf{C280} & Sentence-initial conjunctive adverbs (\tok{~However}, \tok{~Thus}, \tok{~Therefore}, \tok{~Furthermore}, \tok{~Moreover}) in academic/legal texts & Grammar & --\\
\textbf{C281} & Token \tok{~this} as part of "In this" in scientific/technical discourse & Grammar & \sae{3}{27379}\\
\textbf{C282} & \tok{~with} after \tok{~patients} in medical/scientific literature & Grammar & --\\
\textbf{C283} & Relative pronouns (\tok{~which}, \tok{~who}, \tok{~whose}) in technical/formal writing & Grammar & \sae{2}{5279}\\
\textbf{C284} & Exponent \tok{3}-\tok{7} after \tok{**} in \dataset{dm\_mathematics} polynomial expressions & Math & \sae{3}{31787}\\
\textbf{C285} & \tok{a} after \tok{(} for multiple choice option labels in \dataset{dm\_mathematics} & & \sae{4}{21437}\\
\textbf{C286} & \tok{://} after \tok{http}/\tok{https} in URLs across \dataset{hackernews}, \dataset{github-all}, \dataset{enron\_emails} & Code & --\\
\textbf{C287} & Slash \tok{/} and hyphen \tok{-} as formatting delimiters in URLs, paths, fractions, and compound terms & Formatting & --\\
\textbf{C288} & Contrastive conjunctions (\tok{~but}, \tok{~while}) and similar words after commas & Grammar & \sae{2}{27816}\\
\textbf{C289} & \tok{~case} after "the/this/in" in legal and technical contexts & Grammar & --\\
\textbf{C290} & Newline after \tok{True}/\tok{False} answers in \dataset{dm\_mathematics} & Formatting & \sae{4}{11424}\\
\textbf{C291} & Token \tok{~terms} in bigram "the terms" in \dataset{dm\_mathematics} "Collect the terms" expressions & Bigrams & \sae{2}{25825}\\
\textbf{C292} & \tok{~prime} after "a" or "the" in \dataset{dm\_mathematics} primality questions & & --\\
\textbf{C293} & \tok{type} after \tok{-} in "ref-type=" attribute patterns in \dataset{pubmed\_central} & Code & \sae{4}{9348}\\
\textbf{C294} & Months of the year following prepositions like \tok{~in} and \tok{~on} in date expressions & & --\\
\textbf{C295} & \tok{<|endoftext|>} token after sentence-ending punctuation in scientific/technical documents & Formatting & --\\
\textbf{C296} & URL domain suffix \tok{com} after period in \dataset{enron\_emails} and \dataset{hackernews} & Code & \sae{2}{23912}\\
\textbf{C297} & Complementizer \tok{that} after verbs/phrases expressing belief, knowledge, or observation & Grammar & --\\
\textbf{C298} & Paragraph breaks after sentence-ending punctuation in \dataset{freelaw} legal documents & Formatting & \sae{2}{17680}\\
\textbf{C299} & Conjunction \tok{~and} in various contexts & Grammar & --\\
\textbf{C300} & Relative pronouns \tok{~that}/\tok{~which} introducing clauses & Grammar & --\\
\textbf{C301} & Hyphen \tok{-} in ISO date lists (YYYY-MM-DD format) in \dataset{stackexchange} & Code & \sae{4}{18262}\\
\textbf{C302} & LaTeX subscript closing \tok{\}\_} and HTML \tok{HREF} attributes in \dataset{arxiv} and \dataset{github-all} & Math & --\\
\textbf{C303} & Newline after HTML closing tags \tok{>} or \tok{">} in \dataset{stackexchange} & Formatting & \sae{4}{13671}\\
\textbf{C304} & Asterisk \tok{*} in code documentation comments (JSDoc/Javadoc style) in \dataset{github-all} & Math & --\\
\textbf{C305} & Sentence-initial capitalized words after periods (\tok{~In}, \tok{~At}, \tok{~Since}) in scientific/technical text & Grammar & \sae{2}{9902}\\
\textbf{C306} & \tok{$\sim\sim\sim$} separator between comments in scraped \dataset{hackernews} comments & Formatting & \sae{2}{17795}\\
\textbf{C307} & Opening brace \tok{\{} after LaTeX commands and code blocks in \dataset{github-all}, \dataset{stackexchange}, and \dataset{arxiv} & Math & --\\
\textbf{C308} & Period after "U.S" in legal citations (\dataset{freelaw}, \dataset{uspto\_backgrounds}) & Formatting & \sae{2}{11214}\\
\textbf{C309} & First-person pronoun \tok{~I} after subordinate clauses in informal text & Grammar & \sae{4}{4232}\\
\textbf{C310} & Hyphen after "anti" prefix in biomedical/scientific text & Formatting & \sae{4}{23580}\\
\textbf{C311} & \tok{~Collect}/\tok{~Expand} math problem keywords in \dataset{dm\_mathematics} & Math & \sae{4}{3475}\\
\textbf{C312} & \tok{D} prefix in DSP16 function calls in \dataset{github-all} & & \sae{3}{1074}\\
\textbf{C313} & \tok{I} after newline starting Stack Exchange questions & Grammar & \sae{4}{4232}\\
\textbf{C314} & \tok{~closest} after "the" in \dataset{dm\_mathematics} proximity questions & Math & \sae{4}{21638}\\
\textbf{C315} & Multiple choice option \tok{b} after \tok{(} in \dataset{dm\_mathematics} & Math & \sae{2}{23839}\\
\textbf{C316} & Opening parenthesis \tok{(} before multiple choice options in \dataset{dm\_mathematics} & Formatting & \sae{3}{31725}\\
\textbf{C317} & Newline after code statement endings in \dataset{github-all} and \dataset{stackexchange} & Formatting & --\\
\textbf{C318} & Token \tok{~as} as part of "as well as" & Bigrams & --\\
\textbf{C319} & Variable names (single letters like \tok{j}, \tok{q}, \tok{z}, \tok{h}, \tok{i}) in \dataset{dm\_mathematics} equations and scientific text & Math & --\\
\textbf{C320} & Hyphen/dash after newlines and in compound terms across \dataset{dm\_mathematics}, \dataset{pubmed\_central}, \dataset{arxiv} & Formatting & --\\
\textbf{C321} & Underscore \tok{\_} formatting character in code/LaTeX contexts & Math & --\\
\textbf{C322} & URL/path separator \tok{/} after domain names and path segments & Code & --\\
\textbf{C323} & Numeric \tok{~0} after comma in arrays/lists in \dataset{github-all} & Code & --\\
\textbf{C324} & \tok{return} keyword in code after indentation in \dataset{github-all} and \dataset{stackexchange} & Code & \sae{3}{13106}\\
\textbf{C325} & Double newline continuation in \dataset{freelaw} legal documents & Formatting & \sae{4}{18262}\\
\textbf{C326} & Ordinal suffix \tok{h} after apostrophe in "n'th term" patterns in \dataset{dm\_mathematics} & Grammar & \sae{2}{24180}\\
\textbf{C327} & Ordinal position + \tok{~biggest}/\tok{~smallest} in \dataset{dm\_mathematics} value comparisons & Math & \sae{4}{19908}\\
\textbf{C328} & Function word \tok{~the} after prepositions/conjunctions & Grammar & \sae{2}{13230}\\
\textbf{C329} & \tok{ECT} token in Enron email address patterns (e.g., \tok{/HOU/ECT@ECT}) & & \sae{4}{22254}\\
\textbf{C330} & HTML tag attribute \tok{class} after element names like \tok{div}, \tok{span}, \tok{a}, \tok{li} in \dataset{github-all} and \dataset{stackexchange} & Code & --\\
\textbf{C331} & Legal citation suffix \tok{d} after F.2d/F.3d case reporters in \dataset{freelaw} & & --\\
\textbf{C332} & Double newline after \tok{<|endoftext|>Q:} in \dataset{stackexchange} & Formatting & \sae{2}{3257}\\
\textbf{C333} & \tok{Let} after newline in \dataset{dm\_mathematics} math problems & Math & \sae{3}{1074}\\
\textbf{C334} & Markdown emphasis/italic markers (\tok{\_}) in \dataset{hackernews} discussions & Formatting & --\\
\textbf{C335} & LaTeX right and left commands in mathematical expressions & Math & --\\
\textbf{C336} & LaTeX math notation closing braces and delimiters in \dataset{arxiv}/\dataset{stackexchange} & Math & --\\
\textbf{C337} & Opening brace \tok{\{} after function/control structure declarations in code & Code & \sae{3}{26738}\\
\textbf{C338} & Newline after closing brace \tok{\}} in code (\dataset{github-all}, \dataset{stackexchange}) & Formatting & --\\
\textbf{C339} & Opening parenthesis after \tok{if}/\tok{while} in code & Code & \sae{2}{4280}\\
\textbf{C340} & Code indentation with tab and newline characters in \dataset{github-all} & Code & --\\
\textbf{C341} & Email quote continuation marker \tok{>} in \dataset{enron\_emails} & Code & --\\
\textbf{C342} & \tok{~was}/\tok{~is} after pronouns (\tok{~it}, \tok{~there}, \tok{~he}) in expository text & Grammar & \sae{2}{26870}\\
\textbf{C343} & Single-digit numbers after minus signs in \dataset{dm\_mathematics} & Math & --\\
\textbf{C344} & \tok{Collect} followed by \tok{~the} in \dataset{dm\_mathematics} polynomial term collection problems & Math & \sae{3}{512}\\
\textbf{C345} & Numeric references, citations and section numbers in technical/legal documents && --\\
\textbf{C346} & Token \tok{~with} as part of "associated with" bigram in scientific/technical text & Bigrams & \sae{2}{24995}\\
\textbf{C347} & \tok{~we} after comma in scientific/academic text (\dataset{nih\_exporter}, \dataset{arxiv}, \tok{pubmed}) & Grammar & \sae{3}{6738}\\
\textbf{C348} & \tok{$\sim\sim\sim$} separator between comments in scraped \dataset{hackernews} forum comments & Formatting & \sae{2}{17795}\\
\textbf{C349} & Opening parenthesis \tok{(} after function names in \dataset{dm\_mathematics} & Code & \sae{3}{834}\\
\textbf{C350} & Closing quotation marks after quoted text in \dataset{uspto\_backgrounds}, \dataset{cc}, \dataset{arxiv}, and \dataset{freelaw} & Formatting & --\\
\textbf{C351} & Sentence-initial capital letter \tok{A} after period in scientific/technical text & Grammar & --\\
\textbf{C352} & \tok{Enron} after \tok{/} in email addresses from \dataset{enron\_emails} & & \sae{4}{10333}\\
\textbf{C353} & Exponent \tok{2} after \tok{**} in mathematical expressions & Math & \sae{2}{14139}\\
\textbf{C354} & Sentence-initial \tok{I} after newlines in informal text (emails, forums) & Grammar & \sae{4}{4232}\\
\textbf{C355} & Token \tok{~with} as part of "associated with" bigram in biomedical/scientific texts & Bigrams & --\\
\textbf{C356} & Equation assignment \tok{~=} after \tok{0} in \dataset{dm\_mathematics} & Math & \sae{3}{23065}\\
\textbf{C357} & Token \tok{~these} as part of "of these" bigram in scientific/legal text & Bigrams & --\\
\textbf{C358} & \tok{="} after \tok{ref-type} in \dataset{pubmed\_central} XML cross-reference formatting & Code & \sae{4}{21337}\\
\textbf{C359} & Token \tok{~following} as part of "the following" bigram in \dataset{arxiv} and other technical writing & Bigrams & \sae{2}{7454}\\
\textbf{C360} & Variable names as function arguments and after \tok{Let} in \dataset{dm\_mathematics} & Math & \sae{2}{30901}\\
\textbf{C361} & Token \tok{~base} as part of "In base" questions in \dataset{dm\_mathematics} & Bigrams & \sae{3}{22772}\\
\textbf{C362} & \tok{~is} after \tok{What}/\tok{Which} in \dataset{dm\_mathematics} questions & & \sae{4}{15064}\\
\textbf{C363} & Sentence-initial \tok{~They}/\tok{~Their}/\tok{~There}/\tok{~Those} after period & Grammar & \sae{4}{19240}\\
\textbf{C364} & Token \tok{~the} as part of "be the" bigram in \dataset{dm\_mathematics} derivative definitions & Bigrams & \sae{4}{12133}\\
\textbf{C365} & Sentence-initial \tok{~A} after period in scientific/technical text & Grammar & --\\
\textbf{C366} & Contrastive conjunctions after commas (\tok{~but}, \tok{~while}, \tok{~because}) & Grammar & --\\
\textbf{C367} & Conjunction \tok{~and} in scientific/technical text & Grammar & \sae{2}{1172}\\
\textbf{C368} & LaTeX \tok{\{} delimiter after label, begin, end commands in \dataset{arxiv} & Math & --\\
\textbf{C369} & Indefinite article \tok{~a} in math divisibility questions in \dataset{dm\_mathematics} & Math & --\\
\textbf{C370} & \tok{~things} and \tok{~people} after various context words in \dataset{hackernews}, \dataset{cc}, and \dataset{stackexchange} informal text & Grammar & \sae{3}{2279}\\
\textbf{C371} & Token \tok{s} as part of Possessive "'s" in \dataset{cc} and \dataset{freelaw} & Grammar & --\\
\textbf{C372} & Closing double quote \tok{"} after quoted text in mixed sources & Formatting & --\\
\textbf{C373} & \tok{States} after \tok{United} in country name bigram & Bigrams & \sae{4}{25349}\\
\textbf{C374} & Dollar sign \tok{\$} as LaTeX math delimiter in \dataset{arxiv} and \dataset{stackexchange} & Math & --\\
\textbf{C375} & Auxiliary verbs after pronouns/subjects (\tok{~is}, \tok{~can}, \tok{~was}, \tok{~has}) in technical/formal text & Grammar & --\\
\textbf{C376} & Token \tok{~do} as part of "to do" bigram after infinitive marker & Bigrams & \sae{3}{24016}\\
\textbf{C377} & HTML tag names after \tok{<} in code/markup & Code & \sae{4}{16622}\\
\textbf{C378} & Citation reference \tok{B} after \tok{[@} in \dataset{pubmed\_central} & & \sae{2}{6093}\\
\textbf{C379} & Variable name \tok{v} after \tok{Let} in \dataset{dm\_mathematics} & Math & \sae{3}{1074}\\
\textbf{C380} & \tok{Solve} followed by minus sign \tok{~-} in \dataset{dm\_mathematics} equations & Math & \sae{3}{9056}\\
\textbf{C381} & Function word \tok{~the} after math problem instructions in \dataset{dm\_mathematics} && \sae{2}{13230}\\
\textbf{C382} & Closing parenthesis \tok{)} after function arguments, denominators or other contexts in \dataset{dm\_mathematics} & Math & --\\
\textbf{C383} & Escaped backslash before LaTeX \tok{\\usepackage} commands in \dataset{pubmed\_central} document preambles & Math & \sae{3}{5220}\\
\textbf{C384} & Sentence-initial capitalized words after newlines & Grammar & --\\
\textbf{C385} & Quotation marks and delimiters in code/markup formatting & Formatting & --\\
\textbf{C386} & Sentence-initial keywords (\tok{~Let}, \tok{~Suppose}, \tok{~Is}) after periods in \dataset{dm\_mathematics} & Math & --\\
\textbf{C387} & Newline after sentence/statement endings in \dataset{dm\_mathematics} and technical documents & Formatting & --\\
\textbf{C388} & \tok{)\{} after reference IDs in \dataset{pubmed\_central} markdown cross-references & Formatting & \sae{4}{9348}\\
\textbf{C389} & Time format \tok{00} after colons/periods in timestamps (\dataset{enron\_emails}, \dataset{github-all}) && \sae{2}{23953}\\
\textbf{C390} & Sentence-initial capitalized words after periods in technical/scientific text & &--\\
\textbf{C391} & Variable names and mathematical symbols after operators/delimiters in \dataset{dm\_mathematics}, \dataset{arxiv}, and \dataset{freelaw} & Math & --\\
\textbf{C392} & Code indentation with tabs (\tok{\textbackslash n \textbackslash t \textbackslash t} and \tok{\textbackslash n \textbackslash t \textbackslash t \textbackslash t}) in \dataset{github-all} & Code & \sae{4}{27050}\\
\textbf{C393} & Newline after sentence-ending punctuation in \dataset{enron\_emails}, \dataset{stackexchange}, \dataset{hackernews}, \dataset{wikipedia\_en}, \dataset{freelaw} & Formatting & --\\
\textbf{C394} & Options \tok{c}, \tok{d} in \dataset{dm\_mathematics} multiple choice questions & Math & --\\
\textbf{C395} & HTML list item closing and opening tags (\tok{><} after \tok{li}, \tok{span}) in \dataset{github-all} & Code & --\\
\textbf{C396} & Single lowercase letters, often in words where every letter is separated by a space && --\\
\textbf{C397} & \tok{Suppose} followed by minus sign \tok{~-} in \dataset{dm\_mathematics} equations & Math & \sae{3}{9056}\\
\textbf{C398} & Token \tok{.,} as comma following "i.e." or "e.g." in scientific/legal text & Formatting & \sae{2}{3597}\\
\textbf{C399} & Token \tok{usepackage} as part of LaTeX "\textbackslash{}usepackage" in \dataset{pubmed\_central} documents & Math & \sae{3}{1074}\\
\textbf{C400} & Modal verb \tok{~can} (also \tok{~could}, \tok{~may}, \tok{~would}) after subject pronouns in informal/technical text & Grammar & \sae{4}{31978}\\
\textbf{C401} & \tok{/} after \tok{Corp} or \tok{com} in URLs and email paths in \dataset{enron\_emails} and \dataset{github-all} & Code & \sae{3}{29181}\\
\textbf{C402} & \tok{~example} after "for" or "For" in technical/legal documents & Bigrams & --\\
\textbf{C403} & Nested parenthetical expressions \tok{~((-} in \dataset{dm\_mathematics} & Formatting & \sae{3}{18590}\\
\textbf{C404} & JavaScript/code comment markers \tok{//} in \dataset{github-all} and \dataset{stackexchange} & Code & \sae{2}{22162}\\
\textbf{C405} & Numeric tokens after mathematical operators and delimiters in \dataset{dm\_mathematics} & Math & --\\
\textbf{C406} & Newline after horizontal rule separators in \dataset{hackernews} & Formatting & \sae{4}{2328}\\
\textbf{C407} & \tok{(} after function/method names and \tok{-} in expressions across \dataset{dm\_mathematics}, \dataset{stackexchange}, \dataset{github-all}, and \dataset{arxiv} & Code & --\\
\textbf{C408} & Token \tok{~the} as part of "to the" bigram in \dataset{dm\_mathematics} rounding problems & Code & \sae{3}{18437}\\
\textbf{C409} & Contraction \tok{'re} after pronouns in \dataset{hackernews} and \dataset{stackexchange} discussions & Grammar & \sae{4}{16875}\\
\textbf{C410} & Modal verb completions (\tok{~be}, \tok{~been}, \tok{~not}, \tok{~go}, \tok{~also}) following auxiliaries & Grammar & --\\
\textbf{C411} & \tok{~will} after research proposal descriptions in \dataset{nih\_exporter} & Grammar & \sae{2}{12242}\\
\textbf{C412} & \tok{~had} following pronouns/subjects in past perfect constructions & Grammar & \sae{4}{32084}\\
\textbf{C413} & Ordinal words (\tok{~first}, \tok{~second}, \tok{~third}) after \tok{~the} in \dataset{dm\_mathematics} derivative problems && \sae{3}{5943}\\
\textbf{C414} & \tok{Solve} keyword after newline in \dataset{dm\_mathematics} algebra problems & Math & \sae{4}{17290}\\
\textbf{C415} & Comparison/relation markers (\tok{~with}, \tok{~as}, \tok{~in}, \tok{~=}) in technical and mathematical text && --\\
\textbf{C416} & Six-space indentation after newline in \dataset{github-all} HTML/XML code & Formatting & \sae{2}{9186}\\
\textbf{C417} & \tok{HOU/} path separator in Enron email addresses && \sae{3}{1074}\\
\textbf{C418} & Citation reference letter \tok{B} after \tok{[@} in \dataset{pubmed\_central} && \sae{2}{6093}\\
\textbf{C419} & \tok{/} after "coq" in Coq library file paths in \dataset{github-all} & Formatting & \sae{3}{29181}\\
\textbf{C420} & \tok{~the} after newline in \dataset{hackernews}, \dataset{freelaw}, and \dataset{enron\_emails} & Grammar & \sae{4}{18242}\\
\textbf{C421} & Numeric literal \tok{~0} after equals/operators in \dataset{dm\_mathematics} & Math & \sae{3}{8250}\\
\textbf{C422} & Indentation whitespace after newlines in code & Formatting & \sae{2}{30019}\\
\textbf{C423} & Newline after \tok{True}/\tok{False} answers and section headers in \dataset{dm\_mathematics}, \dataset{pubmed\_central}, and \dataset{arxiv} & Formatting & --\\
\textbf{C424} & Conjunctive adverbs (\tok{~thus}, \tok{~therefore}, \tok{~hence}) after coordinating conjunctions in technical/scientific text & Grammar & \sae{2}{8884}\\
\textbf{C425} & Token \tok{~be} as part of "Let ... be" phrase in \dataset{dm\_mathematics} derivative problems & Grammar & \sae{4}{24505}\\
\textbf{C426} & HTML tag contents (\tok{li}, \tok{td}, \tok{b}), a mix of opening and closing but mostly the former & Code & --\\
\textbf{C427} & Minus signs (sometimes combined with other characters) and left brackets in function arguments in \dataset{dm\_mathematics} & Math & --\\
\textbf{C428} & Variable name \tok{~h} after \tok{Let} in \dataset{dm\_mathematics} & Math & \sae{3}{6264}\\
\textbf{C429} & Newline after multiple consecutive newlines in \dataset{enron\_emails} and \dataset{freelaw} & Formatting & \sae{2}{3257}\\
\textbf{C430} & Domain name TLDs and email/URL components (\tok{.com}, \tok{.org}, \tok{.edu}) & Code & --\\
\textbf{C431} & Second newline in double newlines in structured text & Formatting & --\\
\textbf{C432} & Token \tok{~order} as part of "in order" bigram & Bigrams & \sae{4}{18262}\\
\textbf{C433} & Token \tok{~have} as part of "I have" bigram in informal/technical text & Bigrams & \sae{4}{32084}\\
\textbf{C434} & Token \tok{~no} as part of negation phrases after "is/was/were" & Grammar & \sae{3}{624}\\
\textbf{C435} & \tok{OMNI} XML tag names in \dataset{enron\_emails} & Code & --\\
\textbf{C436} & \tok{)?} question delimiter in \dataset{dm\_mathematics} math problems & Formatting & \sae{4}{6666}\\
\textbf{C437} & Closing parenthesis \tok{)} after \tok{a} in multiple choice options in \dataset{dm\_mathematics} & & \sae{3}{24655}\\
\textbf{C438} & \tok{Simplify} command after newline in \dataset{dm\_mathematics} algebra problems & Math & \sae{2}{31862}\\
\textbf{C439} & \tok{Calculate} after newline in \dataset{dm\_mathematics} math problems & Math & \sae{3}{23953}\\
\textbf{C440} & Uppercase letter sequences in base64/encoded data & Code & --\\
\textbf{C441} & Token \tok{|} as a separator in wiki tables and LaTeX formatting & Formatting & \sae{4}{13088}\\
\textbf{C442} & \tok{ref-type} attribute hyphen in \dataset{pubmed\_central} XML references & Formatting & \sae{4}{22796}\\
\textbf{C443} & Newline after username in \dataset{hackernews} comment replies & Formatting & --\\
\textbf{C444} & \tok{~found} after passive/discovery verbs in scientific and factual texts & Grammar & --\\
\textbf{C445} & Prepositions \tok{~in}/\tok{~with} in technical/scientific text & Grammar & --\\
\textbf{C446} & Token \tok{:} as part of "Sent:" field delimiter in email headers and metadata formatting & Formatting & --\\
\textbf{C447} & \tok{States} after \tok{United} in legal/encyclopedic contexts & Bigrams & \sae{2}{633}\\
\textbf{C448} & Citation reference markers in \dataset{pubmed\_central} and \dataset{arxiv} & Formatting & \sae{3}{10916}\\
\textbf{C449} & Newline-following tokens (indentation/formatting) in code and structured text & Formatting & --\\
\textbf{C450} & Token \tok{~the} as part of "that the" bigram in legal/technical documents & Bigrams & \sae{4}{18242}\\
\textbf{C451} & \tok{~what} after comma in \dataset{dm\_mathematics} base arithmetic questions & & \sae{2}{5836}\\
\textbf{C452} & Punctuation (\tok{,} or \tok{.}) after \tok{<|endoftext|>} token & Formatting & --\\
\textbf{C453} & Modal verb followed by \tok{~be} (passive/future constructions) & Grammar & --\\
\textbf{C454} & Token \tok{~to} as part of bigrams "is to" and "was to" & Bigrams & --\\
\textbf{C455} & Capital letters as part of capitalized words/names following punctuation and conjunctions & Grammar & --\\
\textbf{C456} & Token \tok{~term} as part of "th term" in math sequences, \tok{ories} in file paths (theories/setoid\textbackslash{}\_ring) & Math & --\\
\textbf{C457} & Token \tok{you} as part of various forms of "Thank you" bigrams & Bigrams & \sae{2}{19835}\\
\textbf{C458} & \tok{Suppose} keyword starting math problems in \dataset{dm\_mathematics} & Math & \sae{4}{17290}\\
\textbf{C459} & Exponent \tok{3} or \tok{4} after \tok{**} in \dataset{dm\_mathematics} polynomial expressions & Math & \sae{3}{31787}\\
\textbf{C460} & Contraction apostrophe after \tok{don}, \tok{doesn}, \tok{didn}, \tok{isn}, \tok{can} in \dataset{hackernews}, \dataset{cc}, \dataset{stackexchange} & Grammar & \sae{2}{32558}\\
\textbf{C461} & \tok{00} following \tok{-} in \dataset{pubmed\_central} citation reference IDs (e.g., \tok{viruses-10-00682}) && \sae{4}{6666}\\
\textbf{C462} & Minus sign/subtraction operator after equals sign in \dataset{dm\_mathematics} & Math & \sae{3}{14513}\\
\textbf{C463} & Percentage symbol \tok{\%} after numeric values in scientific/medical texts & Math & \sae{3}{23877}\\
\textbf{C464} & Exclamation marks and ellipses in informal/emphatic text && --\\
\textbf{C465} & Newline after line-wrapped legal text in \dataset{freelaw} & Formatting & \sae{4}{18262}\\
\textbf{C466} & Token \tok{.} after parts of abbreviations or as decimal points & Formatting & --\\
\textbf{C467} & Single-digit numbers after minus signs in \dataset{dm\_mathematics} & Math & --\\
\textbf{C468} & Token \tok{~new} as part of bigram "a new" across diverse technical and legal contexts & Bigrams & \sae{3}{11996}\\
\textbf{C469} & Closing double quote after quoted text/strings & Formatting\\
\textbf{C470} & Token \tok{~many} as part of bigram "How many" in \dataset{dm\_mathematics} time calculation questions & Bigrams & \sae{4}{31179}\\
\textbf{C471} & Token \tok{nearest} as part of bigram "the nearest" in \dataset{dm\_mathematics} & Bigrams & \sae{4}{8078}\\
\textbf{C472} & Numeric digits after decimal points in scientific/technical text & Math & --\\
\textbf{C473} & Token \tok{~the} as part of bigram "to the" across technical/legal documents & Bigrams & \sae{4}{18242}\\
\textbf{C474} & Period after \tok{U.S} abbreviation in legal citations and formal documents & Formatting & \sae{2}{11214}\\
\textbf{C475} & Conditional \tok{~you} after \tok{if}/\tok{If}/\tok{when}/\tok{what} in conversational text & Grammar & \sae{4}{21933}\\
\textbf{C476} & Token \tok{~a} as part of bigram "Let a" in \dataset{dm\_mathematics} variable definitions & Bigrams & \sae{3}{16348}\\
\textbf{C477} & Newline after \tok{<|endoftext|>} token in \dataset{hackernews} and \dataset{freelaw} & Formatting & --\\
\textbf{C478} & \tok{~the} followed by \tok{~remainder}/\tok{~closest} in \dataset{dm\_mathematics} math word problems & Math & --\\
\textbf{C479} & \tok{Enron} after \tok{@} in \dataset{enron\_emails} email addresses && \sae{2}{30764}\\
\textbf{C480} & Token \tok{~control} in scientific/technical contexts && \sae{2}{4778}\\
\textbf{C481} & Sentence-initial capitalized words after newlines in \dataset{dm\_mathematics} && --\\
\textbf{C482} & 24-space indentation in \dataset{pubmed\_central} tables & Formatting & --\\
\textbf{C483} & Double newline paragraph breaks in \dataset{cc}, \dataset{pubmed\_central}, \dataset{wikipedia\_en}, and \dataset{arxiv} & Formatting & \sae{2}{3257}\\
\textbf{C484} & \tok{What} after newline in \dataset{dm\_mathematics} question sequences && \sae{4}{21647}\\
\textbf{C485} & Uppercase letters and short tokens in base64/encoded strings and technical content & Code & -- \\
\textbf{C486} & Prepositions \tok{~to}, \tok{~in} in mathematics and code and a range of other tokens & Grammar & --\\
\textbf{C487} & Relative pronoun \tok{~who} after person references in descriptive clauses & Grammar & \sae{4}{13411}\\
\textbf{C488} & Auxiliary verbs \tok{is}, \tok{~has}, \tok{~have} after subjects in technical/scientific text & Grammar & --\\
\textbf{C489} & LaTeX superscript \tok{\textasciicircum} and subscript \tok{\_\{} operators in mathematical notation & Math & --\\
\textbf{C490} & Two-digit numeric date/time components after delimiters && --\\
\textbf{C491} & \tok{\textbackslash n \textbackslash t \textbackslash t \textbackslash t} newline-tab indentation in HTML templates in \dataset{github-all} & Formatting & \sae{4}{26957}\\
\textbf{C492} & Exponent \tok{3} after \tok{$**$} in polynomial expressions in \dataset{dm\_mathematics} & Math & \sae{3}{31787}\\
\textbf{C493} & \tok{~been} after auxiliary verbs (\tok{~has}/\tok{~have}/\tok{~had}) in perfect tense constructions & Grammar & \sae{3}{14473}\\
\textbf{C494} & Abbreviation punctuation \tok{.,} after "e.g" and similar Latin abbreviations in technical/legal text & Formatting & --\\
\textbf{C495} & Token \tok{~etc} after a comma & Grammar & \sae{4}{22683}\\
\textbf{C496} & Digit \tok{0} after decimal points in numeric data (\dataset{github-all}, \dataset{arxiv}) && --\\
\textbf{C497} & Cyrillic text and non-ASCII characters in \dataset{stackexchange} and \dataset{github-all} && --\\
\textbf{C498} & \tok{~Court} and \tok{~court} after "the/this/Supreme/District" in \dataset{freelaw} legal documents && \sae{3}{6380}\\
\textbf{C499} & \tok{~including} after comma in list introductions & Grammar & \sae{4}{22796}\\
\textbf{C500} & Division operator after closing parenthesis in \dataset{dm\_mathematics} expressions & Math & --\\
\textbf{C501} & Sentence-initial \tok{~You} after periods in \dataset{hackernews}, \dataset{cc}, and \dataset{stackexchange} & Grammar & \sae{4}{21933}\\
\textbf{C502} & \tok{>} quote marker after newline in \dataset{hackernews} and \dataset{enron\_emails} discussions & Formatting & \sae{3}{31415}\\
\textbf{C503} & \tok{~with} following comma in lists & Grammar & \sae{4}{22796}\\
\textbf{C504} & Italic text marker \tok{$*$} in scientific/academic writing (\dataset{pubmed\_central}, \dataset{arxiv}) & Formatting & \sae{2}{30291}\\
\textbf{C505} & \tok{In} after newline in \dataset{dm\_mathematics} base arithmetic problems && --\\
\textbf{C506} & Opening angle bracket \tok{<} in HTML/XML markup in \dataset{github-all} and \dataset{stackexchange} & Code & \sae{3}{14450}\\
\textbf{C507} & Contraction \tok{'s} after "it" and "that" in informal text & Grammar & \sae{3}{22099}\\
\textbf{C508} & Prepositions \tok{~in}, \tok{~on}, \tok{~at}, often in technical contexts & Grammar & --\\
\textbf{C509} & Sentence-initial capitalized words after periods (\tok{~In}, \tok{~It}, \tok{~With}, \tok{~Since}, \tok{~First}) in scientific/technical text from \dataset{nih\_exporter}, \dataset{wikipedia\_en}, \dataset{cc}, \dataset{uspto\_backgrounds}, \dataset{pubmed\_central} & Grammar & --\\
\textbf{C510} & LaTeX escaped backslash \tok{\textbackslash\textbackslash} often after newlines in \dataset{arxiv} math equations & Math & --\\
\label{table:pythia14m_cluster_list}
\end{longtable}

\subsection{Cluster taxonomy in Pythia-14M}\label{section:results_pythia_14m_taxonomy}

Below we provide a taxonomy of the susceptibility clusters, classifying them into six categories: Formatting, Grammar, Code, Math, Bigrams and Other. 

\subsubsection{Syntactic \& Grammatical Patterns}

\begin{itemize}
    \item \textbf{Auxiliary \& Modal Verbs}
    \begin{itemize}
        \item Modal verbs like \tok{~can}, \tok{~may} \textbf{C14, 63, 400} and the infinitive \tok{~be} \textbf{C453}.
        \item Auxiliary verbs \tok{~is}, \tok{~has}, \tok{~have} \textbf{C1, 140, 488}.
    \end{itemize}

    \item \textbf{Determiners \& Articles}
    \begin{itemize}
        \item \tok{~a} following prepositions \textbf{C4, 192}.
        \item \tok{~an} \textbf{C26, 121, 232}.
        \item \tok{~the} in specific bigrams like ``that the'' or ``to the'' \textbf{C53, 270, 408}.
    \end{itemize}

    \item \textbf{Pronouns \& References}
    \begin{itemize}
        \item Relative pronouns \tok{~which}, \tok{~who}, \tok{~whose} \textbf{C104, 116, 283}.
        \item Third-person pronouns \tok{~he}, \tok{~she}, \tok{~they} \textbf{C74, 124, 227}.
    \end{itemize}

    \item \textbf{Prepositions \& Conjunctions}
    \begin{itemize}
        \item Conjunctions \tok{~and}, \tok{~but}, \tok{~or} \textbf{C55, 113, 299}.
        \item Prepositions \tok{~in}, \tok{~on}, \tok{~at}, \tok{~with} \textbf{C11, 28, 64}.
    \end{itemize}

    \item \textbf{Other}:
    \begin{itemize}
        \item Passive constructions like \tok{~by} after past participle \tok{~caused}, \tok{~supported} \textbf{C41}.
    \end{itemize}
\end{itemize}

\subsubsection{Mathematical \& Scientific Notation}

\begin{itemize}
    \item \textbf{Algebraic Operations}
    \begin{itemize}
        \item Exponentiation \tok{**} \textbf{C2, 40, 108}.
        \item Multiplication \tok{*} \textbf{C15, 57, 155}.
        \item Division/Fractions \tok{/} \textbf{C145, 224}.
        \item Equality \tok{~=} \textbf{C54, 356}.
    \end{itemize}

    \item \textbf{Variable Definitions}
    \begin{itemize}
        \item ``Let [variable] be...'' \textbf{C70, 129, 194}.
        \item Single-letter variables (\tok{x}, \tok{y}, \tok{n}) \textbf{C10, 51, 67}.
    \end{itemize}

    \item \textbf{Problem Statements}
    \begin{itemize}
        \item ``What is...'' \textbf{C169}.
        \item ``Calculate...'' \textbf{C439}.
        \item ``Collect the terms...'' \textbf{C344}.
        \item ``nearest to'' \textbf{C471}.
    \end{itemize}

    \item \textbf{LaTeX Typesetting}
    \begin{itemize}
        \item Math delimiters \tok{\$} \textbf{C49, 68, 374}.
        \item Commands \tok{\textbackslash begin}, \tok{\textbackslash frac}, \tok{\textbackslash usepackage} \textbf{C150, 249, 399}.
    \end{itemize}

    \item \textbf{Numeric \& Unit Formatting}
    \begin{itemize}
        \item Decimals and floating points \textbf{C23, 127, 496}.
        \item Negative numbers/Minus signs \textbf{C7, 204, 343}.
    \end{itemize}
\end{itemize}

\subsubsection{Code, Markup \& Technical Syntax}

\begin{itemize}
    \item \textbf{Markup \& Tags (HTML/XML)}
    \begin{itemize}
        \item Opening/closing tags \tok{<} \tok{>} \textbf{C52, 170, 242}.
        \item Attributes \tok{ref-type}, \tok{class} \textbf{C330, 358}.
        \item Specific tags \tok{div}, \tok{span}, \tok{li} \textbf{C426}.
    \end{itemize}

    \item \textbf{Control Structures \& Syntax}
    \begin{itemize}
        \item Conditionals \tok{if (} \textbf{C97, 339}.
        \item Braces \tok{\{} \tok{\}} for blocks \textbf{C86, 162, 337}.
        \item Comments \tok{//} \textbf{C404}.
    \end{itemize}

    \item \textbf{Indentation \& Whitespace}
    \begin{itemize}
        \item Indentation (4 or 8 spaces) \textbf{C17, 106}.
        \item Newlines in code blocks \textbf{C60, 103, 491}.
    \end{itemize}

    \item \textbf{Infrastructure \& Metadata}
    \begin{itemize}
        \item URL protocols \tok{http}\tok{://} \textbf{C195, 286}.
        \item Domain suffixes \tok{.}\tok{com}, \tok{.}\tok{org} \textbf{C107, 296}.
        \item Email headers ``Sent:'' \textbf{C446}.
    \end{itemize}
\end{itemize}

\subsubsection{Document Structure \& Formatting}

\begin{itemize}
    \item \textbf{Sentence \& Document Boundaries}
    \begin{itemize}
        \item End of text token \tok{<|endoftext|>} \textbf{C45, 171}.
        \item Sentence starters (Capitalized words) \textbf{C36, 96, 384}.
        \item Question marks \textbf{C87, 131}.
    \end{itemize}

    \item \textbf{Citations \& References}
    \begin{itemize}
        \item Brackets in citations and references within text \textbf{C20, 256}.
        \item Legal citations ``U.S.C.'', ``F.2d'' \textbf{C210, 215}.
        \item Figure references \textbf{C33}.
    \end{itemize}

    \item \textbf{Visual Separators}
    \begin{itemize}
        \item Horizontal rules \tok{---}, \tok{===} \textbf{C229, 306, 406}.
        \item Double newlines for paragraphs \textbf{C22, 325, 483}.
    \end{itemize}

    \item \textbf{Lists \& Enumeration}
    \begin{itemize}
        \item Multiple choice options ``(a)'', ``(b)'' \textbf{C315, 437}.
        \item List bullets and numbers \textbf{C166}.
    \end{itemize}
\end{itemize}

\subsubsection{Bigrams}

\begin{itemize}
\item \textbf{Bigrams}
    \begin{itemize}
        \item ``United \tok{States}'' \textbf{C373}.
        \item ``associated\tok{~with}'' \textbf{C346}.
        \item ``Thank \tok{you}'' \textbf{C457}.
        \item ``based\tok{~on}'' \textbf{C99}.
    \end{itemize}
\end{itemize}

\subsection{Low-level vs high-level patterns}\label{appendix:cluster-levels}

The clusters discovered by our methodology vary in their level of abstraction. We attempt here to characterize the distinction between ``low-level'' or ``syntactic'' clusters and ``high-level'' or ``abstract'' clusters. The distinction is not sharp, and even our most abstract examples are still quite low-level. Nonetheless this is an interesting qualitative aspect of the clusters to track as we look at larger models in the future: one can see for example a distinction between the rather low-level SAEs in \citet{bricken2023monosemanticity} for a one-layer transformer with the much more abstract examples in \citet{templeton2024scaling}.

Low-level clusters tend to be defined by a specific token in a specific positional pattern, depending on immediate, local features such as ``what character precedes this one?'' Examples include \clu{148} (the digit \tok{2} as an exponent), \clu{17} (8-space indentation after newlines), \clu{154} (comma after \tok{ECT} in Enron email address lists), and \clu{389} (time format \tok{00} after colons in timestamps).

Higher-level clusters tend to abstract over token identity to capture a shared functional role, or require recognizing compositional structure across longer distances.

\begin{itemize}
    \item \clu{10}, \cluq{58}, \cluq{269}, \cluq{319}: Variable names in \pilesub{dm\_mathematics}
    \item \clu{11}: Prepositions (\tok{~to}, \tok{~on}, \tok{~by}, \tok{~in}, \tok{~for})
    \item \clu{19}: Comparison words (\tok{~smaller}, \tok{~greater}, \tok{~bigger})
    \item \clu{63}: Modal verbs (\tok{~will}, \tok{~is}, \tok{~should}, \tok{~would}, \tok{~must}) following subjects
    \item \textbf{C100}: Double space separator between multiple choice options in \dataset{dm\_mathematics}
    \item \clu{125}: Superlative/ordinal adjectives (\tok{~first}, \tok{~most}, \tok{~best}, \tok{~following}) in natural language
    \item \clu{183}: Superlative adjectives (\tok{~most}, \tok{~best}, \tok{~last}, \tok{~fastest})
    \item \clu{217}: Content of closing HTML tags (e.g. \tok{a}, \tok{b}, \tok{td})
    \item \clu{247}: Common verbs after plural subjects (\tok{~have}, \tok{~are}, \tok{~know}, \tok{~to})
    \item \clu{252}: Digits after minus sign in \dataset{dm\_mathematics} 
    \item \textbf{C280}: Sentence-initial conjunctive adverbs (\tok{~However}, \tok{~Thus}, \tok{~Therefore}, \tok{~Furthermore}, \tok{~Moreover})
    \item \clu{410}: Modal verb completions (\tok{~be}, \tok{~been}, \tok{~not}, \tok{~go}, \tok{~also})
    \item \clu{424}: Conjunctive adverbs (\tok{~thus}, \tok{~therefore}, \tok{~hence})
    \item \clu{426}: HTML tag contents (\tok{li}, \tok{td}, \tok{b}), both opening and closing but mostly the former
    \item \clu{440}: Uppercase letter sequences in base64/encoded data
\end{itemize}

\subsection{Detailed Example of Clusters}\label{appendix:detailed_examples}

Below we give 30 randomly selected samples from each of the twelve clusters showcased in \cref{fig:cluster_cards}, as well as an evaluation of whether this sample is ``on theme''.

Note that the clusters are strongly monothematic. \clu{455} (which was selected for having high entropy among its tokens) is the only cluster with fewer than 75\% of contexts clearly having the same theme. Seven out of the twelve clusters (including every cluster that was randomly selected) have greater than 90\% of examples matching the cluster's theme. We view this as compelling evidence that the clusters are identifying context-token pairs that are related to each other, with little outside noise, and that the labels given to each cluster are broadly accurate.

It is also interesting to note that when contexts in a cluster are ``off-theme'', they are frequently related. For example in \clu{294} ( months and seasons) the examples that fall out of the theme include years and the word ``date''.



\section*{Supplementary material}

Cluster 11   (cluster size=7071)
Prepositions
score: 23/30 = 0.767
YES=23  NO=7  ?=0  total=30

| # | ✓/✗ | Sample |
|---|---|---|
| 1 | ✓ | insertion makes them highly desirable and essential for current spinal surgeries. We have tried to develop a new software |
| 2 | ✓ | valley. Of these, approximately 10,000 live in area C administers by the Israeli Coordinator |
| 3 | ✓ | high-performance liquid chromatography (HPLC) is discussed, with special reference to the need for standardized |
| 4 | ✗ | conduit in this context is defined as any tunnel, canal, pipeline, aqueduct, flume |
| 5 | ✓ | alkylene) moiety is composed of oxyalkylene units selected from 2 to 5 carbon oxy |
| 6 | ✓ | bkit-transition:all 1s;<br><br>3, 4, 4, 3).<br>The analog reproduced signal thus filtered by the analog filter 4 |
| 7 | ✓ | you are done, you can store the rotisserie rod behind the grill on the storage hooks. |
| 8 | ✓ | <\|endoftext\|> and AJ hugging, and a recording of a voice mail that |
| 9 | ✗ | mathbf{x}$ is a normalized vector $[\phi(\mathbf{x}), 1-\phi(\mathbf |
| 10 | ✗ | October 17, 1998.<br><br>Dog Show<br>Dog Show was an aptly titled parody of |
| 11 | ✓ | and fuel based on usage)<br>TW will bill Burlington through normal course of business (reservation |
| 12 | ✓ | console.<br><br>FB will be forced to provide a comprehensive export feature by law, which will |
| 13 | ✗ | the formation of PCSK9/LDL receptor complexes in vitro. A1AT also promotes the |
| 14 | ✓ | isto can be reduced if Hapke's parameters for prominent terrains on the satellites are available |
| 15 | ✓ | equery<br>thank you, this is really the type of thing i wanted to hear, as i |
| 16 | ✗ | <\|endoftext\|> jsonResponse = sr.ReadToEnd();<br>                    Console. |
| 17 | ✓ | is<br>writing something up to send to me.  Before I agree on anything I will talk |
| 18 | ✓ | password 4-digit PIN, please don't go out and change yours to this!<br>> |
| 19 | ✓ | and the loci of Lewey bodies. We explored this issue using various antibodies to map the distribution of |
| 20 | ✓ | var d = new Date();<br><br> Cases that applied a stricter construction of § 17a(3) refused to reopen the bankruptcy case |
| 21 | ✓ | are detected and their boundaries traced. Detected nuclei are then linked through time to create a globally coherent |
| 22 | ✓ | lethality of cyclin A2 knockout mice has obviated understanding its role in other aspects of mammalian |
| 23 | ✓ | to enable structure elucidation and metabolite identification. These macrocyclic compounds, produced by the dinoflag |
| 24 | ✓ | (stem cell factor receptor), a type III receptor tyrosine kinase that is implicated in cell growth and survival |
| 25 | ✓ | . 3. Gby subunits are required for normal stimulation of G protein activation by receptors and regulate many |
| 26 | ✗ | 2 x Kapa Sybr Fast qPCR Mastermix (Biolabo Scientifics Instruments, Chât |
| 27 | ✗ | calculating portion for calculating a value related to a power of a transmission signal from first baseband signals quadr |
| 28 | ✓ | In Haynes, the Supreme Court held that<br>an officer's decision to commence or continue a |
| 29 | ✓ | bills and never seemed to want<br>for much.<br><br>It always seemed to me if the game |
| 30 | ✓ | and fuel based on usage)<br>TW will bill Burlington through normal course of business (reservation |

Cluster 161   (cluster size=60)
em dashes
score: 25/27 = 0.926
YES=25  NO=2  ?=3  total=30

| # | ✓/✗ | Sample |
|---|---|---|
| 1 | ✓ | months in on Keto, planning on taking blood tests in another 3 months -- I'm not |
| 2 | ✓ | of RTOs may not resolve.  Eliminating native load exceptions -- ie, |
| 3 | ✓ | style and atmosphere -- I could have sworn it was haunted by kindred spirits -- I hope so. |
| 4 | ✓ | _in the future_.<br><br>Some emails have been recovered the obvious way -- by querying the |
| 5 | ✓ | is unconscionable price-gouging by the big energy producers -- most of |
| 6 | ✓ | az group deployment create --resource-group abc-devops-test --template |
| 7 | ✓ | &I wanna see the phone glued to your ear!8 -- Ed Baughman |
| 8 | ✓ | Bosnia, from 1992 to 1996, both weapons and "advisers" -- that is, M |
| 9 | ✓ | satiety development and recovery; additionally flutamide or tamoxifen treatments -- alone or together with |
| 10 | ✓ | The theory is that it takes the body a while to adapt to anything -- esp. when |
| 11 | ? | web site:<br>http://www.nyiso.com ---> Services ---> Committees -- |
| 12 | ✓ | Subject:name change<br>Importance:High<br><br>Linda -- would you please let |
| 13 | ✓ | antralis occurred statistically more often in winter, functional dyspepsia -- in autumn, and |
| 14 | ✓ | with your leads?<br><br>That's the tip of the ice berg -- I think if you |
| 15 | ✓ | by the search firm nomination process)<br>CA-ISO.<br>Area included -- CA IOU's |
| 16 | ✓ | of which are pretty big, eg. the entire Australian<br>Capital Territory -- will kill you. |
| 17 | ✓ | of the town and declared it a'safe area.'" (ibid.) -- new discovery of a |
| 18 | ✓ | , it was very rewarding. I met some<br>very knowledgeable and interesting people -- even some that I |
| 19 | ✓ | , Chris<br>Subject:<br><br>hey toughguy -- whats the story -- how's everything going |
| 20 | ✓ | LINGTON, Va., May 21, 2012 /PRNewswire/ -- Westinghouse Electric |
| 21 | ✓ | oohoo!<br><br>Also in this issue:<br>Enron.com -- it's fabulously |
| 22 | ✗ | Crisis May Loom As 2001 Economic Threat<br><br>HEW/Veag -5:Vatten |
| 23 | ✓ | companies each bought<br>two names and were counted twice.)<br>But -- and this is a |
| 24 | ✗ | ,<br>521 U.S. 844, 874 (1997) ("In evaluating the free |
| 25 | ✓ | <\|endoftext\|>ISK OF RETAIL OPERATIONS -- January 18-19 |
| 26 | ✓ | .com<br> cc:<br> Subject: Jolles<br><br>Help -- pls help me |
| 27 | ✓ | point: for decades in the UK, school education forked at age 16 -- the point at |
| 28 | ? | you have only few branches). See man less:<br><br>-F or --quit-if- |

| 29 | ✓ | the place. Now, you can find<br>fantastic journalists at BuzzFeed -- they just aren't |
| 30 | ? | ~~~<br>RichardCA<br>Good:<br><br>git log --all --graph -- |

Cluster 167   (cluster size=71)
Text emphasizers, commonly '_', and '_' as subscript indicator
score: 25/30 = 0.833
YES=25 NO=5 ?=0 total=30

| # | ✓/✗ | Sample |
|---|-----|--------|
| 1 | ✓ | notes"}. Synthesis and analytical data for representative members within series **12**--**19** are provided below. |
| 2 | ✓ | .0. Results described as significant are based on a criterion of *P* \< 0 |
| 3 | ✓ | moderate and weak eccentricity damping, corresponding to $t_e/t_a$<br>\equiv |
| 4 | ✓ | } W_n(z) \cos \left(\bm{G_}n \cdot \ |
| 5 | ✓ | Hoe92]. The complex residues and the phases are obtained as $r_E=1. |
| 6 | ✓ | invest_ in us" and "why<br>you should invest in _us_ ". |
| 7 | ✓ | .2/4.3) which are the primary subunits of *I* ~(to, |
| 8 | ✓ | Garey and Johnson(1979)): Given a sequence of integers $a_1 \ldots a |
| 9 | ✓ | of stable isotope values for the leaves of *Mangifera indica* (Common name: |
| 10 | ✓ | it one of those typical things that<br>have the potential to go _really_ viral; something like |
| 11 | ✓ | , and seemingly equivalent<br>happiness, it's arguable that _not_ keeping a cat indoors |
| 12 | ✗ | WT, @heston_2009], the World Development indicators [@worldbank_2009] and @ |
| 13 | ✓ | the<br>numbers support it. The use of gaming consoles and smartphones _improve_ IQ, but |
| 14 | ✓ | A*4 + A), your code should like:<br>A_Multiply_4 <= A& |
| 15 | ✓ | Also, the section about proximity fuzing: cruise missiles _are_ intended to phys |
| 16 | ✗ | [@blundell_92], Cheng [*et al*]{} [@cheng_93] and Lind |
| 17 | ✓ | .<br><br>_Very_ sound advice. Doing this has produced _amazing_ results for us more |
| 18 | ✗ | <http://en.wikipedia.org/wiki/Millennium_Challenge_2002 |
| 19 | ✗ | map* will mean an exact symplectic diffeomorphism of a subset of ${\fR^}2$ onto its |
| 20 | ✓ | .<br><br>Some of this advice is good but actually being _speedy_ is not necessary. |
| 21 | ✓ | H^2 \, R^4}{c_b^2\,Z_* \, \alpha'^ |
| 22 | ✓ | is that employees' very perception of<br>"frequent communication" _changes_ depending on which kind |
| 23 | ✓ | approach the subject from a different angle.<br><br>Many kids _hate_ math but might enjoy |
| 24 | ✗ | not optimal".format(total_nonoptimal))<br><br>if __name__ == "__main |
| 25 | ✓ | is debatable<br><br>Not, not really. It's a _very_ good thing, and |
| 26 | ✓ | specific Ninj1-deficient mice, Ninj1^fl/fl^ mice were bred with |
| 27 | ✓ | [@Chan2005]. Messaging platforms which are widely adopted include *Facebook Messenger* for social networking or |

| # | ✓/✗ | Sample |
|---|---|---|
| 28 | ✓ | easy for a con man to set up shop. It takes _forever_ to get him |
| 29 | ✓ | B]. Here, however, we prefer to save the name *free energy* for the one that |
| 30 | ✓ | ------<br>DomKM<br>"Analysis of bank _Personal Identification Number numbers_ "<br><br>------ |

Cluster 187   (cluster size=261)
'people'
score: 22/28 = 0.786
YES=22 NO=6 ?=2 total=30

| # | ✓/✗ | Sample |
|---|---|---|
| 1 | ✓ | lfoundation.org<br><br>Category:1964 births<br>Category:Living people<br>Category:Dutch |
| 2 | ✓ | .com - Official Site<br><br>Category:1954 births<br>Category:American people of Cuban descent |
| 3 | ✓ | a reduction in the number of deaths in the United States by the number of people now dying from that |
| 4 | ✗ | a<br>replay of 1996.<br><br>??"It's too big, too complicated. |
| 5 | ✗ | iamNumber4<br>This just seems like the exact opposite of "Do one thing; and do it |
| 6 | ✓ | s or '90s, they have continued to struggle, [while] people at the top who |
| 7 | ✓ | but not in the way it was. "You think there are these amazing people running the country who |
| 8 | ✓ | collection of six<br>provocative, character-driven films focused on passionate people<br>making a world |
| 9 | ✓ | cycles like this every 10-20 years. Ironically, you either lose people to more<br>inn |
| 10 | ✓ | of oppressive<br>regimes. It takes dedicated, brave, well-funded people to access an area |
| 11 | ✗ | liquid crystal material suitable for a low driving voltage is required, that is to say, a liquid crystal |
| 12 | ✓ | , the Internet has been applied in and brought great influence to every aspect of people's daily life, |
| 13 | ✓ | 's a propaganda lie with no<br>substantiated evidence to defame White European people. To extract mon |
| 14 | ? | been found in our laboratory to appear in a repetitive fashion in different and unrelated individuals. Initial investigations during |
| 15 | ✗ | with Trump attorney Michael Cohen. Today, she sees the potential for a much bigger payday in making |
| 16 | ✓ | they<br>> > > still<br>> > > > haven't started interviewing people yet, since Barbara |
| 17 | ✓ | /Linux<br>would do the same. But I would assume that some knowledgeable people would<br>know about |
| 18 | ✗ | ~~~<br>jaequery<br>thank you, this is really the type of thing i wanted to hear |
| 19 | ✓ | atrophy are all aspects of AD that are also common in cognitively intact older people. However, the |
| 20 | ✓ | 6110<br>Almost everyone I know in the trades makes more money than the people I know<br>in |
| 21 | ✓ | more social desktop area there's<br>Mapme out of Israel... What are people's take on Google |
| 22 | ✓ | that I love...and I do love APOOO because of all the people I've met because |
| 23 | ? | may be heading your way on Saturday?  We'll<br>probably figure things out in the next |

| #  | ✓/✗ | Sample |
|----|-----|--------|
| 24 | ✓ | ------<br>bamboozled<br>Why is this surprising? Actually curious why people think this is weird |
| 25 | ✓ | Rich! Rich is not only a client, he's one of the people I've learned |
| 26 | ✓ | locking off all together. Touch ID is supposed to make it easy for those people to at least have |
| 27 | ✓ | , but also show need for promotion of post exposure prophylaxis, especially among older people<br>.<\|endoftext\|>Urinary |
| 28 | ✓ | issues have been termed "attitudinal" problems or lowered expectations and aspirations for people with disabilities which may |
| 29 | ✗ | every time you want to unlock your iOS device? Do you know how many times a day people pick |
| 30 | ✓ | of these "alternative" markup<br>formats. Normal non-technical people want Word (or |

Cluster 194   (cluster size=1737)
Variable names
score: 23/30 = 0.767
YES=23  NO=7  ?=0  total=30

| #  | ✓/✗ | Sample |
|----|-----|--------|
| 1  | ✓ | (s(k))?<br>14*k - 2<br>Let d(u) = -67 |
| 2  | ✓ | 44))*-2. Solve -50 = -5*i + 5*n, 10*n |
| 3  | ✗ | Whites.10<br><br>Margaret Sanger was also strongly anti-Semitic. She started |
| 4  | ✓ | 25.5. Let j = d - -2.5. Let v = j - 10 |
| 5  | ✓ | <br>-8, -2/7, 701<br>Factor -3*i**3/7 |
| 6  | ✓ | p(q) = -2*q + 28. Let d(m) = 143* |
| 7  | ✓ | b - 54989. Is b composite?<br>False<br>Suppose 4*r - 4*q |
| 8  | ✓ | - 353248752 = 0.<br>73, 22096<br>Suppose q**4/3 |
| 9  | ✓ | d(i(c))?<br>-603*c**2<br>Let i(g) = |
| 10 | ✓ | 7*m**2/2 + 6*m - 3. Let c be g(-4 |
| 11 | ✗ | in 0.2, m, y?<br>y<br>Let v = -15.8 + |
| 12 | ✓ | <br>Let o(a) = -7*a + 15. Let n be o(-7 |
| 13 | ✓ | 6*n. Determine j(u).<br>9<br>Let l(o) = 3. |
| 14 | ✓ | (b) x   (c) -1<br>b<br>Let v = 2 - 2 |
| 15 | ✓ | + t + 130. Is t a composite number?<br>True<br>Let p(n) = |
| 16 | ✗ | and sepsis, as well as increased viral infections caused by agents such as cytomegalovirus and influenza |
| 17 | ✗ | usepackage{mathrsfs}<br>          \usepackage{upgreek}<br>          \setlength{\oddsidemargin}{-69pt |
| 18 | ✗ | }^k} \frac{\phi}{2} {\boldsymbol{w} }'{\boldsymbol{\ |
| 19 | ✗ | b) -0.2   (c) 6   (d) 5<br>b<br>Which |
| 20 | ✓ | - 1)<br>Let s = -2 + 4. Let u(l) be the first |
| 21 | ✓ | -1. Let r = 0.2 - 0.3. Let u = 2 - 4 |

| # | ✓/✗ | Sample |
|---|---|---|
| 22 | ✓ | 37, 10726, 11689, 12272?<br>-9*f**3 - f |
| 23 | ✓ | -127407432*v - 5308643<br>Let y(m) = 227* |
| 24 | ✗ | Woods, Trevor<br>Subject:RE: e-mail address<br><br>Dear Michael, |
| 25 | ✓ | -5<br>Suppose -16*d = 14*d. Suppose f + 2*w |
| 26 | ✓ | + 26 = 0, h = -2*a - 16. Let w = 1 + h |
| 27 | ✓ | u - 117295. What is k(19)?<br>87<br>Let l(f) = |
| 28 | ✓ | Let h be 2 - (-3 - 876/(-172)). Let i = h - - |
| 29 | ✓ | is r rounded to six dps?<br>-0.000005<br>Let n = 25 + - |
| 30 | ✓ | *r. Does 35 divide h?<br>True<br>Suppose -5*i + 39 = - |

Cluster 207   (cluster size=330)
'be' as part of 'to be'
score: 29/30 = 0.967
YES=29 NO=1 ?=0 total=30

| # | ✓/✗ | Sample |
|---|---|---|
| 1 | ✓ | of innocuous, child-proof, flavorless mush!Demand to be challenged. To be |
| 2 | ✓ | e., looks in the direction of, a particular launcher that is to be actuated. The |
| 3 | ✓ | authors have read the manuscript and have given final approval of the current version to be published. |
| 4 | ✓ | sixty days before the post adjudication sale the political subdivision "shall cause notice to be given" to the |
| 5 | ✓ | Joyce Annette (née Flint), Jeffrey Dahmer grew up to be a normal child. |
| 6 | ✓ | that Each soul … wherever they may reside in the physical … has chosen to be there … before they |
| 7 | ✓ | a) How many fields of the "full" model are going to be ignored<br>b) |
| 8 | ✓ | Rome codebase is open source.  If you need your RSS reader to be able to cope with |
| 9 | ✓ | vibrations that may occur during transportation, to thereby cause the lead terminals 2 to be damaged. In |
| 10 | ✓ | robber. In like manner, I believe that it is better to be called names than |
| 11 | ✗ | .<br>$h(x-y) = x + y - 1 = hx^2 |
| 12 | ✓ | boring song and performance of the decade. Personally, I would be embarrassed to be shown up by the |
| 13 | ✓ | going into the division, playing the Redskins, it's going to be a fun game." |
| 14 | ✓ | species on the role of the large granular lymphocyte (LGL) shown to be the mediator of NK |
| 15 | ✓ | ▯▯▯ *  multiple bytes to be packed into a single |
| 16 | ✓ | to the passing of time with an existential fury that's too clever to be anything other than ca |
| 17 | ✓ | negative input voltage Vres.sub.i-1 and the sum to be multiplied by two. |
| 18 | ✓ | Which one is grammatically correct?<br>He is believed to be awarded the prize at |
| 19 | ✓ | State Farm Insurance Company. Rather, the Bankruptcy Court has allowed these claims to be resolved in a more |
| 20 | ✓ | , 5cm diameter) without introducing air bubbles. Subsequently, the extract to be partitioned is added |
| 21 | ✓ | 122A may not be used for moving the seat back member so as to be a set position, |
| 22 | ✓ | here, but I think it's important in this day and age<br>to be diligent in connecting |

| # | ✓/✗ | Sample |
|---|---|---|
| 23 | ✓ | to respond. With this information, the physician can tailor the treatment regimen to be better adapted to the |
| 24 | ✓ | ⬜⬜⬜timeoutTimer,<br><br>⬜⬜⬜// To know if global events are to be dispatched<br>⬜⬜⬜fire |
| 25 | ✓ | or OCLN-derived peptide may block the spreading of virus (to be tested in Aim 1 |
| 26 | ✓ | 506th, 101st.<br>John Neumann, first American bishop to be canonized<br>Nil |
| 27 | ✓ | was observed for the first time in protein kinases C but is now known to be present in many other |
| 28 | ✓ | : they are all attached to the frame of the snowmobile when adapted to be moveable from a |
| 29 | ✓ | that Enid Lyons and Senator Dorothy Tangney became the first women to be elected to the Federal |
| 30 | ✓ | ]   The Fourth Amendment provides:<br><br>The right of the people to be secure in their persons |

Cluster 217   (cluster size=204)
Content of html tags or citations
score: 28/30 = 0.933
YES=28  NO=2  ?=0  total=30

| # | ✓/✗ | Sample |
|---|---|---|
| 1 | ✓ | 8_1_1_value.html">v8::Value</a>* obj); |
| 2 | ✓ | addition to 9565 \[[@b18-viruses-01-01178],[@b20-viruses- |
| 3 | ✓ | ukogu, Ernest" <Ernest.Onukogu@dowjones. |
| 4 | ✓ | 879" class="Function"<br>        >P.isEquivalence</a<br>        ><a |
| 5 | ✓ | doc.html"><FONT CLASS="NavBarFont1"><B>Help</B></FONT></A |
| 6 | ✓ | a href="data-data-access-methods.html">Methods</a></li |
| 7 | ✓ | HRV in real-time RT-PCR was performed as described in \[[@b18-viruses- |
| 8 | ✓ | int-msg-sub-events.html">Subscribe Conversation Content</a></li |
| 9 | ✓ | a<br>        ><a name="4829"<br>        >          </a<br>        ><a |
| 10 | ✓ | href="../installation/index.html">Installation Instructions</a></li |
| 11 | ✓ | li><a href="../classes/Plugin Color.html">Plugin Color</a></li |
| 12 | ✗ | ukogu, Ernest" <Ernest.Onukogu@dowjones. |
| 13 | ✓ | ="../advanced_init/row_callback.html">Row created callback</a><br>⬜⬜⬜⬜⬜⬜⬜⬜</ |
| 14 | ✓ | -Pequiven.SEIPBundle.Controller.html">Controller</a><br>⬜⬜⬜⬜⬜⬜⬜</ |
| 15 | ✓ | <br><br>    You et al, 2009[@b46-cia- |
| 16 | ✓ | .core.Context.html#_preConstruct">_preConstruct</a></li |
| 17 | ✓ | ><br><div class="line"><a name="l01836"></a><span class=" |
| 18 | ✓ | .FormLoginBundle.html">FormLoginBundle<span></span></a><br>⬜⬜⬜⬜⬜⬜< |
| 19 | ✓ | ="../../../../../deprecated-list.html"><FONT CLASS="NavBarFont1"><B>Deprecated</B |
| 20 | ✓ | emf.cdo.common"><code>sessions</code></a>.</div |
| 21 | ✓ | government/system/laws.html">Treaties, laws and regulations</a></li |
| 22 | ✓ | about risk factors for asthma and respiratory infections in Finnish military conscripts \[[@b27-viruses- |
| 23 | ✓ | li><br><li><a href="#method.summary">Method</a></li |
| 24 | ✓ | FrameworkBundle.Tests.Templating.Loader.html">Loader</a><br>⬜⬜⬜⬜⬜⬜</ |
| 25 | ✓ | in java.lang">String</a></code></td><br><td class="colLast |
| 26 | ✓ | FALSE.</td></tr><br><tr class="separator:ga7b1340cb5 |

| #  | ✓/✗ | Sample |
|----|-----|--------|
| 27 | ✓ | CG__.Doctrine.Tests.Common.Util.html">Util</a><br>□□□□□□</ |
| 28 | ✓ | b91905e1cc">Ogre::GLRenderSystem</a><br>, < |
| 29 | ✗ | ACMP_GET_RX_STATE_COMMAND_TIMEOUT_MS},<br>              { |
| 30 | ✓ | ="../general/creating_libraries.html">Creating Libraries</a></li> |

Cluster 244   (cluster size=183)
Closing '>' in html tags
score: 29/30 = 0.967
YES=29 NO=1 ?=0 total=30

| #  | ✓/✗ | Sample |
|----|-----|--------|
| 1  | ✓ | LC/LibOrder.v</code></li><br>               <li>31 K <code |
| 2  | ✓ | -int-client-props.html">Client Properties</a></li> |
| 3  | ✓ | ><br>          </div><br>          </div><br>        </footer><br><br>        <script |
| 4  | ✓ | <div class="col-md-5"><br>                   <strong><br>                     Facsim |
| 5  | ✓ | a href="../../deprecated-list.html">Deprecated</a></li><br><li>< |
| 6  | ✓ | .html">Using API in callbacks</a><br>□□□□□□</li><br>□□□□□□<li |
| 7  | ✓ | ><a href="../../help-doc.html">Help</a></li><br></ul> |
| 8  | ✓ | TLC/LibTacticsDemos.v</code></li><br>               <li |
| 9  | ✓ | events itemMembers"><br><br>          </ul><br>        </li> |
| 10 | ✓ | ot/Q_ordered_field_properties.vo</code></li><br>               <li |
| 11 | ✓ | html">Base style - order-column</a><br>□□□□□□</li><br>□□□□□□<li |
| 12 | ✓ | ma decimal place</a><br>□□□□□□</li><br>□□□□□□<li><br>□□□□□□□□<a |
| 13 | ✓ | LC/LibMin.vo</code></li><br>               <li>71 K <code |
| 14 | ✓ | .ca/en/mobile.html">Mobile applications</a></li><br><li>< |
| 15 | ✓ | Type</a><br>□□□□□□</li><br>□□□□□□</ul></li><br>□□□□<li |
| 16 | ✓ | XCocoaWindow</a><br></li><br><li>_setWorldMat |
| 17 | ✓ | installation/index.html">Installation Instructions</a></li><br>□□□□□□<li |
| 18 | ✓ | href="../general/helpers.html">Helpers</a></li><br>□□□□□□<li |
| 19 | ✗ | "Western PA in general seems to be slipping back into (as I've been told it<br>> |
| 20 | ✓ | li><br><li>Constr \| </li><br><li>< |

| # | ✓/✗ | Sample |
|---|---|---|
| 21 | ✓ | ">State saving</a><br>󰀀󰀀󰀀󰀀󰀀󰀀󰀀</li><br>󰀀󰀀󰀀󰀀󰀀󰀀<li><br>󰀀󰀀󰀀󰀀󰀀󰀀󰀀󰀀<a |
| 22 | ✓ | level-1">value_type</span></a></p></li><br><li class |
| 23 | ✓ | _Qpositive_to_Q_properties.vo</code></li><br>        <li |
| 24 | ✓ | tr><br></table><br></td></tr><br></table><br></td></ |
| 25 | ✓ | .html">Highlighting rows and columns</a><br>󰀀󰀀󰀀󰀀󰀀󰀀󰀀</li><br>󰀀󰀀󰀀󰀀󰀀󰀀󰀀<li |
| 26 | ✓ | to_Qpositive.glob</code></li><br>        <li>198 K <code |
| 27 | ✓ | ="#field.summary">Field</a> &#124; </li><br><li>< |
| 28 | ✓ | erBundle.Command.html">Command</a><br>󰀀󰀀󰀀󰀀󰀀󰀀</li><br>󰀀󰀀󰀀󰀀<li |
| 29 | ✓ | " cellpadding="0" bgcolor="#FFFFFF"><br>    <tr><br>      <td |
| 30 | ✓ | .html">Sortable generated columns</a><br>        </li> |

Cluster 294   (cluster size=312)
Months and seasons
score: 24/26 = 0.923
YES=24  NO=2  ?=4  total=30

| # | ✓/✗ | Sample |
|---|---|---|
| 1 | ✓ | part III<br>Bombs Away, interview with Jeffrey Nyquist, 18 December 2004<br>Unmask |
| 2 | ✓ | Giddings was the first female Premier of the state of Tasmania between January 2011 and March 2014 |
| 3 | ✓ | replace the probe head on FRPS. During the latest experimental campaign in the spring of 2010, a |
| 4 | ✓ | The petition of real parties in interest for review by the Supreme Court was denied June 1, 1995. |
| 5 | ✓ | result of an on-the-job injury. A hearing was held on October 11, 2004, |
| 6 | ✓ | Don't judge me.<br><br>Winter Soldier: The Bitter March #2 (Mar |
| 7 | ✗ | How to delete a portlet in Liferay 6.1 programmatically from code |
| 8 | ✓ | of Texas and won on the Symetra Tour in 2014.<br><br>In January 2013, Pressel |
| 9 | ✓ | 387.6 (now Code § 46.2-355).<br>On September 15, 1990, |
| 10 | ✓ | Jones on bass and Cattini on drums, De Lane Lea in March 1968) also use |
| 11 | ✓ | :󰀀Storage allocations<br><br>Jill,<br>Lynn and Steve January want to change the |
| 12 | ✓ | trial memorandums of law, and the matter was submitted for decision on November 6, 1996. |
| 13 | ✓ | , including their Bankruptcy Schedule C. The Debtors' counsel also represented at the July 14, 2011 hearing |
| 14 | ✓ | ). For the last budget year of the grant (July 1, 2006 - June.30, 2007 |
| 15 | ✓ | the first  (a title given to certain light infantry regiments) in October, followed by the |
| 16 | ? | (that  is,  November  25  of  2008,  2009,  or |
| 17 | ✓ | flight to Guantanamo Bay, and was scheduled to remain there until February 8. Taylor Dec |
| 18 | ✓ | <br>Benjamin Franklin Johnson, Jr. (September 30, 1914 – July 1, 2006) |

| # | ✓/✗ | Sample |
|---|---|---|
| 19 | ✓ | August 1926. |
| | | The first Soldier's Medals were awarded on October 17, 1927 to |
| 20 | ? | His prior honors include the Joseph S. Meyers award in journalism in 2008 and the Robert F |
| 21 | ✓ | the same name. |
| | | Early life |
| | | Liu Heng was born in May, 1954 in Beijing |
| 22 | ✓ | dream,' buying a estate in River Vale, New Jersey. In August 1912, she had |
| 23 | ? | June 7th, accepted another order "under protest, especially as to delivery date." On June 11 |
| 24 | ✗ | etering Systems Give Residential Customers Greater Control Over Energy Use And Costs |
| | | Agreements |
| 25 | ✓ | The 16 countries do not have to finalize their WBC rosters until January, though we already |
| 26 | ✓ | meeting next week.  I am tied up in other PRC meetings on December 7, 8 |
| 27 | ✓ | the club. Ritchie was sacked by Barnsley on 21 November 2006, with the |
| 28 | ✓ | â□□s vice president of student services is scheduled to go to trial April 28 in a San |
| 29 | ? | ina to be a distinct species, and it was elevated to species status in 2015 by Tulloss |
| 30 | ✓ | area in 1960.  Raby was struck from the Navy List on 1 June 1968, and subsequently |

Cluster 316   (cluster size=201)
Opening parenthesis for multiple choice option
score: 29/30 = 0.967
YES=29 NO=1 ?=0 total=30

| # | ✓/✗ | Sample |
|---|---|---|
| 1 | ✓ | c |
| | | Which is the nearest to 3?   (a) 66    (b) -1 |
| 2 | ✓ | ?   (a) -0.3   (b) u    (c) 0 |
| 3 | ✓ | smallest value?   (a) -3   (b) q    (c) t |
| 4 | ✓ | smallest value?   (a) r   (b) -1    (c) 0. |
| 5 | ✓ | 1?   (a) x   (b) 2/11    (c) 9/ |
| 6 | ✗ | 22). |
| | | 38*l |
| | | Expand (0 + 0 - 2*s)*(2 + 2 |
| 7 | ✓ | (d) 4 |
| | | a |
| | | Which is the seventh smallest value?   (a) 2 |
| 8 | ✓ | is the closest to -5/3?   (a) 5    (b) 2/ |
| 9 | ✓ | 12.47?   (a) 4   (b) 1    (c) -2 |
| 10 | ✓ | . Which is the third smallest value?   (a) -2    (b) n |
| 11 | ✓ | smallest value?   (a) o   (b) -2    (c) q |
| 12 | ✓ | 0.01 |
| | | Which is the nearest to -3/2?   (a) -1 |
| 13 | ✓ | .48. Which is the biggest value?   (a) r    (b) -0 |
| 14 | ✓ | 7998? |
| | | 4/9 |
| | | Which is the biggest value?   (a) -0 |
| 15 | ✓ | )/(-9)*(-10)/3. Which is the third biggest value?   (a) 3 |
| 16 | ✓ | a) -1/4   (b) -2/9    (c) 2 |
| 17 | ✓ | Which is the third smallest value?   (a) -4    (b) -21 |
| 18 | ✓ | the biggest value?   (a) 5   (b) l    (c) 2 |
| 19 | ✓ | 3/7 |
| | | Which is the closest to -63.7?   (a) 0. |
| 20 | ✓ | m = -176 + 173. Which is the second biggest value?   (a) m |
| 21 | ✓ | 65. Which is the second biggest value?   (a) d    (b) f |
| 22 | ✓ | -501? |
| | | 0.2 |
| | | Which is the smallest value?   (a) -4 |
| 23 | ✓ | + 1.9. Which is the closest to 0.2?   (a) t |
| 24 | ✓ | = f - 1.77. Which is the second smallest value?   (a) -6 |

| #  | ✓/✗ | Sample |
|----|-----|--------|
| 25 | ✓ | (c) 6<br>a<br>Which is the third smallest value?   (a) -24 |
| 26 | ✓ | 0.04 + 0.44. Which is the biggest value?   (a) j |
| 27 | ✓ | smallest value?   (a) -5   (b) j   (c) 0. |
| 28 | ✓ | /65?<br>-3<br>Which is the second smallest value?   (a) 22876 |
| 29 | ✓ | j - 10.8. Which is the closest to -1?   (a) 3/ |
| 30 | ✓ | ?   (a) -1.72   (b) 4   (c) -0 |

Cluster 411   (cluster size=226)
'will', commonly in the context of research'
score: 25/26 = 0.962
YES=25 NO=1 ?=4 total=30

| #  | ✓/✗ | Sample |
|----|-----|--------|
| 1  | ✗ | :<br><br>Try this:<br>s = s.TrimEnd(',',' '); |
| 2  | ? | DC without knowing what happened to her younger sister.<br>The group only decides to act when D |
| 3  | ✓ | and re-evaluated twice per year as in Aim#1. They will receive a comprehensive oral |
| 4  | ✓ | can take it away from you. With this book and CD, your career will benefit for years to |
| 5  | ✓ | undergo clonal expansion in rats immunized to produce autoimmune interstitial nephritis. These studies will utilize a panel of |
| 6  | ✓ | these studies will characterize the molecular heterogeneity and functional complexity of CTCL, which will shed light on how |
| 7  | ✓ | study transformation using guinea pig CMV, and Dr. Richard Tenser who will study latency. A |
| 8  | ✓ | with the help of transgenic mice and cell line reporters. Ultimately, this component will address both intrinsic, |
| 9  | ✓ | community will be assessed at baseline and one-year follow-up. Participants will be administered implicit cognitive |
| 10 | ✓ | to provide an overview and will be covered briefly. The focus of this article will be on novel treatments |
| 11 | ✓ | the FEFsem neurons in the generation of the TVOR. This experiment will provide greater understanding of |
| 12 | ✓ | lioma therapies by inhibiting immune suppressive myeloid-derived cells. The hypothesis will be addressed by the |
| 13 | ✓ | B) to treat the lungs of rabbits with pulmonary inflammatory injury. Additional work will be performed to establish |
| 14 | ✓ | of cannabinoid-based medications, will also be evaluated. Aim 3 will characterize the pharmacokinetics |
| 15 | ✓ | otine in selectively decreasing nicotine self-administration across repeated injections. The specific aims will determine if the n |
| 16 | ✓ | which impairs our options to treat life-threatening fungal disease. This proposal will identify molecules released by |
| 17 | ? | Whether the state is likely to enforce the values he favors is a question he leaves un-examined |
| 18 | ✓ | dyes in trapped microspheres. The microspheres and instrumentation developed by this project will be suitable for continuous |
| 19 | ✓ | year). 2) A neighborhood level intervention when combined with a household level intervention will provide an enhanced effect |
| 20 | ✓ | effect of the concentration of each of these nucleotides on cellular DNA synthesis. Studies will also be conducted on |
| 21 | ✓ | hacked this in one day using some UI for existing stuff, so the UI will be a bit |
| 22 | ✓ | Hauser of the CDS Group of Mars Corp.<br><br>The subject will be:  Com |
| 23 | ✓ | to the surgeon, more personalized and directed operations can be conducted. This work will allow for more preservation |
| 24 | ✓ | will be based on similar tools established previously for handling HLA sequences. The database will be made freely available |
| 25 | ✓ | neuroblastoma to address a fundamental problem with high clinical relevance. Our results will establish the requirement for |
| 26 | ✓ | . The true-inducer which alters GAL4-GAL80 interaction will be determined. The |
| 27 | ✓ | the sample source. As a safety measure, an aliquot of all samples will be stored offsite |

67

| # | ✓/✗ | Sample |
|---|---|---|
| 28 | ? | investigated using antagonists against AA products and autocoids. In summary, this study seeks to define the neuro |
| 29 | ✓ | diagnosis, treatments, visit duration, source of payment, referral). This information will permit description of the |
| 30 | ? | . Malarial retinopathy (MR) is a retinal manifestation of CM that presents with a highly specific |

Cluster 455   (cluster size=5122)
Capital letters, often in names
score: 15/27 = 0.556
YES=15 NO=12 ?=3 total=30

| # | ✓/✗ | Sample |
|---|---|---|
| 1 | ✓ | was not engineered for our benefit.We are about to change this.Mother Nature proceeds by an entirely |
| 2 | ✗ | to the level of a taking or damage under article I, section 22 is the State required |
| 3 | ✗ | The pro erlangers I know are mostly lisp or haskell programmers who consider |
| 4 | ? | "Love Divine, All Loves Excelling" and "Come, Thou Long Expected |
| 5 | ✓ | -C20 cycloalkenyl, C6-C20 aryl, C7-C20 |
| 6 | ✓ | -7      Powergen      PWG   USD   41.25 B-3-3 |
| 7 | ✗ | (\Omega)$ and that: $$\begin{array}{cc} \mathtt {(2)} &\int |
| 8 | ✓ | CBM, TRM, and load flow input data.  TPs have incentives to |
| 9 | ✓ | DEVELOPMENT@ENRON_DEVELOPMENT cc: Steven J Kean/NA |
| 10 | ✗ | years and 51% at 5 years. The combined retention rate was 88% at 2 years and 74 |
| 11 | ✗ | first-round funding from CenterPoint Ventures, InterWest, and Sevin Rosen. |
| 12 | ✓ | limited liability business Setija Oesaha was established and headquartered in Pabean sub- |
| 13 | ✗ | H. pylori* supernatant on bacteriocin- and acidogenicity-related genes expression of *S |
| 14 | ✓ | Says that their approach evidenced a  willingness to work with DYN to integrate the |
| 15 | ? | proposed research is innovative because it focuses on two independent strategies designed specifically to enhance A clearance from the brain |
| 16 | ✓ | 59 F.2d at 794; Interpool Ltd., 635 F.Supp. at |
| 17 | ✗ | metal—everybody comes from a family of some kind. Hrishikesh Hirway, creator |
| 18 | ? | $g'(e)=g(e)$ for any $e\in E(T')$. |
| 19 | ✓ | measured at a wavelength of 560 nm. The reaction mixture contained 0.1 M phosphate buffer (pH |
| 20 | ✓ | HEALTH RELEVANCE: Savannah State University (SSU), an HBCU with 94 |
| 21 | ✗ | in a workout. But you'll want to cut your calories and eat very smart and healthy |
| 22 | ✗ | ) were independent risk factors of bronchopleural fistula for patients with benign lung disease after pneumonectomy |
| 23 | ✓ | into the cam ring 2, and then turned to align the guide grooves 3A with open ends 2 |
| 24 | ✗ | and "duped" to a "brain-injured, senile, crazy old man |
| 25 | ✓ | ref-type="ref"}). However, knowledge of how noncoding genes regulate GLUT1 expression is |
| 26 | ✓ | compared with Steve Jobs, it's worth remembering that Steve Ballmer and Bill Gates comprehensively beat |
| 27 | ✗ | -point calibration showed that using the H2Aneal recipe the etch rate in $Si_ |
| 28 | ✓ | , Alan Aronowitz/HOU/ECT@ECT Subject: EOL Sports Game |
| 29 | ✓ | The Jordan Valley (, Ghor al-Urdun;, Emek HaY |
| 30 | ✗ | FNAAABCABL7BQABQgARACECaQBCAEIA |

\label{fig:cluster_detail_examples}

\end{document}